\documentclass{article}

\usepackage[main, final]{neurips_2026}

\usepackage[utf8]{inputenc} %
\usepackage[T1]{fontenc}    %
\usepackage{hyperref}       %
\usepackage{url}            %
\usepackage{booktabs}       %
\usepackage{amsfonts}       %
\usepackage{nicefrac}       %
\usepackage{microtype}      %
\usepackage{xcolor}         %

\usepackage{amsmath,amsfonts,bm, amsthm, mathtools}

\newcommand{\node}{{\bm x}}
\newcommand{\invar}{{\bm z}}

\newcommand{\weight}{{\lambda}}

\newcommand{\parent}{{\mathrm{pa}}}
\newcommand{\child}{{\mathrm{ch}}}

\newcommand{\sib}{{\mathrm{sib}}}
\newcommand{\ean}{{\mathrm{an}_{\mathrm{sib}}}}
\newcommand{\para}{{\bm{\theta}}}
\newcommand{\synth}{{\tilde{\bm{h}}}}
\newcommand{\grad}{{\bm{g}}}
\newcommand{\sgrad}{{\tilde\grad}}
\newcommand{\paragrad}{{\grad}}
\newcommand{\sparagrad}{{\tilde\paragrad}}
\newcommand{\hparagrad}{\hat{\grad}}
\newcommand{\hsparagrad}{\hat{\tilde\paragrad}}
\newcommand{\jacob}{{\bm{J}}}
\newcommand{\mse}{{\Delta}}
\newcommand{\smse}{{\tilde \Delta}}
\newcommand{\cmse}{{\rho}}
\newcommand{\scmse}{{\tilde \rho}}
\newcommand{\error}{{\nu}}
\newcommand{\serror}{{\tilde\nu}}
\newcommand{\bias}{{b}}
\newcommand{\sbias}{{\tilde{b}}}
\newcommand{\Mse}{\Delta}
\newcommand{\sMse}{{\tilde{\Delta}}}

\newcommand{\Berror}{{{\epsilon}}}

\newcommand{\sss}{{\tilde{\bm{p}}}}
\newcommand{\qqq}{{\tilde{ \bm{s}}}}

\DeclareMathOperator{\tr}{tr}

\newtheorem{theorem}{Theorem}

\newtheorem{assumption}{Assumption}

\newtheorem{lemma}{Lemma}
\newtheorem{definition}{Definition}

\def\eqref#1{equation~\ref{#1}}

\def\1{\bm{1}}

\DeclareMathAlphabet{\mathsfit}{\encodingdefault}{\sfdefault}{m}{sl}
\SetMathAlphabet{\mathsfit}{bold}{\encodingdefault}{\sfdefault}{bx}{n}

\def\gF{{\mathcal{F}}}

\def\gL{{\mathcal{L}}}

\def\gN{{\mathcal{N}}}
\def\gO{{\mathcal{O}}}

\def\gX{{\mathcal{X}}}

\def\sR{{\mathbb{R}}}

\newcommand{\E}{\mathbb{E}}

\newcommand{\Var}{\mathrm{Var}}

\newcommand{\Cov}{\mathrm{Cov}}

\DeclareMathOperator*{\argmin}{arg\,min}

\usepackage{algorithm}
\usepackage{algpseudocode}
\usepackage{caption}
\usepackage{float}
\usepackage{autonum}
\usepackage{wrapfig}
\usepackage{natbib}
\usepackage{booktabs}
\usepackage{siunitx}
\usepackage{enumitem}
\usepackage{amssymb}
\usepackage{graphicx}
\usepackage{subcaption}
\usepackage{fancyhdr}

\title{Is Backpropagation Optimal? When Synthetic Gradients Improve Sample Efficiency}

\author{Yibo Jacky Zhang \\
Stanford University\\
\texttt{yiboz@stanford.edu} \\
\And
Zeyu Tang \\
Stanford University \\
\texttt{zeyu@cs.stanford.edu} \\
\And
Sanmi Koyejo \\
Stanford University \\
\texttt{sanmi@stanford.edu} \\
}

\begin{document}

\maketitle

\begin{abstract}
     Backpropagation is the default learning rule for artificial neural networks and is often treated as the settled approach whenever differentiability is available. In this work, we revisit this convention through a theoretical lens of sample efficiency. We introduce a unified vectorized feedback framework for loss-based and reward-based learning on computational graphs, in which synthetic gradients emerge as a natural alternative to backpropagation. We characterize the conditions under which synthetic gradients can achieve a lower gradient-estimation mean squared error than backpropagation. We construct examples illustrating that this sample efficiency advantage can be arbitrarily large. Experiments on contextual bandits and reinforcement learning tasks demonstrate the potential of our theoretical findings.
\end{abstract}

\section{Introduction}\label{sec:intro}

Backpropagation~\citep{Linnainmaa1976taylor, werbos1981applications, rumelhart1986learning} has become the default learning rule for artificial neural networks, repeatedly outperforming alternatives~(e.g., \citep{lillicrap2016random, lee2015difference, jaderberg2017decoupled, scellier2017equilibrium, hinton2022forward}).
As a result, backpropagation (BP) is often treated as the settled answer to neural networks learning whenever differentiability is available. In this paper, we revisit this convention from the perspective of gradient-estimation sample efficiency.

To formalize the sample efficiency analysis, we consider the gradient estimation problem on a computational graph with vectorized terminal feedback. This vectorized feedback framework unifies loss-based learning, where the vectorized feedback is the gradient of a loss, and reward-based learning, where it may be given by policy gradients \citep{williams1992simple, sutton1999policy}. 
In this picture, as detailed in the next section, backpropagation is analogous to a Monte Carlo (MC) estimator, and temporal-difference (TD) learning~\citep{sutton1988learning} provides a classical bootstrapped alternative to MC estimation. In the vectorized feedback setting, this bootstrapped alternative turns out to be the synthetic gradients~\citep{jaderberg2017decoupled, czarnecki2017understanding}.

Synthetic gradients (SG) were originally introduced from the perspectives of decoupled computation and distributed training,  and their connection to TD learning was also noted early on~\citep{jaderberg2017decoupled}. However, prior work largely studied synthetic gradients in essentially chain-structured computational graphs, and established no clear empirical or theoretical advantage over unconstrained backpropagation. As a result, the topic has gradually fallen out of favor, with only limited recent interest \citep{pemberton2024bpmathbflambda, janfaza2023ada}. 

Our point of departure is that this negative picture reflects an incomplete analysis. As with TD versus MC methods, the advantage of synthetic gradient propagation is not universal, but regime-dependent. TD methods can offer practical sample-efficiency advantages over MC methods through bootstrapping, and their theoretical analysis is an active area of research~\citep{kearns2000bias, lazaric2010finite, pmlr-v75-bhandari18a, cheikhi2023statistical, cheng2024surprising}. In contrast, synthetic gradients remain theoretically underexplored, with limited existing work focused mainly on training convergence~\citep{czarnecki2017understanding, huo2018decoupled}. 
However, the analogy between TD and synthetic gradients suggests a possible sample-efficiency advantage over backpropagation, which leads to our central question:
\begin{center}
    \begin{minipage}{\linewidth}
        \centering
        {\textit{When can synthetic gradients improve gradient-estimation sample efficiency over backpropagation?}}
    \end{minipage}
    
\end{center}

Formally, this paper studies gradient estimation on differentiable computational graphs. We compare the mean squared error of empirical gradient estimators formed from i.i.d. samples under backpropagation and synthetic gradient propagation.

The main contribution of this paper is thus a formal theoretical investigation into this underexplored yet potentially important question. Our results show that \emph{backpropagation can be suboptimal in terms of gradient-estimation quality when gradient feedback contains uncertainty. We identify conditions under which synthetic gradients reduce gradient-estimation MSE}. \textbf{Our results can be summarized as follows.}

We introduce a vectorized feedback framework in Section~\ref{sec:motivation} that unifies loss-based and reward-based learning on computational graphs. Motivated by this framework, we formally define the gradient estimation problem in Section~\ref{sec:theory}. We then identify two conditions, defined formally in Section~\ref{sec:condition}, that govern the comparison. Condition $(A)$ says there must be uncertainty in gradient estimations. Specifically, $(A)$ requires \emph{at least one} of the following to hold: $(A. 1)$ \textit{randomness in the terminal gradient feedback}; $(A. 2)$ \textit{internal stochasticity in the network computation}; $(A. 3)$\textit{ sparse connectivity in the network that leads to partial observability at individual nodes}. Condition $(B)$ is a \textit{local sufficiency requirement}: each node's local observation contains sufficient information for the conditional mean of its incoming gradient feedback. Under these conditions, the comparison between backpropagation and synthetic gradients becomes sharp.

Our theory yields the following separation. With condition $\neg (A) \wedge (B)$, backpropagation is optimal in estimating mean squared error (MSE) among all conditionally unbiased gradient estimators. 
In contrast, with $(A) \wedge (B)$, an oracle synthetic gradient generator achieves a smaller estimation MSE than backpropagation. Thus, for synthetic gradient propagation to improve sample efficiency, some source of randomness or structural sparsity is necessary. This also helps explain why prior work often found synthetic gradients to be worse: synthetic gradients were typically evaluated in deterministic and fully connected networks on relatively clean datasets, essentially where condition $(A)$ is absent. When condition $(B)$ does not hold exactly, or when the synthetic gradient is not an idealized oracle, synthetic gradients may be biased. In this case, we show that a mixture of synthetic and backpropagated gradients can be preferable due to a better bias-variance tradeoff.

Inspired by the conditions $(A)$ and $(B)$, in Section~\ref{sec:example} we show examples where synthetic gradient propagation can be arbitrarily more sample efficient than backpropagation with estimated non-oracle gradient predictors. This construction exploits a concrete realization of the local sufficiency condition $(B)$: a network of specialized expert nodes that decomposes a global task into locally sufficient components.

Motivated by our theoretical findings, in Section~\ref{sec:exp} we apply synthetic gradients to a class of sparsely connected networks. We evaluate this approach on a contextual bandit problem with MNIST contexts \citep{lecun2002gradient} and on POPGym, a recent reinforcement learning benchmark for partially observable environments \citep{moradpopgym}. In both settings, the results reflect our theoretical findings: synthetic gradient propagation improves sample efficiency when non-trivial gradient uncertainty is present. 
We discuss this connection, along with limitations and future directions, in Section~\ref{sec:discuss}. Related work, proofs, experimental details, and ablation studies are provided in the appendix.

\section{Vectorized Feedback on Computational Graphs}\label{sec:motivation}
In this section, we formalize the analogy between TD versus MC estimation and synthetic gradients versus backpropagation on a layered DAG. 
We begin by fixing some general notation.
Vectors and matrices use boldface characters, e.g., \(\grad\) and \(\jacob\). Unless otherwise specified, these are random variables. A sample is denoted with a subscript, e.g., \(\grad_i\), while an estimator formed from samples is denoted by a hat, e.g., \(\hat{\grad}\). We write \([n] := \{1,2,\dots,n\}\). Additional notation is introduced as needed.

Consider a neural network modeled by a layered DAG $G=(V,E)$. Let $\node_v$ denote the activation at a neuron $v\in V$. For each layer $l=0,\dots,L$, let $V_l$ denote the set of nodes in that layer. The input-layer (layer $l=0$) receives external inputs, and then each non-input node $v\in V_l$ computes its activation through a (possibly stochastic) differential function based on the collective activations $\invar_v$ of its parent neurons $\parent(v)\subseteq V_{l-1}$.
Let $v_L \in V_L$ denote a terminal node whose activation $\bm y:=\node_{v_L}$ serves as system output. Then, the terminal node $v_L$ receives a global vectorized feedback $\grad_{v_L}$. 
This vectorized feedback unifies several cases. In loss-based learning, one may take $\grad_{v_L} = \nabla_{\bm y}\ell$ for a loss term $\ell$ that may depend on the input data, target data, and the network output $\bm y$. In reward-based learning, the global feedback $\grad_{v_L}$ may be obtained from policy gradient. 

\begin{figure}[t]
    \centering
        \vspace{-2em}
        \includegraphics[width=0.75\linewidth]{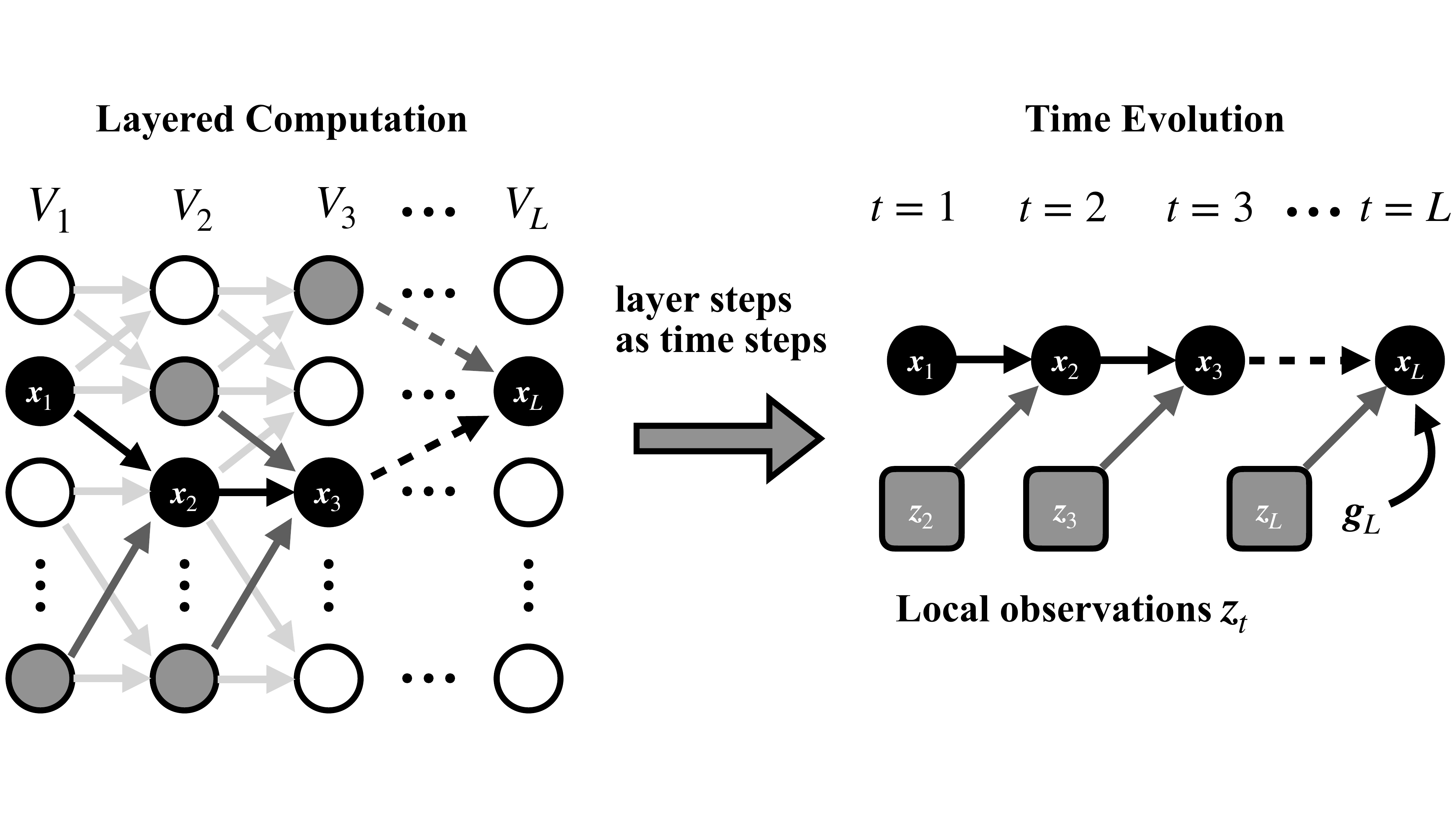}
        \vspace{-1.5em}
        \caption{\textbf{From a network path to a state evolution trajectory.} Left: a single path through a layered network is highlighted in black. Right: after treating layer steps as time steps, the same path is viewed as the trajectory of a state machine. The local observation $\invar_t$ includes node variable $\node_{t-1}$ by default. Shaded nodes denote local observations available at each step.  Dashed arrows indicate information that skips across time. Curved arrow indicates a vectorized feedback  $\grad_L$. 
        }
        \label{fig:motivation}
        \vspace{-1em}
\end{figure}

Backpropagation propagates this terminal feedback backward through the graph by repeated application of the chain rule. 
This recursion is exact, yet restrictive: each intermediate neuron plays a purely passive role. It receives downstream feedback and multiplies it by a Jacobian. 
To see the alternative viewpoint, consider a single path through the network, from a neuron ${v_1}\in V_1$ in the first layer to $v_L\in V_L$ in the final layer. Since we now focus on a single path, we use $\node_l:=\node_{v_l}$ for notational ease, and similarly drop the node indicator ``$v$'' for other relevant variables in the remainder of this section.

Now, let us view its layered computation $v_1\to v_2\to \cdots \to v_L$ as the time evolution of a state machine. At step $t$, the variable $\node_{t}$ is the local state of this machine, while the collection of variables at its parent nodes is viewed as a local observation $\invar_{t}$. 
After $L$ such local transitions, the trajectory receives terminal vectorized feedback $\grad_L$. This process is illustrated in Figure~\ref{fig:motivation}.
Under this view, backpropagation resembles Monte Carlo (MC) credit assignment: for any intermediate step $t$, the signal assigned to $x_t$ is obtained by directly transporting the terminal feedback $\grad_L$ through the chain rule. Thus, the entire global feedback $\grad_L$ is effectively given to the whole trajectory. 

Reward-based learning did not stop at MC, as TD methods can improve sample efficiency by replacing raw sampled future rewards with predicted values~\citep{sutton1988learning, lazaric2010finite}. 
The analogy between TD and synthetic gradients is particularly clean in a finite-horizon reward process with reward $r_L$ only at the terminal time. For value functions, the task is to estimate the value $V_t(\invar_t)=\E[r_L\mid \invar_t]$ from the local information $\invar_t$. For synthetic gradients, the task is to estimate the gradient feedback $\bm{h}_t(\invar_t)= \E[\grad_t \mid \invar_t]$
from local information $\invar_t$. 
Let $\jacob_t^{t+1}$ denote the transposed local Jacobian that maps feedback from step $t+1$ to step $t$, and let
$\jacob_t^L := \jacob_t^{t+1}\jacob_{t+1}^{t+2}\cdots \jacob_{L-1}^{L}$
be the full transposed Jacobian from step $L$ back to step $t$. Then the four estimators can be illustrated as follows.
\[
\renewcommand{\arraystretch}{1.35}
\begin{array}{rcl@{\quad }rcl}
\textbf{MC:}
& \widehat V_t^{\mathrm{MC}}(\invar_t)
&\leftarrow r_L
&
\textbf{TD:}
& \widehat V_t^{\mathrm{TD}}(\invar_t)
&\leftarrow \widehat V_{t+1}^{\mathrm{TD}}(\invar_{t+1}),
\qquad
\widehat V_t^{\mathrm{TD}}(\invar_t)\approx \E[r_L \mid \invar_t],
\\
\textbf{BP:}
& \widehat \grad_t
&\leftarrow \jacob_t^L \, \grad_L
&
\textbf{SG:}
& \widehat {\bm{h}}_t(\invar_t)
&\leftarrow \jacob_t^{t+1}\widehat {\bm{h}}_{t+1}(\invar_{t+1}),
\qquad
\widehat {\bm{h}}_t(\invar_t)\approx \E[\grad_t \mid \invar_t].
\end{array}
\]

This extends the prior TD--SG analogy~\citep{jaderberg2017decoupled,pemberton2024bpmathbflambda} from a single chain of computations to a generic layered DAG. As we later show in Section~\ref{sec:condition}, this generalization is necessary for discovering conditions under which synthetic gradients can improve sample efficiency. However, let us first formalize the gradient-estimation problem and metrics for sample efficiency.

\section{Problem Formulation: Gradient Propagation and Gradient Estimation}\label{sec:theory}

In this section, we formalize gradient propagation and define the estimation criterion used throughout the paper. W\textit{e begin by extending the notation from the previous section.} Throughout the paper, we assume that all random variables have finite second moments, so that their mean squared errors are well defined.

Consider a layered DAG $G=(V,E)$ with node-wise random variables. For each layer $l\in\{0,\dots,L\}$, node $(l,i)\in V_l$ denotes the $i$-th node in layer $l$, where $i\in [m_l]$, and $V=\cup_{l=0}^L V_l$. Let $\node_v\in\sR^{d_v}$ denote the random variable at node $v$, for some dimension $d_v$. The input-layer variables $\{\node_{0,i}\}_{i=1}^{m_0}$ are drawn from an external distribution. For each layer $l\in [L]$ and each node $v\in V_l$, the node variable is generated by a differentiable function $f_v$ with a local parameter $\para_v$:
\begin{align}
    \node_v = f_v(\invar_v;\para_v),
    \qquad
    \text{where }
    \invar_v:=\bigl((\node_{v^-})_{v^-\in \parent(v)},\, \xi_v\bigr).
\end{align}
Here $\parent(v)$ denotes the set of parent nodes of $v$, and $\xi_v$ is an optional auxiliary random variable, independent of everything else, that governs the stochastic behavior of node $v$. For input nodes we may define $\invar_{(0, i)}=\node_{(0, i)}$. 
For each terminal node $(L,i)$, we are given a feedback vector $\grad_{(L,i)}\in \sR^{d_{(L,i)}}$, living in the same Euclidean space as $\node_{L,i}$. This vectorized feedback is the starting point for gradient propagation.

We next define the backward propagation of feedback vectors. For each $l\le L-1$ and each node $v\in V_l$, let $\child(v)\subseteq V_{l+1}$ denote the set of children of $v$. The backpropagated gradient at node $v$ is defined recursively by
\begin{align}
    \grad_v
    =
    \sum_{v^+\in \child(v)} \jacob_v^{\,v^+}\grad_{v^+},
    \qquad
    \text{where }
    \jacob_v^{\,v^+}
    :=
    \left(\frac{\partial \node_{v^+}}{\partial \node_v}\right)^\top .
\end{align}
Here $\frac{\partial \node_{v^+}}{\partial \node_v}$ is the Jacobian of the local computation at node $v^+$, and it is generally a random matrix because it depends on the $f_v$'s input $\invar_v$, which is a random variable.
This recursion is the usual backpropagation: each node aggregates the feedback from its children under the chain rule. 

Synthetic gradients replace part of the exact incoming gradient with a locally generated one. \textbf{\textit{We use a tilde to denote quantities associated with synthetic gradient propagation}}, e.g., $\tilde{\grad}$. We formalize this in the following definition.

\begin{definition}[\textbf{Synthetic Gradient Propagation}]\label{def:sgprop}
For each edge $v\to v^+$ with $v\in V_l$, $l\le L-1$, and $v^+\in \child(v)$, let $\synth_v^{\,v^+}\in \sR^{d_v}$ denote a synthetic gradient generated locally by node $v^+$ for node $v$, and let $\weight_v^{\,v^+}\in [0,1]$ be a mixing weight. We define the propagated feedback recursively by
\begin{align}
    \sgrad_v
    =
    \sum_{v^+\in \child(v)}
    \Bigl(
        \weight_v^{\,v^+}\jacob_v^{\,v^+}\sgrad_{v^+}
        +
        \bigl(1-\weight_v^{\,v^+}\bigr)\synth_v^{\,v^+}
    \Bigr),
\end{align}
with terminal condition $\sgrad_v=\grad_v$ for all $v\in V_L$.
\end{definition}

When all mixing weights $\weight=1$, it reduces to ordinary backpropagation. When all $\weight=0$, backward propagation is fully synthetic and can be decoupled, allowing each node to update without waiting for downstream gradients. Intermediate values allow one to  trade off bias and variance.

The propagated feedback at node $v$, whether exact or synthetic, induces a parameter gradient at the local parameter $\para_v$:
\begin{align}
    \sparagrad_v^\para
    =
    \jacob_v^{\,\para}\sgrad_v,
    \qquad 
    \paragrad_v^\para
    =
    \jacob_v^{\,\para}\grad_v,
    \qquad
    \text{where }
    \jacob_v^{\,\para}
    :=
    \left(\frac{\partial \node_v}{\partial \para_v}\right)^\top.
\end{align}

We now define the criterion of estimator quality used throughout the paper. Let $\hsparagrad_v^\para$ be an estimator of the parameter gradient at node $v$, e.g., the empirical mean of i.i.d.\ samples. We evaluate estimators by mean squared error (MSE) relative to the true target $\E\,\paragrad_v^\para$.

\begin{definition}[\textbf{Gradient Estimator Mean Squared Error}]\label{def:mse}
Given an estimator $\hsparagrad_v^\para$ of the parameter gradient at node $v$, its MSE is
\begin{align}
    \smse_v^2
    :=
    \E\,\bigl\|
        \hsparagrad_v^\para - \E\,\paragrad_v^\para
    \bigr\|_2^2.
\end{align}
When the estimator uses only backpropagation, we write $\hparagrad_v^\para$ and $\mse_v^2$ for the estimator and its MSE.
\end{definition}

This is a classic criterion to capture sample complexity. A smaller MSE means fewer samples are needed to estimate the target gradient accurately. One may derive sample complexity bounds from $\smse_v^2$ using Chebyshev-type inequalities. Stronger assumptions, e.g., sub-Gaussianity, would yield sharper concentration bounds. In this paper, we impose no tail assumptions and work directly with MSE due to its exactness. 

To analyze $\smse_v^2$, we decompose it into the following terms.

\begin{definition}\label{def:variance_terms}
For a fixed node $v$, define the expected conditional variance ($\scmse_v^2$), the variance of the conditional mean ($\serror_v^2$), and the squared bias ($\sbias^2_v$):
\begin{align}
    \scmse_v^2
    :=
    \E\,\bigl\|
        \sparagrad_v^\para - \E[\sparagrad_v^\para\mid \invar_v]
    \bigr\|_2^2,\quad
    \serror_v^2
    :=
    \E\,\bigl\|
        \E[\sparagrad_v^\para\mid \invar_v] - \E\,\sparagrad_v^\para
    \bigr\|_2^2,\quad
    \sbias_v^2
    :=
    \bigl\|
        \E\,\sparagrad_v^\para - \E\,\paragrad_v^\para
    \bigr\|_2^2.
\end{align}
Similarly, we define the corresponding backpropagation quantities $\cmse_v^2$ and $\error_v^2$. By definition, backpropagation is unbiased, so $\bias_v^2=0$.
\end{definition}

Intuitively, the quantity $\scmse_v^2$ measures the gradient signal that cannot be predicted from the node's local observations. Then, $\serror_v^2$ measures the part of the signal that is predictable from $\invar_v$. Finally, $\sbias_v^2$ measures the bias introduced by the gradient estimator. %

\begin{lemma}\label{lemma:decomposition}
Let $\hsparagrad_v^\para = \frac{1}{n}\sum_{i\in[n]}\sparagrad_{v,i}^\para$
be the empirical mean of $n$ i.i.d.\ samples. Then
\begin{align}
    \smse_v^2
    =\frac{1}{n}\left( \scmse_v^2+\serror_v^2\right)
    +
    \sbias_v^2.
\end{align}
In particular, for a conditionally unbiased estimator, i.e., $\E\,[\sparagrad_v^\para\mid \invar_v]=\E\,[\paragrad_v^\para\mid \invar_v]$, we have
\begin{align}
    \smse_v^2
    = \frac{1}{n}\left( \scmse_v^2+\error_v^2\right).
\end{align}
\end{lemma}

Lemma~\ref{lemma:decomposition} makes the bias-variance tradeoffs explicit, which is expected from the MC vs TD analogy. The next section is devoted to characterizing when this tradeoff favors synthetic gradient propagation.

\section{Conditions for When Synthetic Gradient Improves Sample Efficiency}\label{sec:condition}

We now turn to the central question. \emph{When can synthetic gradients improve sample efficiency?} 

To identify the conditions, we need to introduce the following definitions needed for our discovery. 
Define the sibling set (self-included) of a node $v\in V_{l}$ by
\begin{align}
    \sib(v) := \{v'\in V_{l}:\parent(v')\cap \parent(v)\neq \varnothing\}.
\end{align}
\begin{definition}[\textbf{Extended Ancestor Set}]\label{def:ances}
For a non-input node $v\in V_l$, we define its sibling-extended ancestor set by the following recursion. 
\begin{align}
    \ean(v):=\left( \cup_{v^-\in \parent(\sib(v))} \ean(v^-) \right)\cup \sib(v), 
\end{align}
i.e., $\ean(v)$ contains the siblings of $v$ and the extended ancestors of its siblings' parents, including its own ancestors.  We define the sibling-extended ancestor set of any input node to be empty.
\end{definition}
For any set of nodes $U\subset V$, define $\bm Z_{U}:=\{\invar_u:u\in U\}$ to be the collection of the corresponding local observables. 

Under the MSE criterion, we discover that the sample efficiency is governed by two conditions.

\begin{definition}\label{def:conditions}
We introduce conditions $(A)$ and $(B)$.

\textbf{(A): Uncertainty sources in the gradient feedback}: (A) holds if \textbf{any} of the following conditions is true.
\begin{itemize}
    \item[(A.1)] \textbf{Stochastic terminal gradient feedback.}  
    The external gradient feedback $\grad_{(L,i)}$ to some terminal node $(L, i)\in V_L$ is not almost surely a deterministic function of its output $\node_{(L, i)}$.

    \item[(A.2)] \textbf{Stochastic network computation.}  
    Some node $v\in V$ performs stochastic computation.

    \item[(A.3)] \textbf{Sparsely connected network.}  
    Some non-input node $v\in V_l$ is sparsely connected to its previous layer, i.e., $\parent(v)\neq V_{l-1}$.

\end{itemize}
\textbf{(B): Conditional mean sufficiency.} Every non-input node's local observation $\invar_v$ is sufficient for its conditional mean gradient $\E[\grad_v\mid \invar_v]$ relative to its extended ancestors' information. Formally, for any node $v\in V_l$ and any ancestor subset $U_l\subseteq \ean(v)$, we have $\E[\grad_v\mid \invar_v]=\E[\grad_v\mid \invar_v,\bm{Z}_{U_l}]$ almost surely.

Note that both conditions are stated in terms of the raw gradient $\grad_v$, not in terms of any particular propagation rule. They are properties of the problem, not of the estimator.
\end{definition}

Condition $(A)$ ensures that there are uncertainty sources in the incoming gradient feedback at some node, and hence room for variance reduction. The first two cases are direct, while the third case is subtler: even when the terminal feedback and all node computations are deterministic, a sparsely connected node may still see randomness in its incoming gradient due to a partial observation of its upstream states.

Then, condition $(B)$ makes the comparison sharp. First, when $(A)$ fails, it implies that there is no realized uncertainty left in the incoming gradient. Second, when $(A)$ holds, it ensures an oracle synthetic gradient to be unbiased. Condition $(B)$ can also be realized concretely as a form of specialization: an expert node whose functions are largely independent of others. This interpretation will become more transparent in Section~\ref{sec:example}.

We note that $(B)$ can be viewed as a weaker version of the Markov assumption:  standard Markov assumption requires future variables to be conditionally independent of the full history given the current observables, while condition $(B)$ requires only that the conditional mean of the incoming gradient feedback be determined by the node's current observables. Existing sample-complexity analyses comparing TD and MC estimators typically impose Markovian state representations, often in tabular or linear settings~\citep{kearns2000bias, lazaric2010finite, pmlr-v75-bhandari18a, cheikhi2023statistical, cheng2024surprising}.

In the next section, we show that these two conditions yield a clean separation between backpropagation and synthetic gradient propagation under an idealized oracle synthetic gradient estimator. As condition $(B)$ and the oracle assumption can be strong, we then relax both of them to develop a more general comparison in Section~\ref{subsec:generalization}.

\subsection{Clean Comparisons Under Idealized Conditions}\label{subsec:clean}

First, we establish the results for the optimality of backpropagation.

\begin{theorem}\label{thm:bp-optimal}
Under $\neg (A) \wedge (B)$, every gradient estimator $\hat{\sgrad}^\para$ formed by empirical means of i.i.d. conditionally unbiased gradient samples (i.e., $\E\,[\sparagrad_v^\para\mid \invar_v]=\E\,[\paragrad_v^\para\mid \invar_v]$) has its estimation MSE greater or equal to that of the empirical-mean backpropagation estimator: $\smse_v^2\geq \mse_v^2$ for any non-input nodes $v\in V$.

\end{theorem}

 In other words, under $(B)$, it is necessary for $(A)$ to hold for synthetic gradient propagation to be more sample efficient: one needs a source of uncertainty for any improvement over exact backpropagation to even be possible. %
We first consider an idealized synthetic gradient oracle, i.e., the prediction of the incoming gradient based on local observables:
\begin{align}
    \synth_{v}^{\, v^+}=\jacob_{v}^{\, v^+}\synth_{v^+},\quad \text{where }\ \synth_{v^+}:=\E\,[\sgrad_{v^+}\mid \invar_{v^+}]. \label{eq:synth-def}
\end{align}

Note that Conditions $(A.1)$--$(A.3)$ each correspond to a source variable $\bm \chi$ that introduces realized uncertainty  at some node $v\in V$:
\begin{itemize}[leftmargin=2em]
    \item under \textnormal{(A.1)}, \(\bm \chi\) is an exogenous source of terminal feedback randomness;
    \item under \textnormal{(A.2)}, \(\bm \chi\) is an internal source of stochastic computation;
    \item under \textnormal{(A.3)}, \(\bm \chi\) is a collection of node variables unobserved by $v$ due to sparse
    connectivity.
\end{itemize}

We impose the following technical assumption to exclude degenerate cases that would otherwise obscure the main message of our theorem.
\begin{assumption}[\textbf{Non-degenerate uncertainty source}]\label{assume:1}
When condition $(A)$ holds, at least one of its uncertainty sources is non-degenerate: there exist non-terminal $v\in V$ and a source variable $\bm{\chi}$ corresponding to
$(A.1)$--$(A.3)$, such that  $\E\left[\grad_v^\theta \mid \invar_v, \bm Z_{\child(v)}\right]
        \neq
        \E\left[\grad_v^\theta \mid \invar_v, \bm Z_{\child(v)} ,\bm{\chi}\right]$ with positive probability.
\end{assumption}
This rules out the degenerate case where the source $\bm \chi$ does not introduce effective uncertainty: it should remain effective to some parameter gradient conditioned on local observables $\{\invar_v, \bm Z_{\child(v)} \}$.

\begin{theorem}\label{thm:synth-better}
    Under $(A)\wedge (B)$ and Assumption~\ref{assume:1}, consider synthetic gradient defined in \eqref{eq:synth-def} and all mixing weights $\lambda=0$. Then, for all non-input $v\in V$: $\smse^2_v\leq \mse^2_v$, and there exists some $v$: $\smse^2_v< \mse^2_v$.
\end{theorem}

Therefore, under condition $(B)$, condition $(A)$ cleanly characterizes when the oracle synthetic gradients can achieve better gradient-estimation MSE. In the next section, we relax both the condition $(B)$ and the oracle assumption, moving to a more general regime where there is indeed a bias-variance tradeoff achieved by non-trivial mixing weights $\lambda$.

\subsection{Comparison Under Generalized Conditions }\label{subsec:generalization}

\begin{definition}[\textbf{(B'): Generalized conditional mean sufficiency}]\label{def:B-eps}
Every non-input node's local observation $\invar_v$ is approximately sufficient for its conditional mean gradient $\E[\grad_v\mid \invar_v]$ relative to its extended ancestors' information. Formally,  $\forall v\in V_l$ and any ancestor subset $U_l\subseteq \ean(v)$, there is a constant $\Berror>0$: 
\begin{align}
    \left\|\E[\grad_v\mid \invar_v] - \E[\grad_v\mid \invar_v,\bm{Z}_{U_l}] \right\|_2 \leq \Berror,\qquad a.s.
\end{align}

\end{definition}

To relax the synthetic gradient oracle setting, we consider a generic synthetic gradient generator. Specifically, for some constant $\epsilon'>0$ and for each edge $v\to v^+$, let $\synth_v^{\,v^+}=\jacob_v^{\,v^+}\synth_{v^+}$ satisfy
\begin{align}
    \left\|
        \synth_{v^+}
        -
        \E[\sgrad_{v^+}\mid \invar_{v^+}]
    \right\|_2
    \leq \epsilon',
    \qquad a.s.
    \label{eq:approx-synth-def}
\end{align}

Thus, we allow two sources of approximation: imperfect local sufficiency through $(B')$, and imperfect synthetic gradient generator through \eqref{eq:approx-synth-def}. We next study how these approximations change the comparison between backpropagation and synthetic gradient propagation, and in particular how they lead to nontrivial optimal choices of the mixing weights $\weight$. In this subsection, for all node $v$, we consider $\weight_v=\weight_v^{\, v^+}$ for any of its children $v^+\in \child(v)$. This means that each node treats the synthetic gradient from their children equally.

If bias is present, the optimal mixing weight must trade bias against variance. Fix a non-terminal node $v\in V_l$. The parameter gradient induced by synthetic gradient propagation can be written as
\begin{align}
\sgrad^{\para}_v&=\weight_{v}\underbrace{\jacob^\para_v\sum_{v^+\in \child(v)} \jacob_{v}^{\, v^+} \sgrad_{v^+}}_{=:\sss_v}  + (1-\weight_{v}) \underbrace{\jacob^\para_v\sum_{v^+\in \child(v)}\synth_{v}^{\, v^+}}_{=:\qqq_v}=\weight_{v}\sss_v + (1-\weight_{v})\qqq_v.
\end{align}
Here $\sss_v$ is the component obtained by backpropagating gradient feedback through Jacobians, whereas $\qqq_v$ is the component obtained from synthetic gradients.

Let $\hat\sss_v$ and $\hat\qqq_v$ denote the empirical means of $\sss_v$ and $\qqq_v$, respectively, based on $n$ i.i.d.\ samples. Then the MSE $\smse^2_v$ of the empirical mean estimator and the optimal mixing weight $\weight^\star$ is
\begin{align}
    \smse^2_v=\E\,\| \hsparagrad^\para_v-\E\, \grad_{v}^\para\|_2^2 =\E\,\left\| \weight_{v}\hat\sss_v + (1-\weight_{v})\hat\qqq_v-\E\,\grad_{v}^\para\right\|_2^2 ,\quad\text{and}\  \ \weight_v^\star = \argmin_{\weight_v\in [0, 1]}\ \smse^2_v.
\end{align}

    If $\sss_v = \qqq_v$ almost surely, then $\weight_v$ is irrelevant and we pick $\weight_v^\star=0$.  Otherwise, standard quadratic optimization shows that there is a unique solution:
    \begin{align}
        \weight^\star_v = \frac{\E\,\| \hat\qqq_v-\E\,\grad_{v}^\para\|_2^2-\E\,\langle \hat\qqq_v-\E\,\grad_{v}^\para,\hat\sss_v- \E\, \grad_{v}^\para\rangle}{\E\,\| \hat\qqq_v-\hat\sss_v\|_2^2}\quad \text{projection onto }\ [0, 1]. \label{eq:lambda-star}
    \end{align}
    Thus $\weight_v^\star<1$ if incorporating a nonzero synthetic component improves over pure backpropagation at node $v$.
    The next theorem gives a  sufficient condition for this to happen.

\begin{theorem}\label{thm:general-benefits}
    Under condition (B') and \eqref{eq:approx-synth-def}, assume bounded jacobians $\forall v, v^+:\ \max\{\|\jacob_{v}^{\, v^+}\|_2, \|\jacob_{v}^{\para}\|_2\}\leq J$. 
    A sufficient condition for $\weight^\star_v<1$ is:
   \begin{align}
        \frac{\tr\Cov(\sss_v)-\tr\Cov(\qqq_v)}{n}>\left(J^2(\Berror+\epsilon')\sum_{v^+\in \child(v)} (1+c_{v^+})\right)^2
    \end{align}
    where $\forall u\in V: c_u$ is defined recursively via 
    \begin{align}
        c_{u}=J\sum_{u^+\in \child(u)}(1-\weight_{u}+c_{u^+}),\quad \text{and for terminal nodes } \ \forall u\in V_L: c_u=0. 
    \end{align}
    Here $\tr\Cov(\cdot)$ denotes the trace of the covariance matrix, i.e., the total variance of a random vector, i.e., $\tr\Cov(\grad)=\E\, \|\grad-\E \,\grad\|_2^2$.

\end{theorem}

Theorem~\ref{thm:general-benefits} has a simple interpretation. The left-hand side is the variance reduction obtained by using the synthetic component $\qqq_v$ instead of the backpropagated component $\sss_v$, after averaging over $n$ samples. The right-hand side is the bias penalty induced by synthetic gradients. Thus, a synthetic mixture can be beneficial whenever the variance reduction dominates the induced bias.

A question remains: how large can this sample efficiency advantage be when the synthetic-gradient generator must itself be estimated from data rather than treated as given? In the next section, we answer this question by constructing examples in which the advantage can be arbitrarily large.

\section{An Example of Finite-Partition Expert Network }\label{sec:example}

We consider a two-hidden-layer architecture in which the first hidden layer a single node, and the second hidden layer has nodes $V_2=\{(2, j)\}_{j\in [m_2]}$. In the rest of the section, we denote $j:=(2, j)$. We focus on gradient estimation on the node $(1,1)\in V_1$ in the first layer, and thus in this section \emph{we omit the subscript when referring to this target node}. For example, $\grad^\para$ refers to the backpropagated gradient on this target node $(1,1)\in V_1$.

The first layer node computes a $d$-dim feature representation $\node=f(\invar; \para)\in \sR^d$, where $\invar$ is the input sampled from an external source  (i.e., from layer $l=0$).  The second layer contains $m_2$ expert nodes, where each expert $j\in[m_2]$ specializes in a $k$-dimensional subspace of the feature vector $\node$. Formally, each expert $j$ effectively acts on $\bm u_{j}=P_{j}\node$, where $P_{j}:\sR^d\to \sR^k$ is a projection. Each expert outputs 
$y_j=f_{j}(\bm u_j)\in\sR$. Then, the final node collects the vector of expert outputs $\bm y=(y_1,\ldots,y_{m_2})\in\sR^{m_2}$.

Each expert node $j$ solves an independent sub-task, such that the global loss function decomposes across the experts. The global loss function and its gradient to expert node $j$ are respectively $\gL:=
    \sum_{j=1}^{m_2}\ell_j(\bm{u}_j,y_j)$, and $\grad_j
    :=
    \frac{\partial \gL}{\partial y_j}
    =
    \frac{\partial \ell_j(\bm u_j,y_j)}{\partial y_j}$.
We compare ordinary backpropagation and synthetic-gradient propagation under three sources of stochastic expert feedback, each inspired by a sub-condition in condition $(A)$. 
\begin{itemize}[leftmargin=1em]
    \item \textbf{Case 1: noisy terminal feedback.}
    The external feedback carries i.i.d. noise:
$
    \hat\grad_j=\grad_j+\xi_j,
    \qquad
    \E[\xi_j]=0,\qquad \E[\xi^2_j]=\sigma^2.$
    \item \textbf{Case 2:  stochastic expert activation.}  
    Each expert is activated with probability $p$ independently every time. The gradient feedback is thus    $\hat\grad_j=\frac{\eta_j}{p}\grad_j, \qquad  \eta_j\sim\mathrm{Bernoulli}(p)$.
\item \textbf{Case 3: partial observation.}  In this case, the terminal node sees some other variables $\invar'_j$, independent across $j$, such that the global loss term $\gL'=
    \sum_{j=1}^{m_2}\ell_j(\bm{u}_j,y_j, \invar'_j)$, and its gradient to expert $j$ becomes
    $\hat\grad_j
    =\frac{\partial \gL'}{\partial y_j}=
    \frac{\partial \ell_j(\bm u_j,y_j, \invar'_j)}{\partial y_j}$.
\end{itemize}

The synthetic gradient from expert $j$ to the target node $(1,1)$ needs to be estimated: $\hat\synth^j(\bm{u}_j) \approx \jacob^j\E[\grad_j\mid \invar_j]$.
To characterize sample efficiency when the synthetic gradient estimator is itself learned from data, we restrict attention to discrete inputs. This yields a tabular setting, analogous to those commonly used in sample-complexity analyses of TD and MC estimators~\citep{cheikhi2023statistical, cheng2024surprising}.

Let $\invar\in \gX$ denote the input variable to our target node $(1,1)\in V_1$. Let the input space $\gX$ be partitioned into $R=2^d$ sets $\gX=\bigsqcup_{r=1}^R \gX_r$.
We write $\invar_r$ for the representative input in partition $r$, and suppose the input variable $\invar$ is sampled from a uniform distribution over these representative inputs. The target node's function is piecewise parameterized:
\begin{align}
    f(\invar; \para) = \sum_{r=1}^R
    \bm 1\{\invar \in\gX_r\}f_{\para_r}(\invar), \quad \text{where}\ \ \para=(\para_1,\ldots,\para_R). \label{eq:def:piecewise}
\end{align}
Without loss of generality, let $\node_r=f_{\para_r}(\invar_r)\in\{\pm1\}^d$.
We assume that the map $r\mapsto \node_r$ is one-to-one and onto
$\{\pm1\}^d$. Therefore, drawing $r$ uniformly from $[R]$ is equivalent
to drawing $\node_r$ uniformly from $\{\pm1\}^d$.
Let each expert projection $P_j$ selects exactly $k$ coordinates of $\node_r$.
Thus $\bm{u}_{j,r}:=P_j\node_r\in\{\pm1\}^k$.
For any $\bm{u}\in\{\pm1\}^k$, let the equivalent class under this projection be 
\begin{align}
    \mathcal F_j(r) := \{s\in[R]:P_j\node_s=P_j\node_r\}, \qquad \text{and thus }\ \ \forall r\in [R]:\ |\mathcal F_j(r)|=2^{d-k}.
\end{align}

In this example, this state-pooling effect is the source of the sample-efficiency advantage of each expert nodes' synthetic gradient  $\synth^j\approx \jacob^j\E[\hat\grad_j\mid \bm{u}_j]$. We defer the full construction of the example and the theorem statement to Appendix~\ref{sec:detail-example}, and give an informal version below.

\begin{theorem}[Informal version of Theorem~\ref{thm:per-state-mse}]
\label{thm:per-state-mse-informal}
Under the settings described in this section, for all the three cases we have
\begin{align}
    \mse^2= 2^{d-k}\, \smse^2.
\end{align}

\end{theorem}

Thus, in this example, the estimation MSE of backpropagation can be arbitrarily larger than that of synthetic gradient propagation. In particular, the gap scales exponentially with the dimension reduction induced by the expert nodes' subspace specialization.

Our theoretical analysis focuses on gradient-estimation MSE. We discuss its connection to optimization convergence in Appendix~\ref{subsec:discuss-appendix-mse}. In this example of the expert network, Condition $(B)$ (Definition~\ref{def:conditions}) holds by construction. We discuss when this condition can be expected to hold more generally in Section~\ref{subsec:discuss-appendix-cond-B}. We next turn to numerical experiments to investigate the benefits of synthetic gradients.

\section{Experiments}\label{sec:exp}

The experiments serve a single purpose: \emph{to investigate whether synthetic gradients can still provide sample-efficiency benefits in more applied settings where the conditions in Section~\ref{sec:condition} may not hold and the synthetic-gradient predictors must be learned from data.}
We investigate synthetic gradients in three experimental settings\footnote{Code is available at \url{https://github.com/jackyzyb/sgprop}}: from the finite-partition expert network as seen in the previous section, to learning contextual bandits and partially observable RL.

\textbf{Simulation of the finite-partition expert network.}
We first simulate the expert network from Section~\ref{sec:example}.  %
Figure~\ref{fig:main-exp-summary}~(a) shows the estimated MSE ratio for the setting of noisy terminal feedback (Case 1). The simulated MSE ratio grows exponentially with $d-k$, matching the predicted results from Theorem~\ref{thm:per-state-mse}. implementation details and simulation results for other cases are shown in Appendix~\ref{subsec:full-exp-expert}.

\textbf{SparseNet.} Following our theoretical findings, we investigate a sparsely connected neural network architecture, where every neuron can propagate synthetic gradients. For a neuron to perform both non-linear computations and synthetic gradient prediction, a natural idea is to use a soft-gating mechanism where it has shared keys for input matching and separate heads for non-linear transformation and synthetic gradient prediction. Each neuron also estimates their optimal mixing weight according to \eqref{eq:lambda-star}. This architecture is named \texttt{SparseNet}, as detailed in Appendix~\ref{subsec:full-exp-sparsenet}. 

\paragraph{MNIST contextual bandit.}
We next test a contextual bandit task learned with policy gradient. The context is a sampled MNIST image~\citep{lecun2002gradient}, the action is a sampled digit label, and the reward is one for a correct label and zero otherwise; during training the reward is flipped with probability $p_{\mathrm{flip}}=0.4$ to simulate a noisy environment. The SparseNet is a five-layer sparse feedforward network with width $200$ and $5$ parents per neuron. The first layer receives input images, and final layer's output is converted into probability logits for sampling actions. We compare SparseNet with synthetic gradient propagation and backpropagation, and a dense MLP. Figures~\ref{fig:main-exp-summary} (b)\&(c) show that synthetic gradient with adaptive mixing weights improves convergence speed relative to the number of samples seen.  All curves report means over $6$ seeds with standard error of the mean (SEM). Implementation details and ablation on the number of parents, $p_{\mathrm{flip}}$, learning rates, and $\lambda$ dynamics are shown in Appendix~\ref{subsec:full-exp-mnist}. 

\begin{figure}[t]
    \centering
    \begin{subfigure}[b]{0.32\textwidth}
        \centering
        \includegraphics[width=\linewidth]{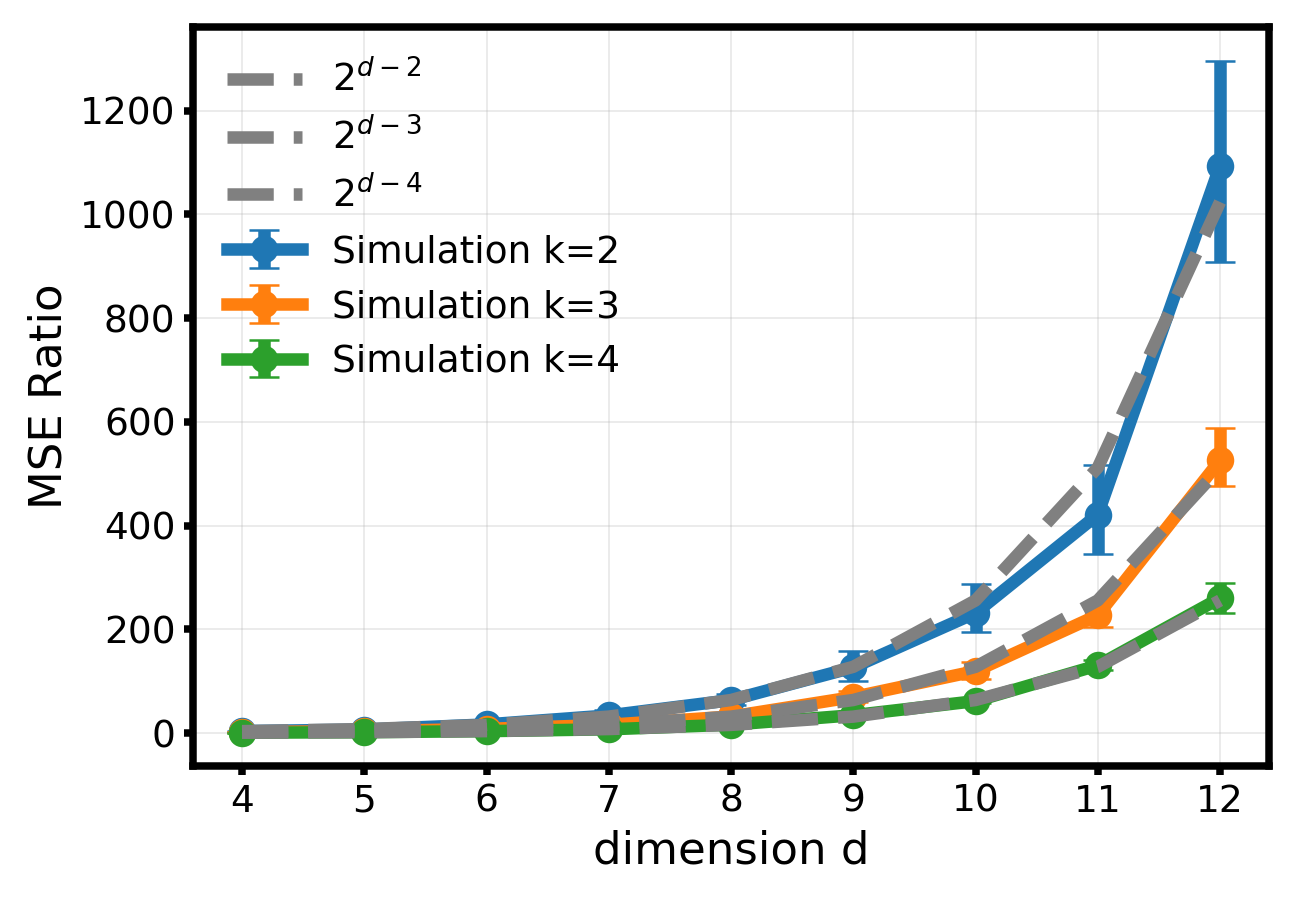}
        \caption{Expert network}
        \label{fig:main-expert-noisy}
    \end{subfigure}
    \begin{subfigure}[b]{0.32\textwidth}
        \centering
        \includegraphics[width=\linewidth]{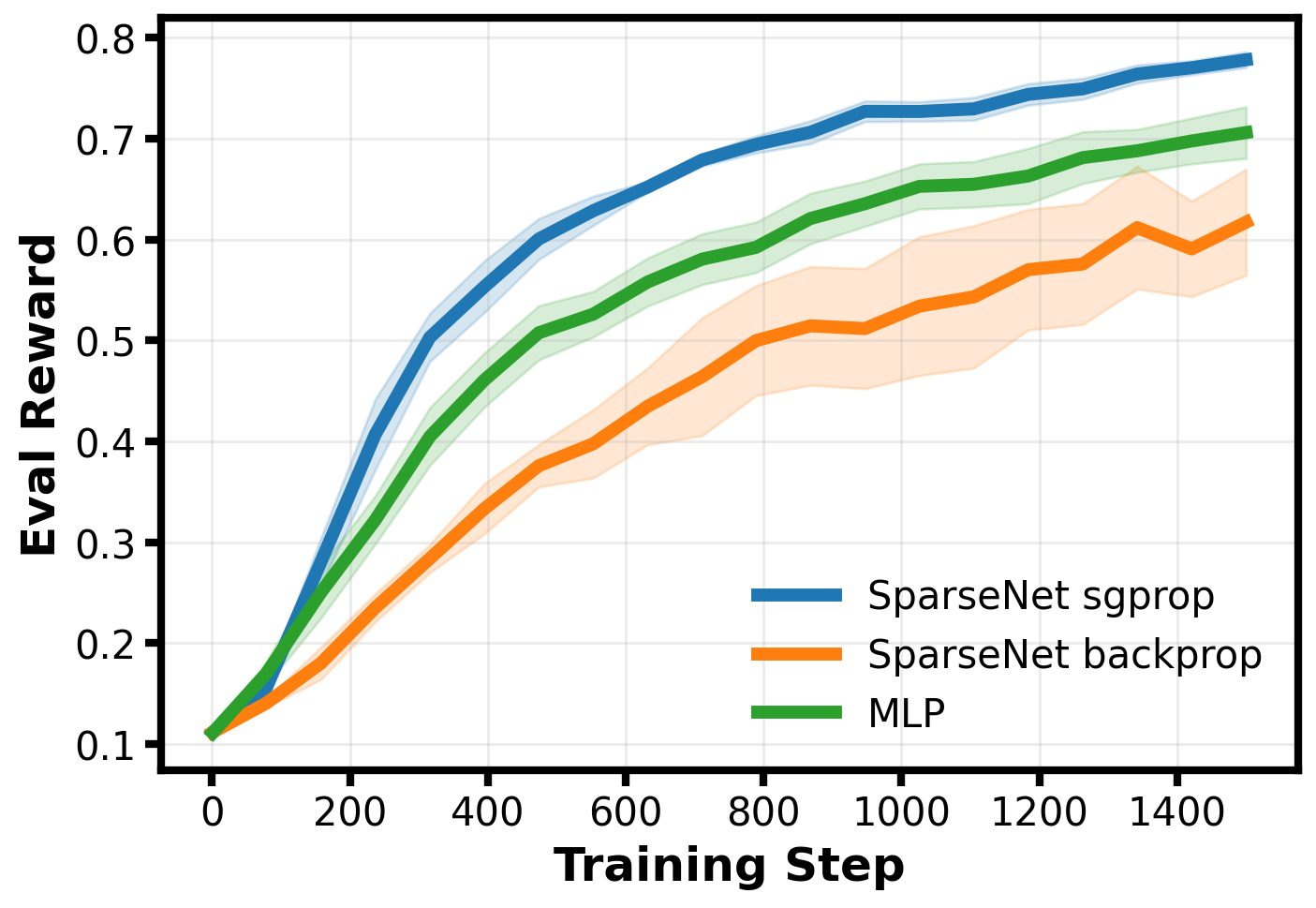}
        \caption{bandit training dynamics}
        \label{fig:main-mnist-training}
    \end{subfigure}
    \begin{subfigure}[b]{0.32\textwidth}
        \centering
        \includegraphics[width=\linewidth]{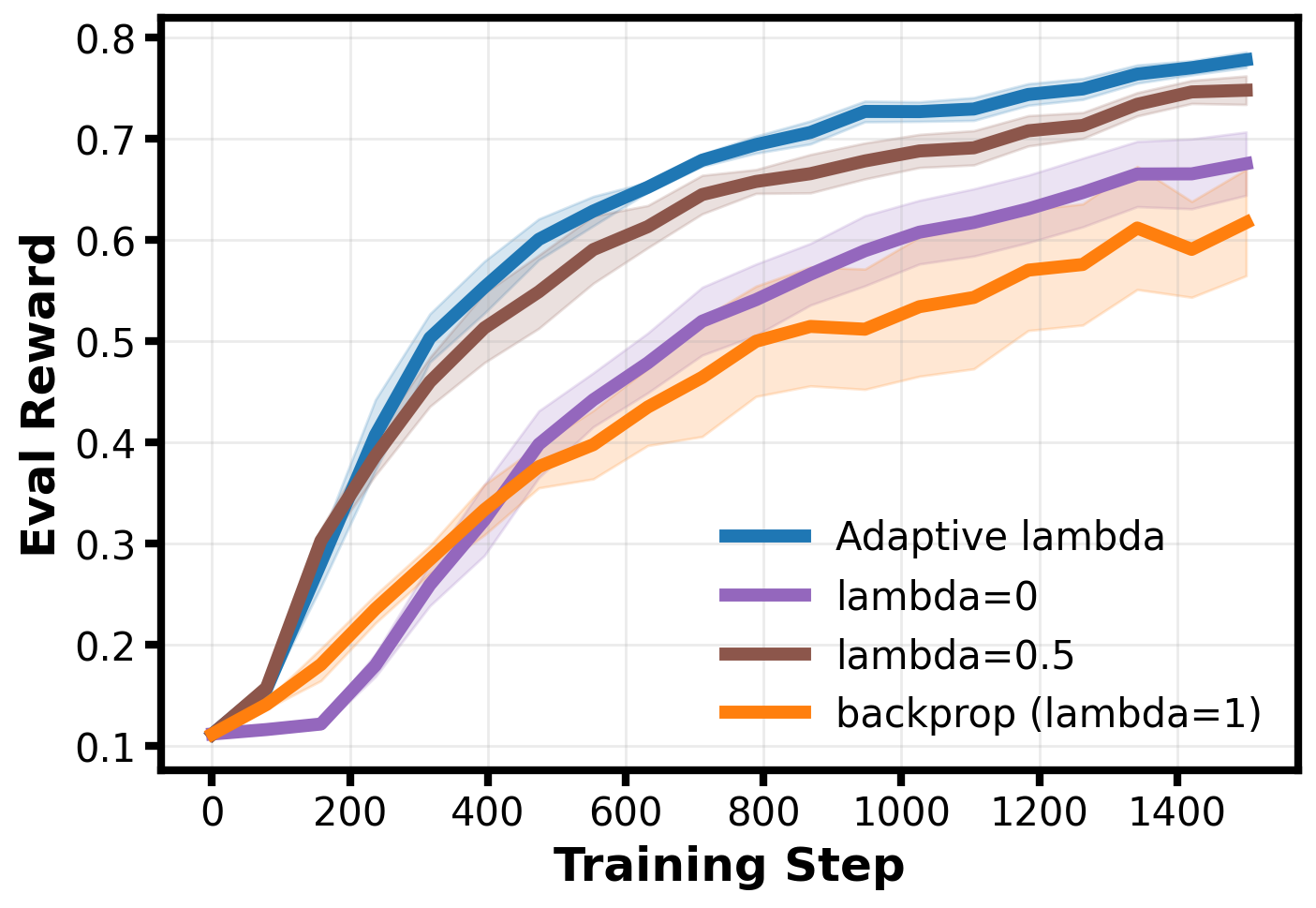}
        \caption{bandit mixing weight}
        \label{fig:main-mnist-theta}
    \end{subfigure}
    \caption{\textbf{Finite-partition and MNIST bandit results.} (a) In the expert simulation, synthetic gradients reduces gradient estimation MSE relative to backpropagation, where the reduction ratio fits the theoretical prediction. (b) In the MNIST contextual bandit, synthetic gradients requires fewer steps (numer of samples seen) to converge. (c) Mixing weights sensitivity in the MNIST contextual bandit experiment.}
    \label{fig:main-exp-summary}
\end{figure}

\paragraph{POPGym reinforcement learning.}
We investigate whether synthetic gradient propagation can help recurrent neural network learning in RL. We use two maze navigation environments from POPGym~\citep{moradpopgym}: \texttt{labyrinth escape} and \texttt{labyrinth explore}, where only adjacent tiles can be observed by the agent. SparseNet is made recurrent by applying the same sparse neuron transition through time: $128$ neurons are each connected to $6$ other neurons and optionally the external input, where the connectivity is uniformly randomly picked and fixed at initialization. During each step, the collection of all neuron outputs is converted into probability logits for sampling actions. Synthetic gradient propagation is applied during backpropagation through time (BPTT) to mix backpropagated gradients with synthetic gradients. We compare SparseNet (sgprop), SparseNet (backprop), GRU~\citep{cho2014learning} and LSTM~\citep{schmidhuber1997long}.  
Figure~\ref{fig:main-popgym-training} shows the four main training curves. All curves report means over $6$ seeds with standard error of the mean (SEM). Implementation details and ablations on learning rates and adaptive mixing weights are presented in Appendix~\ref{subsec:full-exp-popgym}. 

\begin{figure}[t]
    \centering
    \begin{subfigure}[b]{0.24\textwidth}
        \centering
        \includegraphics[width=\linewidth]{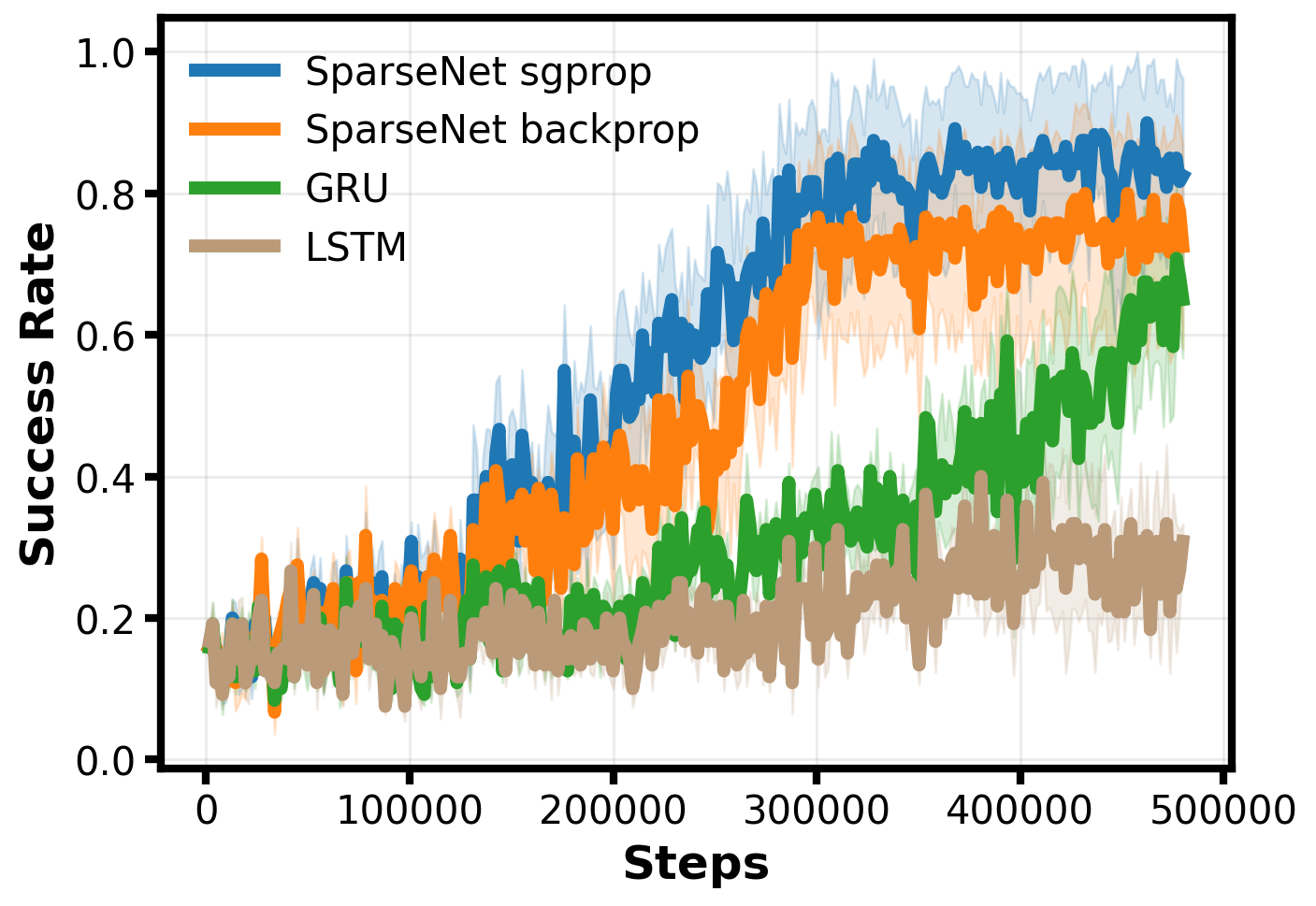}
        \caption{\texttt{escape}, 20 trials.}
        \label{fig:main-popgym-escape20}
    \end{subfigure}
    \begin{subfigure}[b]{0.24\textwidth}
        \centering
        \includegraphics[width=\linewidth]{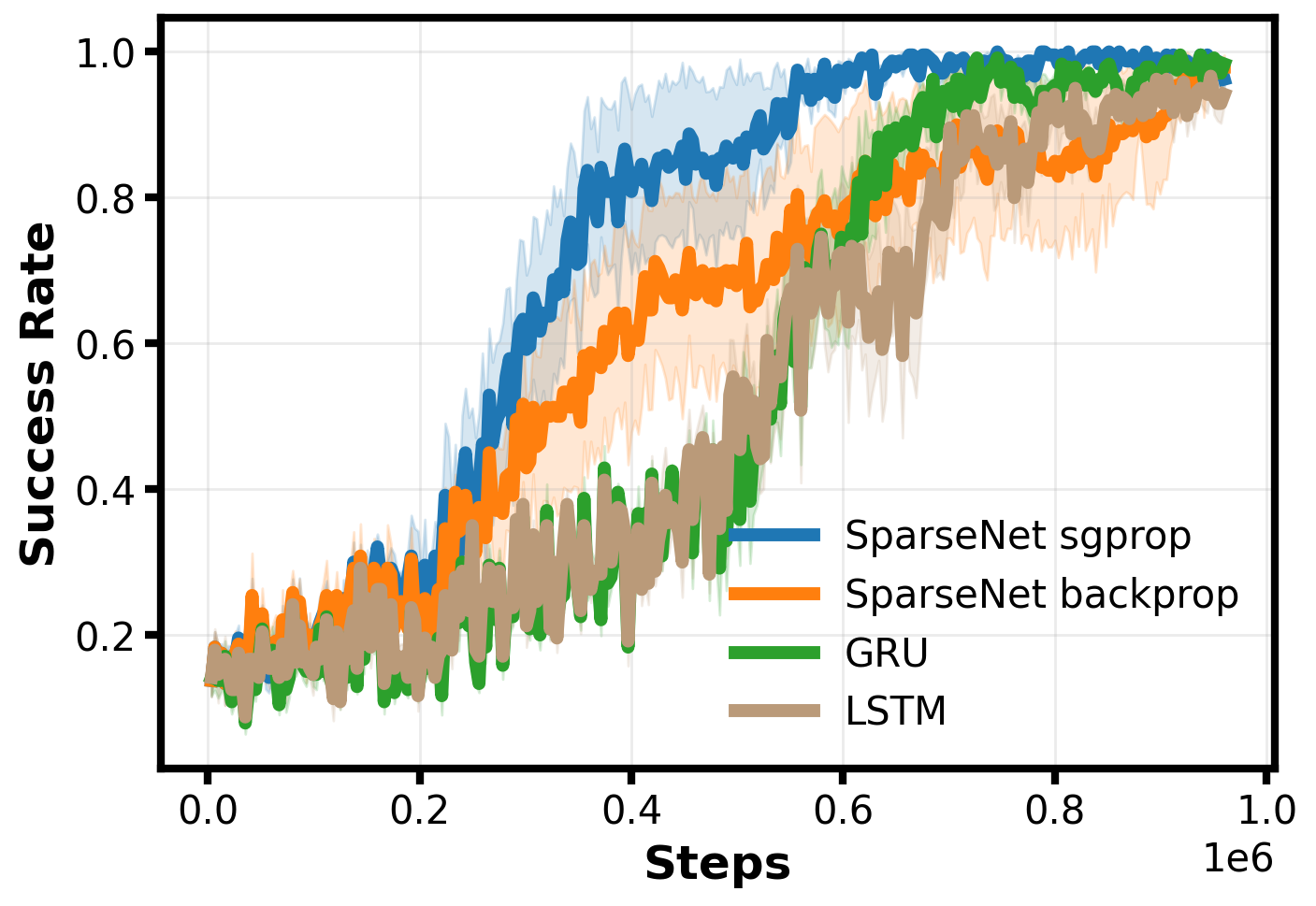}
        \caption{\texttt{escape}, 40 trials.}
        \label{fig:main-popgym-escape40}
    \end{subfigure}
    \begin{subfigure}[b]{0.24\textwidth}
        \centering
        \includegraphics[width=\linewidth]{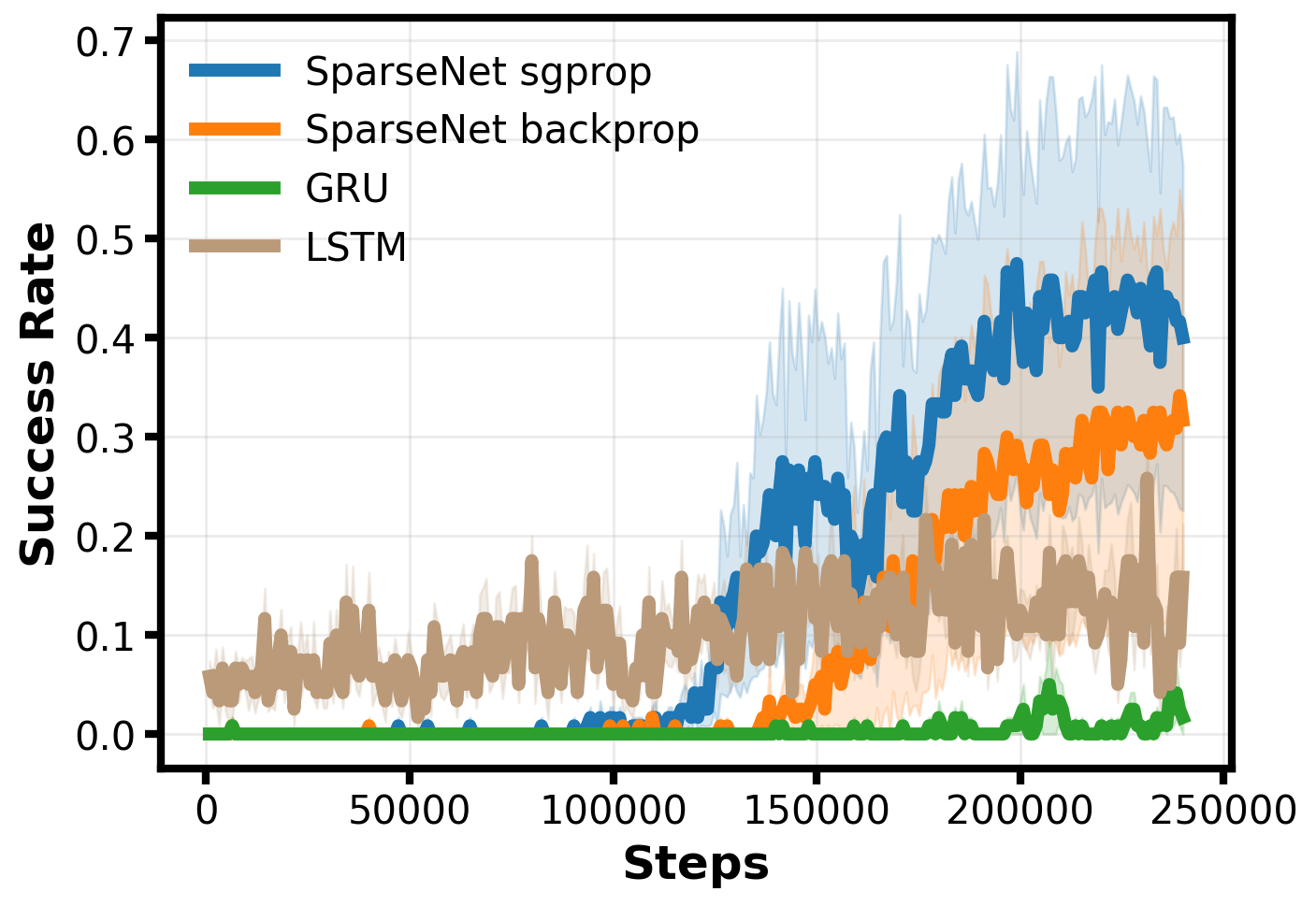}
        \caption{\texttt{explore}, 20 trials.}
        \label{fig:main-popgym-explore20}
    \end{subfigure}
    \begin{subfigure}[b]{0.24\textwidth}
        \centering
        \includegraphics[width=\linewidth]{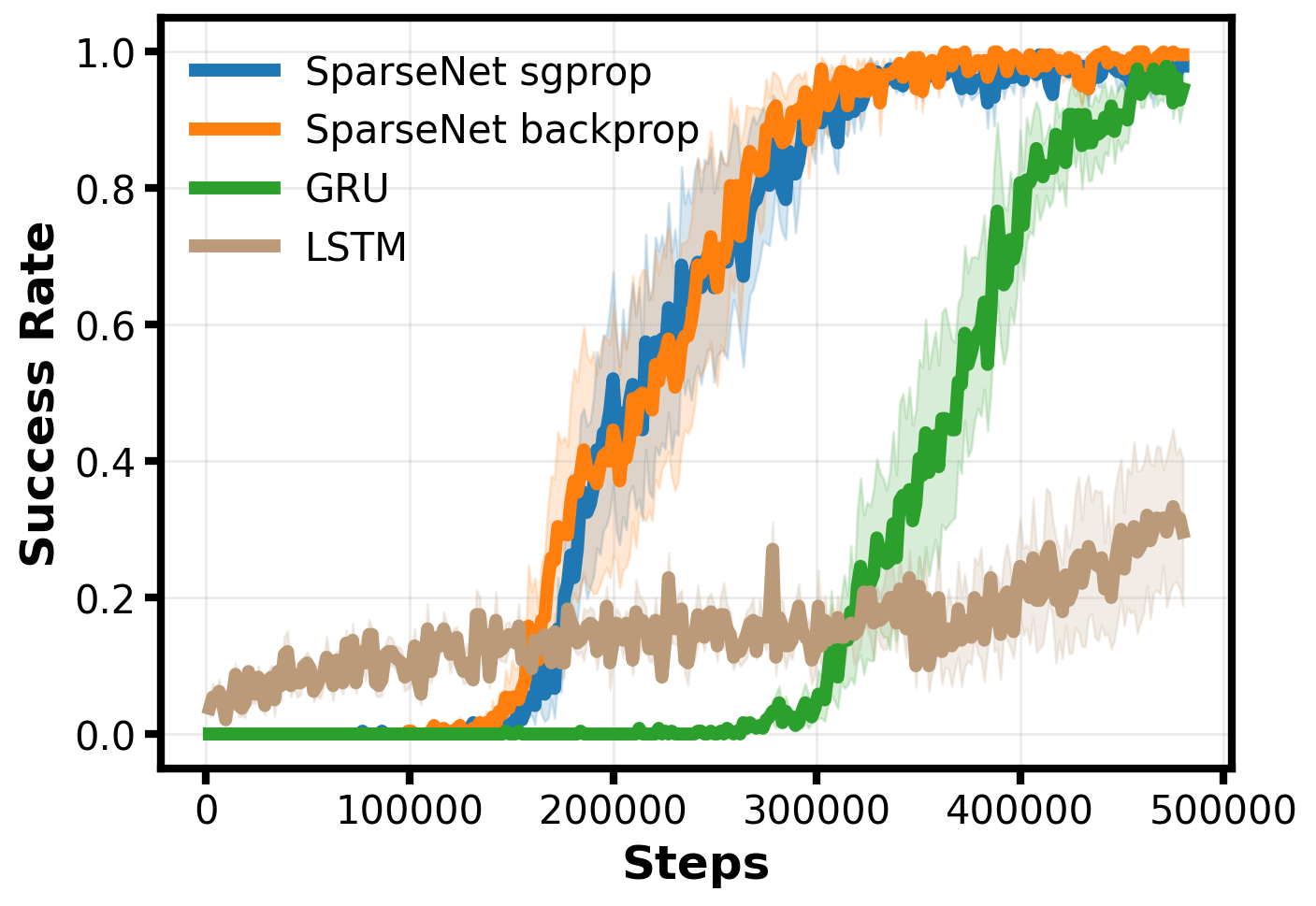}
        \caption{\texttt{explore}, 40 trials.}
        \label{fig:main-popgym-explore40}
    \end{subfigure}
    \caption{\textbf{POPGym  \texttt{labyrinth escape} and \texttt{labyrinth explore}.} Training dynamics for all methods with $20$ sampled trials per gradient update and $40$ sampled trials per gradient update.  Across these settings, synthetic gradients improve over backpropagation in few-sample regimes.}
    \label{fig:main-popgym-training}
\end{figure}

\textbf{Ablation studies.} We provide three additional ablation studies. In Appendix~\ref{subsec:ablation-mse}, by varying the training effort devoted to synthetic-gradient predictors, we show that lower gradient-estimation MSE is associated with higher evaluation reward on the MNIST bandit task. In Appendix~\ref{subsec:ablation-cond-B}, we examine how closely Condition $(B)$ holds in SparseNet by comparing locally predicted gradients with their ground truth. In Appendix~\ref{subsec:ablation-gpt}, we explore a preliminary extension to GPT-style transformers with MoE layers, observing lower validation loss than backpropagation when training with noisy next-token targets.

\section{Discussion, Limitation, and Future Work}\label{sec:discuss}

Our analysis suggests that synthetic gradients have a sample-efficiency role that is previously under-explored.
The conditions under which such benefits may arise can be interpreted as: $(A. 1)$ stochastic external feedback; $(A. 2)$ stochastic neuron computation; $(A. 3)$ neuron sparsely  connected to others; and $(B)$ neuron specialization and local sufficiency. We discuss this perspective in relation to prior work on synthetic gradients, TD versus MC value estimation, and biological learning in Appendix~\ref{sec:literature}.
Several limitations remain. First, Condition $(B)$ can be strong, and it is not comprehensively explored how well it is satisfied in real network systems. Analysis under further relaxations of this condition is an important direction for future theoretical work. Second, our theory focuses primarily on estimation MSE. A more complete sample-complexity analysis could impose explicit distributional assumptions and derive concentration bounds. Moreover, the adaptive mixing weight used in our experiment is rather rudimentary, and better algorithms for estimating these weights may have important practical benefits.  Finally, our experiments are intended as proof-of-concept demonstrations rather than a comprehensive empirical evaluation. Future work should test synthetic gradient propagation across broader architectures, tasks, and scales.

\section*{Acknowledgments}
This work is partially supported by NSF 2046795, 2205329, and 2504264, NIH, ARPA-H, the MacArthur Foundation, Good Ventures, Schmidt Sciences, the Hasso Plattner Foundation, and Stanford HAI.

\bibliography{ref}
\bibliographystyle{plain}

\appendix
\clearpage

\begin{center}
{\huge\bfseries Appendix}
\end{center}
\vspace{2em} 

\section{Related Work}\label{sec:literature}

\textbf{Synthetic Gradients.} 
Synthetic gradients were introduced as decoupled neural interfaces to remove the locking constraints imposed by backpropagation, allowing layers to update using locally predicted gradients rather than waiting for a full backward pass~\citep{jaderberg2017decoupled}. Their convergence was later analyzed to  show that synthetic-gradient training can lead to different layer-wise representations~\citep{czarnecki2017understanding}. Subsequent variants improved the design and training of synthetic-gradient models~\citep{czarnecki2017sobolev,miyato2017synthetic,belilovsky2020decoupled}, including evaluations on larger-scale architectures such as convolutional neural networks~\citep{lecun1989backpropagation,lecun2015deep} and ResNets~\citep{he2016deep}. Synthetic gradients,  defined as a local learning rule,  has also been studied in settings where globally backpropagated gradients are not fully available, including online recurrent learning~\citep{bellec2019biologically,silver2021learning,pemberton2024bpmathbflambda}, transductive meta-learning~\citep{hu2020empirical}, distributed or edge training~\citep{chen2019exploring,neven2020feed}, and  biologically learning~\citep{bellec2019biologically,kaiser2020synaptic,pemberton2021cortico}. These works suggest application-specific benefits of synthetic gradients, as discussed below.

\textbf{Benefits of Synthetic Gradients versus Backpropagation.} One of the clearest benefits of synthetic gradients discovered by prior work is a reduced horizon dependence in recurrent learning~\citep{jaderberg2017decoupled,silver2021learning,pemberton2024bpmathbflambda}. In this setting, backpropagation through time (BPTT) is often applied only over a truncated window, while synthetic gradients can estimate feedback gradients beyond the truncation boundary. Another benefit arises when training is bottlenecked by synchronization or communication: synthetic gradients allow modules to update asynchronously and can improve throughput in distributed settings~\citep{chen2019exploring,neven2020feed}. However, when exact backpropagation is available, prior work generally finds that synthetic gradients, without additional mechanisms, do not match the performance of backpropagation. In this work, we investigate synthetic gradients from a fresh perspective: sample efficiency, and we compare against unconstrained backpropagation. As discussed in Section~\ref{sec:motivation}, the sample efficiency benefit is analogous to the advantage of temporal-difference methods over Monte Carlo methods in value function estimation.

\textbf{TD versus MC.}
Temporal-difference (TD) methods~\citep{sutton1988learning} have long been studied as alternatives to Monte Carlo (MC) value function estimation, replacing unbiased full-return targets with bootstrapped recursive targets. Despite empirical success, understanding when this bootstrapping improves sample efficiency remains a longstanding question and an active research area. Early work on phased TD established an explicit bias-variance tradeoff~\citep{kearns2000bias}. Related oracle-style analyses further showed that TD methods can exploit Markov structure more efficiently than MC in Markov reward processes~\citep{grunewalder2007optimality}. Subsequent work developed nonasymptotic guarantees for TD and its variants in tabular and linear function approximation settings~\citep{lazaric2010finite,bhandari2018finite,cheikhi2023statistical,li2024high,cheng2024surprising}, while lower bounds showed that plain TD is not always statistically optimal~\citep{khamaru2021temporal,li2023accelerated}. More recent analyses extend finite-sample TD theory to nonlinear function approximation~\citep{cai2019neural,sun2022finite,berthier2022non} and partially observable environments~\citep{cayci2024finite}. Overall, the literature suggests that TD is not uniformly superior to MC, but can be substantially more sample-efficient when it successfully exploits the Bellman structure.

\textbf{Discussion of Biological Learning.} As detailed in Section~\ref{sec:condition} and illustrated in Section~\ref{sec:example}, we identify four conditions under which synthetic gradients may offer sample-efficiency benefits: $(A. 1)$ stochastic external feedback, $(A. 2)$ stochastic neuron computation, $(A. 3)$ sparse connectivity, and $(B)$ neuron specialization and local sufficiency. Some of these conditions have analogies in biological learning, which we briefly discuss here as broader context. Biological learning in real world is intrinsically stochastic, which may be modeled by noisy and delayed external reward signals~\citep{schultz1997neural, angela2005uncertainty, ohmae2015climbing}. Internally, stochasticity is also a well-established feature of neural computation: cortical neurons exhibit trial-to-trial spiking and synaptic transmission is probabilistic~\citep{larkman1992presynaptic, softky1993highly, dayan2005theoretical, orban2016neural, wang2025theoretical}. Moreover, biological networks are sparse and structured rather than densely connected~\citep{song2005highly, oh2014mesoscale, wang2025theoretical}. Finally, evidence suggests that cortical areas or  neuronal populations can exhibit a certain level of specialization and task decomposition~\citep{kanwisher1997fusiform, tasic2018shared, mante2013context,wang2025theoretical}.
These observations offer a tentative biological perspective on our assumptions; whether they support a more concrete relationship with synthetic-gradient propagation remains an open question for future investigation.

\section{Proof of Lemma~\ref{lemma:decomposition}}

\textbf{Lemma~\ref{lemma:decomposition}.}
Let $\hsparagrad_v^\para = \frac{1}{n}\sum_{i\in[n]}\sparagrad_{v,i}^\para$
be the empirical mean of $n$ i.i.d.\ samples. Then
\begin{align}
    \smse_v^2
    =\frac{1}{n}\left( \scmse_v^2+\serror_v^2\right)
    +
    \sbias_v^2.
\end{align}
In particular, for an conditionally unbiased estimator, i.e., $\E\,[\sparagrad_v^\para\mid \invar_v]=\E\,[\paragrad_v^\para\mid \invar_v]$, we have
\begin{align}
    \smse_v^2
    = \frac{1}{n}\left( \scmse_v^2+\error_v^2\right).
\end{align}

\begin{proof}
We start from the standard bias-variance decomposition:
\begin{align}
    \smse_v^2
    &=
    \E\,\left\|
        \hsparagrad_v^\para - \E\,\paragrad_v^\para
    \right\|_2^2\\
    &=
    \E\,\left\|
        \hsparagrad_v^\para - \E\,\sparagrad_v^\para
    \right\|_2^2
    +
    \left\|
        \E\,\sparagrad_v^\para - \E\,\paragrad_v^\para
    \right\|_2^2\\
    &=
    \E\,\left\|
        \hsparagrad_v^\para - \E\,\sparagrad_v^\para
    \right\|_2^2
    +
    \sbias_v^2.
\end{align}
Since $\hsparagrad_v^\para$ is the empirical mean of $n$ i.i.d.\ samples,
\begin{align}
    \E\,\left\|
        \hsparagrad_v^\para - \E\,\sparagrad_v^\para
    \right\|_2^2
    &=
    \frac{1}{n}
    \E\,\left\|
        \sparagrad_v^\para - \E\,\sparagrad_v^\para
    \right\|_2^2.
\end{align}
Now add and subtract $\E[\sparagrad_v^\para\mid \invar_v]$, and apply the law of total variance:
\begin{align}
    \E\,\left\|
        \sparagrad_v^\para - \E\,\sparagrad_v^\para
    \right\|_2^2
    &=
    \E\,\left\|
        \sparagrad_v^\para - \E[\sparagrad_v^\para\mid \invar_v]
    \right\|_2^2
    +
    \E\,\left\|
        \E[\sparagrad_v^\para\mid \invar_v] - \E\,\sparagrad_v^\para
    \right\|_2^2\\
    &=
    \scmse_v^2+\serror_v^2.
\end{align}
Combining the two results shows
\begin{align}
    \smse_v^2
    =\frac{1}{n}\left( \scmse_v^2+\serror_v^2\right)
    +
    \sbias_v^2.
\end{align}
For an unbiased estimator, i.e., $\E\,[\sparagrad_v^\para\mid \invar_v]=\E\,[\paragrad_v^\para\mid \invar_v]$, we have $\sbias^2_v=0$ and $\serror^2_v=\error^2_v$. Thus, 
\begin{align}
    \smse_v^2
    = \frac{1}{n}\left( \scmse_v^2+\error_v^2\right).
\end{align}
\end{proof}

\section{Proof of Theorem~\ref{thm:bp-optimal}}

\textbf{Theorem \ref{thm:bp-optimal}.}
Under $\neg (A) \wedge (B)$, every gradient estimator $\hat{\sgrad}^\para$ formed by empirical means of i.i.d. conditionally unbiased gradient samples (i.e., $\E\,[\sparagrad_v^\para\mid \invar_v]=\E\,[\paragrad_v^\para\mid \invar_v]$) has its estimation MSE greater or equal to that of the empirical-mean backpropagation estimator: $\smse_v^2\geq \mse_v^2$ for any non-input nodes $v\in V$..

\begin{proof}
First, we may examine the $\neg (A)$ in more detail. $\neg (A)$ needs \textbf{all} the following conditions to be true.
\begin{itemize}
    \item[$\neg $(A.1)] \textbf{Deterministic terminal gradient feedback.} 
    The gradient feedback $\grad_{(L,i)}$ to any terminal nodes $(L, i)\in V_L$ is almost surely a deterministic function of the terminal output $\node_{(L, i)}$. 
    \item[$\neg $(A.2)] \textbf{Deterministic network computation.}  
    All nodes $v\in V$ perform deterministic computations. 

    \item[$\neg $(A.3)] \textbf{Fully connected network.}  
    All non-input nodes are fully connected to their previous layer.
\end{itemize}

Next, we prove the following statement for all non-input nodes $v$.
\begin{align}
    \neg (A.1) \wedge \neg (A.2) \wedge \neg (A.3) \wedge (B) \ \implies\  \cmse^2_v =0.
\end{align}
 We prove this theorem by induction.

Define $S(l)$ to be the following statement: \textit{the gradient feedback $\grad_{(l,i)}$ to any node $(l, i)\in V_l$ is almost surely a deterministic function of the node's local observation $\invar_{(l, i)}$.} 

Note that $S(L)$ automatically holds under $\neg (A.1)$. Now suppose $S(l)$ holds for a layer $l>1$, i.e., any node $v\in V_l$ satisfies
\begin{align}
    \grad_v = \E[\grad_v\mid \invar_v],\qquad a.s.. \label{eq:thm-bp-1}
\end{align}

Since the jacobian $\jacob^\para_v$ is a function of $\invar_v$, by \eqref{eq:thm-bp-1} we have 
\begin{align}
    \grad^\para_v=\jacob^\para_v \grad_v =\jacob^\para_v \E[\grad_v\mid \invar_v]= \E[\jacob^\para_v \grad_v\mid \invar_v]=\E[ \grad_v^\para\mid \invar_v],\qquad a.s..
\end{align}
This implies 
 \begin{align}
        \cmse^2_v=\E\!\left[\left\| \paragrad_{v}^\para-\E[\grad^\para_{v}\mid \invar_v ]\right\|_2^2\right]=0.
    \end{align}

    This proves the zero conditional variance for nodes $v\in V_l$ in this layer. Next, we prove that $S(l-1)$ holds. 

    For any node $v\in V_l$ in the layer $l$, by $\neg (A. 3)$ and $(B)$ we have 
\begin{align}
    \forall v^-\in V_{l-1}:\ \  \E[\grad_v\mid \invar_v]=\E[\grad_v\mid \invar_v, \invar_{v^-}]. \label{eq:thm-bp-0}
\end{align}

For $\forall v^-\in V_{l-1}$, we start by definition and then apply those conditions:
\begin{align}
    \grad_{v^-}&=\sum_{v\in \child(v^-)} \jacob_{v^-}^{\, v}\grad_v\\
    &=\sum_{v\in \child(v^-)} \jacob_{v^-}^{\, v}\E[\grad_v \mid \invar_v, \invar_{v^-} ]\qquad a.s. \tag{by (\ref{eq:thm-bp-1}) and (\ref{eq:thm-bp-0}) }\\
    &=\sum_{v\in \child(v^-)} \E[\jacob_{v^-}^{\, v}\grad_v \mid \invar_v , \invar_{v^-} ]\qquad a.s. \tag{$\jacob_{v^-}^{\, v}$ is a function of $\invar_v$}\\
    &=\sum_{v\in \child(v^-)} \E[\jacob_{v^-}^{\, v}\grad_v \mid \invar_{v^-} ]\qquad a.s. \tag{by $\neg (A.2)$ and $\neg (A.3)$, $\invar_v$ is a function of $\invar_{v^-}$}\\
    &=\E\left[ \sum_{v\in \child(v^-)} \jacob_{v^-}^{\, v}\grad_v \ \Big|\  \invar_{v^-}\right]\qquad a.s. \\
    &=\E\left[ \grad_{v^-} \mid  \invar_{v^-}\right]\qquad a.s. 
\end{align}
This proves $S(l-1)$. Therefore, by induction, backpropagation achieves $\cmse^2_v=0$ for all non-input nodes $v$. By Lemma~\ref{lemma:decomposition}, the estimator MSE from backpropagation is.
\begin{align}
    \mse_v^2= \frac{1}{n}\left( \cmse_v^2+\error_v^2\right)=\frac{1}{n}\error_v^2.
\end{align}

Therefore, by Lemma~\ref{lemma:decomposition}, any unbiased estimator cannot have smaller estimation MSE:
\begin{align}
    \smse_v^2= \frac{1}{n}\left( \scmse_v^2+\error_v^2\right)\geq \frac{1}{n}\error_v^2=\mse_v^2.
\end{align}

\end{proof}

\section{Proof of Theorem~\ref{thm:synth-better}}

\textbf{Theorem \ref{thm:synth-better}}
    Under $(A)\wedge (B)$ and Assumption~\ref{assume:1}, consider synthetic gradient defined in \eqref{eq:synth-def} and all mixing weights $\lambda=0$. Then, for all non-input  $v\in V$: $\smse^2_v\leq \mse^2_v$, and there exists some $v$: $\smse^2_v< \mse^2_v$.

\begin{proof}
    We begin by proving the following lemma. 

    \begin{lemma}\label{lemma-unbiased}
        Under (B), synthetic gradient propagation with \eqref{eq:synth-def} and all $\lambda=0$ leads to unbiased $\sgrad_v$. Formally, for all non-input nodes $v$ we have 
        \begin{align}
            \E[\sgrad_v\mid \invar_v]=\E[\grad_v\mid \invar_v].
        \end{align}
    \end{lemma}
    \begin{proof}
        We prove by induction. First, for terminal nodes $\forall v\in V_L$, $\sgrad_v=\grad_v$ by definition, and thus  $\E[\sgrad_v\mid \invar_v]=\E[\grad_v\mid \invar_v]$. 
        
        Now, suppose for a layer $l>1$ we have $\forall v\in V_l: \E[\sgrad_v\mid \invar_v]=\E[\grad_v\mid \invar_v]$. Then, for any $v^- \in V_{l-1}$:
        \begin{align}
            \E[\sgrad_{v^-}\mid \invar_{v^-}]&=\E\left[\sum_{v\in \child(v^-)}\synth_{v^-}^{\, v} \ \Big| \ \invar_{v^-}\right]\\
            &=\E\left[\sum_{v\in \child(v^-)}\jacob_{v^-}^{\, v} \E[\sgrad_v\mid \invar_{v}]  \ \Big| \ \invar_{v^-}\right]\tag{by (\ref{eq:synth-def})}\\
            &=\E\left[\sum_{v\in \child(v^-)}\jacob_{v^-}^{\, v}\E[\grad_v\mid \invar_v] \ \Big| \ \invar_{v^-}\right]\tag{by induction}\\
            &=\E\left[\sum_{v\in \child(v^-)}\jacob_{v^-}^{\, v}\E[\grad_v \mid \invar_v , \invar_{v^-} ] \ \Big| \ \invar_{v^-}\right]\tag{by (B)}\\
            &=\E\left[\sum_{v\in \child(v^-)}\E[\jacob_{v^-}^{\, v}\grad_v \mid \invar_v , \invar_{v^-} ] \ \Big| \ \invar_{v^-}\right]\tag{$\jacob_{v^-}^{\, v}$ is a function of $\invar_v$}\\
            &=\E\left[\sum_{v\in \child(v^-)}\jacob_{v^-}^{\, v}\grad_v \ \Big| \ \invar_{v^-}\right] \tag{double expectation}\\
            &=\E\left[\grad_{v^-} \mid \invar_{v^-}\right] .
        \end{align}
        Therefore, by induction, $\E[\sgrad_v\mid \invar_v]=\E[\grad_v\mid \invar_v]$ holds for all non-input nodes $v$.
    \end{proof}

    To proceed, we first prove the following inequality that explicitly shows how uncertainty appears  under Condition $(A)$ and Assumption~\ref{assume:1}:
    \begin{align}
        &\E\|\grad^\para_v - \E[\grad^\para_v \mid \invar_v, \bm{Z}_{\child(v)} ]\|_2^2 \\
        &= \E\|\grad^\para_v - \E[\grad^\para_v \mid \invar_v, \bm{Z}_{\child(v)},\bm{\chi} ]+\E[\grad^\para_v \mid \invar_v, \bm{Z}_{\child(v)},\bm{\chi} ]-\E[\grad^\para_v \mid \invar_v, \bm{Z}_{\child(v)} ]\|_2^2\\ 
        &= \E\|\grad^\para_v - \E[\grad^\para_v \mid \invar_v, \bm{Z}_{\child(v)},\bm{\chi} ]\|^2_2+ \E\|\E[\grad^\para_v \mid \invar_v, \bm{Z}_{\child(v)},\bm{\chi} ]-\E[\grad^\para_v \mid \invar_v, \bm{Z}_{\child(v)} ]\|_2^2\\ 
        &\geq  \E\|\E[\grad^\para_v \mid \invar_v, \bm{Z}_{\child(v)},\bm{\chi} ]-\E[\grad^\para_v \mid \invar_v, \bm{Z}_{\child(v)} ]\|_2^2\\
        &>0 \tag{by Assumption~\ref{assume:1}}.
    \end{align}

    Therefore, under Condition $(A)$ and Assumption~\ref{assume:1}, there is some node $v\in V_l$ having 
    \begin{align}
        \E\|\grad^\para_v - \E[\grad^\para_v \mid \invar_v, \bm{Z}_{\child(v)} ]\|^2>0, \label{eq:lemma-var-1}
    \end{align}
    where 
    \begin{align}
        \bm{Z}_{\child(v)}:= \{\invar_{v^+}:\ v^+\in \child (v)\}.
    \end{align}

    With Lemma~\ref{lemma-unbiased} and \eqref{eq:lemma-var-1}, we are ready to prove the theorem. Note that Lemma~\ref{lemma-unbiased} implies
    \begin{align}
        \E[\sparagrad^\para_{v}\mid \invar_{v}]&=\E[\jacob_v^{\, \para} \sgrad_{v}\mid \invar_{v}]=\jacob_v^{\, \para} \E[\sgrad_{v}\mid \invar_{v}]\tag{$\jacob_v^{\, \para} $ is a function of $\invar_v$}\\
        &=\jacob_v^{\, \para} \E[\grad_{v}\mid \invar_{v}]\tag{by Lemma \ref{lemma-unbiased}}\\
        &=\E[\jacob_v^{\, \para} \grad_{v}\mid \invar_{v}]=\E[\paragrad^\para_{v}\mid \invar_{v}]. \label{eq:thm-synth-1-1}
    \end{align}

    We start from the conditional MSE. For any non-input node $v$:
    \begin{align}
        \scmse^2_v&=\E\!\left[\left\|  \sparagrad_{v}^\para-\E[\sgrad^\para_v\mid \invar_v]\right\|_2^2 \right]= \E\!\left[\left\|  \sparagrad_{v}^\para-\E[\grad^\para_v\mid \invar_v]\right\|_2^2 \right] \tag{by (\ref{eq:thm-synth-1-1})}.
    \end{align}
    
    Recall synthetic gradient propagation with \eqref{eq:synth-def} and that all $\lambda_v=0$, we have 
    \begin{align}
        \sparagrad_{v}^\para&=\jacob^\para_v \sgrad_v=\jacob^\para_v \sum_{v^+\in \child(v)}\synth_v^{\, v^+}=\jacob^\para_v \sum_{v^+\in \child(v)}\jacob_v^{\, v^+}\E[\sgrad_{v^+}\mid \invar_{v^+}]\\
        &= \jacob^\para_v\ \sum_{v^+\in \child(v)}\jacob_v^{\, v^+}\E\left[\grad_{v^+} \mid \invar_{v^+}\right]\tag{Lemma~\ref{lemma-unbiased}}\\
        &= \jacob^\para_v\ \sum_{v^+\in \child(v)}\jacob_v^{\, v^+}\E\left[\grad_{v^+} \mid \bm{Z}_{\child(v)},  \invar_{v} \right]\tag{by (B)}\\
        &= \E\left[\sum_{v^+\in \child(v)}\jacob^\para_v\ \jacob_v^{\, v^+}\grad_{v^+} \mid \bm{Z}_{\child(v)},  \invar_{v} \right]\tag{$\jacob^\para_v, \jacob_v^{\, v^+}$ are functions of $\bm{Z}_{\child(v)}, \invar_v$}\\
        &=  \E\left[\grad^\para_{v} \mid \bm{Z}_{\child(v)},\invar_v\right]\label{eq:thm-synth-1-2}
    \end{align}
    Because conditioning does not increase variance, we conclude that 
    \begin{align}
        \scmse_v^2 \leq \cmse^2_v.
    \end{align}
    By Lemma~\ref{lemma-unbiased}, $\sgrad^\para_v$ is unbiased, and then by Lemma~\ref{lemma:decomposition}:
    \begin{align}
        \smse_v^2=\frac{1}{n}(\scmse_v^2+\error^2_v) \leq \frac{1}{n}(\cmse_v^2+\error^2_v)=\mse^2_v. \label{eq:thm-synth-1-3}
    \end{align}

    It remains to show that some node $v$ achieves strict inequality.  It follows from \eqref{eq:thm-synth-1-3} that 
    \begin{align}
        \mse_v^2-\smse_v^2 = \frac{1}{n}\left( \cmse_v^2-\scmse_v^2 \right).
    \end{align}
    By \eqref{eq:lemma-var-1}, there is some non-input node $v$ such that 
    \begin{align}
         \E\|\grad^\para_v - \E[\grad^\para_v \mid \invar_v, \bm{Z}_{\child(v)} ]\|_2^2>0,
    \end{align}
    and thus 
    \begin{align}
        \cmse_v^2&=\E\|\grad^\para_v - \E[\grad^\para_v \mid \invar_v ]\|^2=\E\|\grad^\para_v - \E[\grad^\para_v \mid \invar_v ,\bm{Z}_{\child(v)}]+\E[\grad^\para_v \mid \invar_v,\bm{Z}_{\child(v)} ]-\E[\grad^\para_v \mid \invar_v ]\|_2^2\\
        &=\E\|\grad^\para_v - \E[\grad^\para_v \mid \invar_v ,\bm{Z}_{\child(v)}]\|^2_2 + \E \| \E[\grad^\para_v \mid \invar_v,\bm{Z}_{\child(v)} ]-\E[\grad^\para_v \mid \invar_v ]\|_2^2\\
        &> \E \| \E[\grad^\para_v \mid \invar_v,\bm{Z}_{\child(v)} ]-\E[\grad^\para_v \mid \invar_v ]\|_2^2 \tag{by (\ref{eq:lemma-var-1})}\\
        &= \E \| \sgrad_v^\para-\E[\sgrad^\para_v \mid \invar_v ]\|_2^2 \tag{by (\ref{eq:thm-synth-1-2}) and Lemma~\ref{lemma-unbiased}}\\
        &= \scmse_v^2.
    \end{align}
    Therefore, 
    \begin{align}
        \mse_v^2-\smse_v^2 = \frac{1}{n}\left( \cmse_v^2-\scmse_v^2 \right)>0.
    \end{align}

\end{proof}

\section{Proof of Lemma~\ref{lemma-biased}}

We first quantify how the two approximation errors induce bias in synthetic gradient propagation.

\begin{lemma}\label{lemma-biased}
    Under $(B')$, synthetic gradient propagation with \eqref{eq:approx-synth-def}.  Then, for any non-input nodes $v\in V_l$ and collection $\bm Z_{U_{l}}\subseteq \ean(v)$ of extended ancestors, assuming bounded jacobians $\|\jacob_v^{\, v^+}\|_2\leq J$, we have 
    \begin{align}
        \left\|\E[\sgrad_v\mid \invar_v, \bm{Z}_{U_{l}}]-\E[\grad_v\mid \invar_v, \bm{Z}_{U_{l}}]\right\|_2\leq c_v( \Berror+\epsilon'),\qquad a.s. \label{eq:lemma-biased-1}
    \end{align}
    where $\forall v \in V_L: c_v=0$ for terminal nodes, and a node $v\in V_l$ in layer $l<L$:
    \begin{align}
        c_{v}=J\sum_{v^+\in \child(v)}(1-\weight_{v}+c_{v^+}) \label{eq:cv-recursion}
    \end{align}
\end{lemma}

\begin{proof}

    We begin by introducing a convenient notation. 
    \begin{definition}\label{def:O-eps-notation}
        We use $\gO(\epsilon)$ to denote a vector that satisfies $\|\gO(\epsilon)\|_2\leq \epsilon$ almost surely.  
    \end{definition}
    With this notation, we can easily keep track of the scale of errors, e.g., $\gO(\epsilon)+\gO(\epsilon)=\gO(2\epsilon)$; $\jacob \gO(\epsilon)=\gO(\|\jacob\|_2 \epsilon)$.

    Therefore, condition $(B')$ can be expressed as 
    \begin{align}
        \E[\grad_v\mid \invar_v] = \E[\grad_v\mid \invar_v, \bm{Z}_{U_{l}}] +\gO(\epsilon).
    \end{align}
    Similarly, \eqref{eq:approx-synth-def} can be expressed as
    \begin{align}
         \synth_{v^+} = \E[\sgrad_{v^+}\mid \invar_{v^+}] + \gO(\epsilon').
    \end{align}
    
    Next, we prove the lemma by induction. First, for terminal nodes $\forall v\in V_L$, $\sgrad_v=\grad_v$ by definition, and thus  $\E[\sgrad_v\mid \bm{Z}_{U_L}]=\E[\grad_v\mid \bm{Z}_{U_L}]$ for any subset $U_L\subseteq \ean(v)$.
    
    Now, suppose \eqref{eq:lemma-biased-1} holds for a layer $l>1$.  Then, consider any $v^- \in V_{l-1}$ and a subset $U_{l-1}\subseteq \ean(v^-)$. 
    Let $U_l:=\{v^-\}\cup U_{l-1}$. Note that for any $v\in \child(v^-)$, we have $U_l\subseteq \ean(v)$ by Definition~\ref{def:ances}. Therefore,
    \begin{align}
        \E[\sgrad_{v^-}\mid \invar_{v^-}, \bm{Z}_{U_{l-1}}]&= \E[\sgrad_{v^-}\mid \bm{Z}_{U_{l}}]\\
        &=\E\left[\sum_{v\in \child(v^-)} \weight_{v^-} \jacob_{v^-}^{\, v} \sgrad_v + (1-\weight_{v^-}) (\jacob_{v^-}^{\, v}(\E[\sgrad_v \mid \invar_v  ] + \gO(\epsilon'))) \ \Big| \ \bm{Z}_{U_{l}}\right]\\
        &=\weight_{v^-} \underbrace{\E\left[\sum_{v\in \child(v^-)} \jacob_{v^-}^{\, v} \sgrad_v \ \Big| \ \bm{Z}_{U_{l}} \right]}_{A} + (1-\weight_{v^-}) \underbrace{\E\left[\sum_{v\in \child(v^-)}\jacob_{v^-}^{\, v}\E[\sgrad_v \mid \invar_v  ] \ \Big| \ \bm{Z}_{U_{l}} \right] }_{B}\\
        &\qquad + (1-\weight_{v^-})\cdot \gO\left(|\child(v^-)|\cdot J\epsilon'\right)
    \end{align}
    Applying double expectation, let $\epsilon'':=\epsilon+\epsilon'$, we have 
    \begin{align}
        A&=\E\left[\sum_{v\in \child(v^-)} \E\left[  \jacob_{v^-}^{\, v} \sgrad_v \mid \invar_v, \bm{Z}_{U_{l}} \right] \ \Big| \ \bm{Z}_{U_{l}} \right]\\
        &=\E\left[\sum_{v\in \child(v^-)} \jacob_{v^-}^{\, v} \E\left[  \sgrad_v \mid \invar_v, \bm{Z}_{U_{l}} \right] \ \Big| \ \bm{Z}_{U_{l}} \right]\\
        &=\E\left[\sum_{v\in \child(v^-)} \jacob_{v^-}^{\, v} ( \E\left[  \grad_v \mid \invar_v, \bm{Z}_{U_{l}} \right] + \gO(c_v\epsilon'')) \ \Big| \ \bm{Z}_{U_{l}} \right]\tag{by induction}\\
        &=\E\left[\sum_{v\in \child(v^-)} \jacob_{v^-}^{\, v} \E\left[  \grad_v \mid \invar_v, \bm{Z}_{U_{l}} \right] \ \Big| \ \bm{Z}_{U_{l}} \right]+\gO\left(J\epsilon''\sum_{v\in \child(v^-)}c_v \right)\\
         &=\E\left[\sum_{v\in \child(v^-)} \E\left[  \jacob_{v^-}^{\, v} \grad_v \mid \invar_v, \bm{Z}_{U_{l}} \right] \ \Big| \ \bm{Z}_{U_{l}} \right]+\gO\left(J\epsilon''\sum_{v\in \child(v^-)}c_v \right)\\
        &=\E\left[\grad_{v^-} \mid \bm{Z}_{U_{l}} \right]+\gO\left(J\epsilon''\sum_{v\in \child(v^-)}c_v \right)\tag{double expectation}.
    \end{align}
    Similarly, for the other term:
    \begin{align}
        B&=\E\left[\sum_{v\in \child(v^-)}\jacob_{v^-}^{\, v}\E[\sgrad_v \mid \invar_v  ] \ \Big| \ \bm{Z}_{U_{l}} \right] \\
        &=\E\left[\sum_{v\in \child(v^-)}\jacob_{v^-}^{\, v}(\E[\grad_v \mid \invar_v  ]+ \gO(c_v\Berror'')) \ \Big| \ \bm{Z}_{U_{l}}  \right]\tag{by induction} \\
        &=\E\left[\sum_{v\in \child(v^-)}\jacob_{v^-}^{\, v}\E[\grad_v \mid \invar_v  ]\ \Big| \ \bm{Z}_{U_{l}}  \right] +\gO\left(J\epsilon''\sum_{v\in \child(v^-)}c_v \right)\\
        &=\E\left[\sum_{v\in \child(v^-)}\jacob_{v^-}^{\, v}(\E[\grad_v \mid \invar_v, \bm{Z}_{U_{l}}  ] + \gO(\epsilon))\ \Big| \ \bm{Z}_{U_{l}}  \right] +\gO\left(J\epsilon''\sum_{v\in \child(v^-)}c_v \right)\tag{by (B')}\\
        &=\E\left[\sum_{v\in \child(v^-)}\jacob_{v^-}^{\, v}\E[\grad_v \mid \invar_v, \bm{Z}_{U_{l}}  ] \ \Big| \ \bm{Z}_{U_{l}}  \right] +\gO\left(J\epsilon\cdot |\child(v^-)|+J\epsilon''\sum_{v\in \child(v^-)}c_v \right)\\
        &=\E\left[\sum_{v\in \child(v^-)}\E[\jacob_{v^-}^{\, v}\grad_v \mid\invar_v, \bm{Z}_{U_{l}}  ] \ \Big| \ \bm{Z}_{U_{l}}  \right] +\gO\left(J\epsilon\cdot |\child(v^-)|+J\epsilon''\sum_{v\in \child(v^-)}c_v \right)\\
        &=\sum_{v\in \child(v^-)}\E[\jacob_{v^-}^{\, v}\grad_v \mid\bm{Z}_{U_{l}}  ]  +\gO\left(J\epsilon\cdot |\child(v^-)|+J\epsilon''\sum_{v\in \child(v^-)}c_v \right)\tag{double expectation}\\
        &=\E\left[\grad_{v^-} \ \Big| \ \bm{Z}_{U_{l}}  \right] +\gO\left(J\epsilon\cdot |\child(v^-)|+J\epsilon''\sum_{v\in \child(v^-)}c_v \right).
    \end{align}
    Therefore, combining the above, we have 
    \begin{align}
        \E[\sgrad_{v^-}\mid \invar_{v^-}, \bm{Z}_{U_{l-1}}]&=\weight_{v^-}A+(1-\weight_{v^-})B+(1-\weight_{v^-})\cdot \gO\left(|\child(v^-)|\cdot J\epsilon'\right)\\
        &=\E[\grad_{v^-}\mid \invar_{v^-}, \bm{Z}_{U_{l-1}}]+\gO\left((1-\weight_{v^-})J(\Berror+\epsilon')\cdot |\child(v^-)| + J\Berror''\sum_{v\in \child(v^-)}c_v\right)\\
        &=\E[\grad_{v^-}\mid \invar_{v^-}, \bm{Z}_{U_{l-1}}]+\gO\left(c_{v^-}(\Berror+\epsilon')\right),
    \end{align}
    where 
    \begin{align}
        c_{v^-}=J\sum_{v\in \child(v^-)}(1-\weight_{v^-}+c_v).
    \end{align}
    By induction, \eqref{eq:lemma-biased-1} holds for every node.
\end{proof}

\section{Proof of Theorem~\ref{thm:general-benefits}}

\textbf{Theorem~\ref{thm:general-benefits}}
   Under condition (B') and \eqref{eq:approx-synth-def}, assume bounded jacobians $\forall v, v^+:\ \max\{\|\jacob_{v}^{\, v^+}\|_2, \|\jacob_{v}^{\para}\|_2\}\leq J$. 
    A sufficient condition for $\weight^\star_v<1$ is:
   \begin{align}
        \frac{\tr\Cov(\sss_v)-\tr\Cov(\qqq_v)}{n}>\left(J^2(\Berror+\epsilon')\sum_{v^+\in \child(v)} (1+c_{v^+})\right)^2
    \end{align}
    where $\forall u\in V: c_u$ is defined recursively via 
    \begin{align}
        c_{u}=J\sum_{u^+\in \child(u)}(1-\weight_{u}+c_{u^+}),\quad \text{and for terminal nodes } \ \forall u\in V_L: c_u=0. 
    \end{align}

\begin{proof}
    Let $\epsilon'':=\epsilon+\epsilon'$ and we use Definition~\ref{def:O-eps-notation} for a bounded vector such as $\gO(\epsilon'')$. We start by applying Lemma~\ref{lemma-biased} to the following quantity.
    \begin{align}
        \E[\qqq_v]&=\E\left[\jacob^\para_v\sum_{v^+\in \child(v)}\jacob_{v}^{\, v^+}(\E[\sgrad_{v^+} \mid \invar_{v^+} ]+\gO(\epsilon') ) \right]\\
        &=\E\left[\jacob^\para_v\sum_{v^+\in \child(v)}\jacob_{v}^{\, v^+}(\E[\grad_{v^+} \mid \invar_{v^+} ] +\gO(\epsilon'+c_{v^+}\Berror'') )\right]\tag{by Lemma~\ref{lemma-biased}}\\
        &=\E\left[\jacob^\para_v\sum_{v^+\in \child(v)}\jacob_{v}^{\, v^+}(\E[\grad_{v^+} \mid \invar_{v^+},\invar_v ] +\gO(\epsilon'+\Berror+c_{v^+}\Berror'') )\right]\tag{by (B')}\\
        &=\E\left[\sum_{v^+\in \child(v)}\E[\jacob^\para_v\jacob_{v}^{\, v^+}\grad_{v^+} \mid \invar_{v^+},\invar_v ]\right]+\gO\left(J^2\sum_{v^+\in \child(v)}\left( (c_{v^+}+1)\epsilon''\right) \right)\\
        &=\E\left[\grad^\para_v\right]+\gO\left(J^2\epsilon''\sum_{v^+\in \child(v)} (1+c_{v^+})\right) .
    \end{align}
    It is standard to show from \eqref{eq:lambda-star} that $\weight^\star < 1$ if
    \begin{align}
        \E\!\left[\left\| \hat\qqq_v-\E[\grad_{v}^\para]\right\|_2^2\right]<\E\!\left[\left\| \hat\sss_v-\E[\grad_{v}^\para]\right\|_2^2\right]. \label{eq:thm-general-1}
    \end{align}
    The LHS can be decomposed as follows. 
    \begin{align}
        \E\!\left[\left\| \hat\qqq_v-\E[\grad_{v}^\para]\right\|_2^2\right]&=\E\!\left[\left\| \hat\qqq_v-\E[\qqq_v]+\E[\qqq_v]-\E[\grad_{v}^\para]\right\|_2^2\right]\\
        &=\frac{1}{n}\E\!\left[\left\| \qqq_v-\E[\qqq_v]\right\|_2^2\right]+\left\|\E[\qqq_v]-\E[\grad_{v}^\para]\right\|_2^2\\
        &\leq \frac{1}{n}\E\!\left[\left\| \qqq_v-\E[\qqq_v]\right\|_2^2\right]+\left(J^2\epsilon''\sum_{v^+\in \child(v)} (1+c_{v^+})\right)^2 \\ 
        &=\frac{1}{n}\tr\Cov(\qqq_v)+\left(J^2\epsilon''\sum_{v^+\in \child(v)} (1+c_{v^+})\right)^2.
    \end{align}
    Similarly, the RHS shows 
    \begin{align}
        \E\!\left[\left\| \hat\sss_v-\E[\grad_{v}^\para]\right\|_2^2\right]&=\E\!\left[\left\| \hat\sss_v-\E[\sss_v]+\E[\sss_v]-\E[\grad_{v}^\para]\right\|_2^2\right]\\
        &=\frac{1}{n}\E\!\left[\left\| \sss_v-\E[\sss_v]\right\|_2^2\right]+\left\|\E[\sss_v]-\E[\grad_{v}^\para]\right\|_2^2\\
        &\geq\frac{1}{n}\E\!\left[\left\| \sss_v-\E[\sss_v]\right\|_2^2\right]\\
        &=\frac{1}{n}\tr\Cov(\sss_v).
    \end{align}
    Therefore, a sufficient condition for \eqref{eq:thm-general-1} is 
    \begin{align}
        \frac{1}{n}\tr\Cov(\qqq_v)+\left(J^2\epsilon''\sum_{v^+\in \child(v)} (1+c_{v^+})\right)^2<\frac{1}{n}\tr\Cov(\sss_v). 
    \end{align}
\end{proof}

\section{Detailed Setup of the Example of Finite-Partition Expert Network}\label{sec:detail-example}

We consider a two-hidden-layer architecture in which the first hidden layer has a single node, and the second hidden layer has nodes $V_2=\{(2, j)\}_{j\in [m_2]}$. In the rest of the section, we denote $j:=(2, j)$. We focus on gradient estimation on the node $(1,1)\in V_1$ in the first layer, and thus in this section \emph{we omit the subscript when referring to this target node}. For example, $\grad^\para$ refers to the backpropagated gradient on this target node $(1,1)\in V_1$.

The first layer node computes a $d$-dim feature representation $\node=f(\invar; \para)\in \sR^d$, where $\invar$ is the input sampled from an external source  (i.e., from layer $l=0$).  The second layer contains $m_2$ expert nodes, where each expert $j\in[m_2]$ specializes in a $k$-dimensional subspace of the feature vector $\node$. Formally, each expert $j$ effectively acts on $\bm u_{j}=P_{j}\node$, where $P_{j}:\sR^d\to \sR^k$ is a projection. Each expert outputs 
$y_j=f_{j}(\bm u_j)\in\sR$. Then, the final node collects the vector of expert outputs $\bm y=(y_1,\ldots,y_{m_2})\in\sR^{m_2}$.

Each expert node $j$ solves an independent sub-task, such that the global loss function decomposes across the experts. The global loss function and its gradient to expert node $j$ are respectively 
\begin{align}
    \gL:=
    \sum_{j=1}^{m_2}\ell_j(\bm{u}_j,y_j),\qquad\text{and}\ \  \grad_j
    :=
    \frac{\partial \gL}{\partial y_j}
    =
    \frac{\partial \ell_j(\bm u_j,y_j)}{\partial y_j}.
\end{align}

We compare ordinary backpropagation and synthetic-gradient propagation under three sources of stochastic expert feedback, each inspired by a sub-condition in condition $(A)$. 
\begin{itemize}[leftmargin=1em]
    \item \textbf{Case 1: noisy terminal feedback.}
    The external feedback carries i.i.d. noise:
$
    \hat\grad_j=\grad_j+\xi_j,
    \qquad
    \E[\xi_j]=0,\qquad \E[\xi^2_j]=\sigma^2.$
    \item \textbf{Case 2:  stochastic expert activation.}  
    Each expert is activated with probability $p$ independently every time. The gradient feedback is thus    $\hat\grad_j=\frac{\eta_j}{p}\grad_j, \qquad  \eta_j\sim\mathrm{Bernoulli}(p)$.
\item \textbf{Case 3: partial observation.}  In this case, the terminal node sees some other variables $\invar'_j$, independent across $j$, such that the global loss term $\gL'=
    \sum_{j=1}^{m_2}\ell_j(\bm{u}_j,y_j, \invar'_j)$, and its gradient to expert $j$ becomes
    $\hat\grad_j
    =\frac{\partial \gL'}{\partial y_j}=
    \frac{\partial \ell_j(\bm u_j,y_j, \invar'_j)}{\partial y_j}$.
\end{itemize}

The synthetic gradient from expert $j$ to the target node $(1,1)$ needs to be estimated: $\hat\synth^j(\bm{u}_j) \approx \jacob^j\E[\grad_j\mid \invar_j]$.
To characterize sample efficiency when the synthetic gradient estimator is itself estimated from data, we restrict attention to discrete inputs. This yields a tabular setting, analogous to those commonly used in sample-complexity analyses of TD and MC estimators~\citep{cheikhi2023statistical, cheng2024surprising}. We analyze this setting next.

Let $\invar\in \gX$ denote the input variable to our target node $(1,1)\in V_1$. Let the input space $\gX$ be partitioned into $R=2^d$ sets $\gX=\bigsqcup_{r=1}^R \gX_r$.
We write $\invar_r$ for the representative input in partition $r$, and suppose the input variable $\invar$ is sampled from a uniform distribution over these representative inputs. The target node's function is piecewise parameterized:
\begin{align}
    f(\invar; \para) = \sum_{r=1}^R
    \bm 1\{\invar \in\gX_r\}f_{\para_r}(\invar), \quad \text{where}\ \ \para=(\para_1,\ldots,\para_R). %
\end{align}
Let $\node_r=f_{\para_r}(\invar_r)\in\{\pm1\}^d$.
We assume that the map $r\mapsto \node_r$ is one-to-one and onto
$\{\pm1\}^d$. Therefore, drawing $r$ uniformly from $[R]$ is equivalent
to drawing $\node_r$ uniformly from $\{\pm1\}^d$.
Let each expert projection $P_j$ selects exactly $k$ coordinates of $\node_r$.
Thus $\bm{u}_{j,r}:=P_j\node_r\in\{\pm1\}^k$.
For any $\bm{u}\in\{\pm1\}^k$, let the equivalent class under this projection be 
\begin{align}
    \mathcal F_j(r) := \{s\in[R]:P_j\node_s=P_j\node_r\}, \qquad \text{and thus }\ \ \forall r\in [R]:\ |\mathcal F_j(r)|=2^{d-k}.
\end{align}

For each expert $j$ and partition $r$, its has a unique output $y_{j, r}=f_j(\bm u_j)$. let 
\begin{align}
    \grad_{j,r}
    :=
    \frac{\partial \ell_j(\bm u_{j,r},y_{j,r})}{\partial y_j}
\end{align}
denote the exact scalar gradient received by expert $j$. Here, we have fixed all the random variables and thus $\grad_{j,r}$ becomes a constant. 

Since the parameter $\para=(\para_1,\ldots,\para_R)$ is also partitioned, by definition of the piecewise function (\eqref{eq:def:piecewise}), when $\invar_r$ is given, only the parameter $\para^r$ receives gradient feedback:
\begin{align}
    \grad^{\para_r} =  \sum_{j=1}^{m_2}\jacob_r^{\para^r}\jacob^j_r\grad_{j,r}.
\end{align}
Similarly, here $\grad^{\para_r} $ becomes a constant vector.

Because the parameter blocks $\para_r$ are specific to the input $\invar_r$, and each $\invar_r$ is sampled with equal probability, the population mean gradient can be written as
\begin{align}
    \E\, \grad^\para =
    \frac  1 R (\grad^{\para_1},\ldots,\grad^{\para_R}).
\end{align}
Then, for any gradient estimator $\hat\sgrad^\para=\frac{1}{R}(\hat\sgrad^{\para_1},\ldots,\hat\sgrad^{\para_R})$, the estimation MSE decomposes into the sum of per-partition MSE:
\begin{align}
    \smse^2= \E\, \left\|
        \hat\sgrad^\para - \E\,\grad^\para \right\|_2^2
    = \frac1{R^2} \sum_{r=1}^R
    \E\, \|\hat\sgrad^{\para_r}-\grad^{\para_r}\|^2_2=
    \frac1{R^2}
    \sum_{r=1}^R \smse_r^2, \label{eq:mse-to-mser}
\end{align}
where we denote the per-partition MSE by
\begin{align}
    \smse_r^2:=\E\, \|\hat\sgrad^{\para_r}-\grad^{\para_r}\|^2_2.
\end{align}

For every partition $r\in [R]$, we  draw $n$ i.i.d. samples. This is analogue to a phased sampling model used in TD learning analyses (e.g., see~\citep{kearns2000bias}) to isolate the statistical mechanism estimators in a single frozen-parameter phase.  Thus, the samples can be labeled by $(r, i)\in [R]\times [n]$. Define the backpropagation estimator to a parameter block $\para^r$ as
\begin{align}
    \hat\grad^{\para_r}
    :=
    \sum_{j=1}^{m_2}
    \jacob_r^{\para^r}\jacob^j_r
    \left(
        \frac1n\sum_{i=1}^n \hat\grad_{j,r,i}
    \right).
\end{align}
Note that expert $j$'s local observation is $\invar_j=P_j \node=\bm{u}_j$. As the synthetic-gradient generator  approximates $\jacob^j\E[\hat\grad_j\mid \bm{u}_j]$, let the conditional empirical mean as the synthetic gradient generator:
\begin{align}
    \hat \synth^j(\bm{u}_j)
    :=\jacob^j\left( \frac{1}{n\sum_{r} \bm{1}(P_j\node_{r}=\bm{u}_j )} \sum_{r, i} \bm{1}(P_j\node_{r}=\bm{u}_j )\ \hat\grad_{j,r,i} \right)
    =\jacob^j\frac1{n|\mathcal F_j(r)|}
    \sum_{s\in\mathcal F_j(r)}
    \sum_{i=1}^n
    \hat\grad_{j,s,i}.
\end{align}
Thus,  the synthetic-gradient estimator to a parameter block $\para^r$ is
\begin{align}
    \hat\sgrad^{\para_r}
    :=
    \sum_{j=1}^{m_2}
    \jacob_r^{\para^r} \hat\synth_j(\bm{u}_{j,r}).
\end{align}

\begin{theorem}
\label{thm:per-state-mse}
Consider the settings described in this section. Assume that the norm of the gradient feedback to each expert node has a constant norm across experts, i.e., $\forall j, j': |\grad_{j,r}|=|\grad_{j',r}|$, and for Case 3: $\forall j, j': \Var(\hat\grad_{j,r}\mid r)=\Var(\hat\grad_{j',r}\mid r)$ has a constant variance across experts. Then, for all three cases, we have
\begin{align}
    \mse^2=2^{d-k}\, \smse^2.
\end{align}
\end{theorem}

\section{Proof of Theorem~\ref{thm:per-state-mse}}

We first prove the following lemma. 

\begin{lemma}
\label{lem:common-per-state-mse}
Fix a state $r$. Consider the settings and objects described in Theorem~\ref{thm:per-state-mse}. Suppose
\begin{align}
    \hat\grad_{j,s,i}
    =
    \grad_{j,s}+\varepsilon_{j,s,i},
\end{align}
where the variables $\varepsilon_{j,s,i}$ are independent across $j,s,i$, have mean zero, and have variance
\begin{align}
    v_{j,s}:=\E[\varepsilon_{j,s,i}^2].
\end{align}
We have
\begin{align}
    \mse_r^2= \frac1n
    \sum_{j=1}^{m_2}\|\jacob_r^{\para^r}\jacob^j_r\|^2v_{j,r},\qquad \text{and}\ \  
    \smse_r^2=\frac1{n2^{d-k}}
    \sum_{j=1}^{m_2}\|\jacob_r^{\para^r}\jacob^j_r\|^2v_{j,r}.
\end{align}
In particular, when $v_{j,r}$ is a constant across expert $j$, we have 
\begin{align}
    \mse_r^2=2^{d-k}\smse^2_r.
\end{align}
\end{lemma}

\begin{proof}
For ordinary backpropagation, define
\begin{align}
    \overline \varepsilon_{j,r}
    :=
    \frac1n\sum_{i=1}^n \varepsilon_{j,r,i}.
\end{align}
Then
\begin{align}
    \frac1n\sum_{i=1}^n \hat\grad_{j,r,i}
    =
    \grad_{j,r}+\overline \varepsilon_{j,r},
\end{align}
and therefore
\begin{align}
    \hat\grad^{\para_r}-\grad^{\para_r}
    =
    \sum_{j=1}^{m_2}\jacob_r^{\para^r}\jacob^j_r\overline \varepsilon_{j,r}.
\end{align}
Since $\E[\overline \varepsilon_{j,r}]=0$, the estimator is unbiased. Also,
\begin{align}
    \E[\overline \varepsilon_{j,r}^2]
    =
    \frac{v_{j,r}}{n}.
\end{align}
The variables $\overline \varepsilon_{j,r}$ are independent across $j$.
Hence the cross terms vanish, and
\begin{align}
\begin{aligned}
    \mse_r^2=\E
    \left[
        \left\|
        \hat\grad^{\para_r}-\grad^{\para_r}
        \right\|^2
    \right]
    &=
    \E
    \left[
        \left\|
        \sum_{j=1}^{m_2}\jacob_r^{\para^r}\jacob^j_r\overline \varepsilon_{j,r}
        \right\|^2
    \right] \\
    &=
    \sum_{j=1}^{m_2}
    \|\jacob_r^{\para^r}\jacob^j_r\|^2
    \E[\overline \varepsilon_{j,r}^2] =
    \frac1n
    \sum_{j=1}^{m_2}\|\jacob_r^{\para^r}\jacob^j_r\|^2v_{j,r}.
\end{aligned}
\end{align}

We next analyze the synthetic-gradient estimator. Since
$\grad_{j,s}=\grad_{j,r}$ for every $s\in\mathcal F_j(r)$,
\begin{align}
\begin{aligned}
    \hat\synth_j(\node_r)
    &=
    \frac1{n|\mathcal F_j(r)|}
    \sum_{s\in\mathcal F_j(r)}
    \sum_{i=1}^n
    (\grad_{j,s}+\varepsilon_{j,s,i}) \\
    &=
    \grad_{j,r}
    +
    \frac1{n|\mathcal F_j(r)|}
    \sum_{s\in\mathcal F_j(r)}
    \sum_{i=1}^n
    \varepsilon_{j,s,i}.
\end{aligned}
\end{align}
Define
\begin{align}
    \widetilde \varepsilon_{j,r}
    :=
    \frac1{n|\mathcal F_j(r)|}
    \sum_{s\in\mathcal F_j(r)}
    \sum_{i=1}^n
    \varepsilon_{j,s,i}.
\end{align}
Then
\begin{align}
    \hat\synth_j(\node_r)=\grad_{j,r}+\widetilde \varepsilon_{j,r}.
\end{align}
Since the feedback errors are mean zero,
\begin{align}
    \E[\synth_j(\node_r)]=\grad_{j,r},
\end{align}
and hence
\begin{align}
    \E[\hat\sgrad^{\para_r}]=\grad^{\para_r}.
\end{align}

Because $P_j$ selects exactly $k$ coordinates and $\node_r$ ranges
uniformly over $\{\pm1\}^d$,
\begin{align}
    |\mathcal F_j(r)|=2^{d-k}.
\end{align}
Using $v_{j,s}=v_{j,r}$ for $s\in \gF_j(r)$ and independence,
\begin{align}
    \E[\widetilde \varepsilon_{j,r}^2]
    =
    \frac{v_{j,r}}{n|\mathcal F_j(r)|}
    =
    \frac{v_{j,r}}{n2^{d-k}}.
\end{align}
Finally,
\begin{align}
    \hat\sgrad^{\para_r}-\grad^{\para_r}
    =
    \sum_{j=1}^{m_2}\jacob_r^{\para^r}\jacob^j_r\widetilde \varepsilon_{j,r}.
\end{align}
The feedback-error variables are independent across experts $j$, so the cross
terms again vanish. Hence
\begin{align}
\begin{aligned}
    \smse_r^2=\E
    \left[
        \left\|
        \hat\sgrad^{\para_r}-\grad^{\para_r}
        \right\|^2
    \right]
    &=
    \sum_{j=1}^{m_2}
    \|\jacob_r^{\para^r}\jacob^j_r\|^2
    \E[\widetilde \varepsilon_{j,r}^2] \\
    &=
    \frac1{n2^{d-k}}
    \sum_{j=1}^{m_2}\|\jacob_r^{\para^r}\jacob^j_r\|^2 v_{j,r}.
\end{aligned}
\end{align}
Therefore, when $v_{j,r}$ is a constant across expert $j$, we have 
\begin{align}
    \mse_r^2=2^{d-k}\smse^2_r.
\end{align}
\end{proof}

\textbf{Theorem \ref{thm:per-state-mse}.}
Consider the settings described in this section. Assume that the norm of the gradient feedback to each expert node has a constant norm across experts, i.e., $\forall j, j': |\grad_{j,r}|=|\grad_{j',r}|$, and for Case 3: $\forall j, j': \Var(\hat\grad_{j,r}\mid r)=\Var(\hat\grad_{j',r}\mid r)$ has a constant variance across experts. Then, for all three cases, we have
\begin{align}
    \mse^2=2^{d-k}\, \smse^2.
\end{align}

\begin{proof}
The proof for all three cases are applications of
Lemma~\ref{lem:common-per-state-mse}.

For Case 1, set $\varepsilon_{j,s,i}=\xi_{j,s,i}$. Then, $\varepsilon_{j,s,i}$ is i.i.d. across $j, s, i$ with zero mean and constant variance $v_{s}=\sigma^2$.
Invoking  Lemma~\ref{lem:common-per-state-mse} gives the result for Case 1.

For Case 2, set
\begin{align}
    \varepsilon_{j,s,i}
    =
    \left(\frac{\eta_{j,s,i}}{p}-1\right)\grad_{j,s}.
\end{align}
Since $\E[\eta_{j,s,i}/p]=1$, the error has mean zero. Also,
\begin{align}
    \E
    \left[
        \left(
        \frac{\eta_{j,s,i}}{p}-1
        \right)^2
    \right]
    =
    \frac{1-p}{p},
\end{align}
and hence
\begin{align}
    v_{j,s}
    =
    \frac{1-p}{p}\grad_{j,s}^2.
\end{align}
As each $\grad_{j,s}^2$ is constant across expert $j$, we know $v_{j, s}$ is constant across experts. Then, invoking Lemma~\ref{lem:common-per-state-mse} proves the result for Case 2.

It remains to prove Case 3.  Recall
\begin{align}
    \hat\grad_{j}
    =\frac{\partial \gL'}{\partial y_j}=
    \frac{\partial \ell_j(\bm u_j,y_j, \invar'_j)}{\partial y_j}.
\end{align}
The mean gradient for parameter block $\para^s$ is
\begin{align}
    \E\,\grad_s^{\para_s}= \E\left[\jacob_s^{\para_s}\sum_{j\in [m_2]} \jacob^j_s \hat\grad_j \mid \invar_s\right]= \jacob_s^{\para_s}\sum_{j\in [m_2]} \jacob^j_s \ \E\, [\hat\grad_j\mid \invar_s] . 
\end{align}
Define $\bar\grad_{j,s}:=E\, [\hat\grad_j\mid \invar_s]$ and set
\begin{align}
    \epsilon_{j,s,i}= \hat\grad_{j,s,i}-\bar\grad_{j,s}.
\end{align}
By definition, $\epsilon_{j,s,i}$ has zero mean, independent across $j, s, i$. Moreover, by assumption in the theorem statement, its variance $v_{j, s}=\E[\epsilon_{j,s,i}^2]$ is constant across experts. Therefore, invoking Lemma~\ref{lem:common-per-state-mse} proves:
\begin{align}
    \mse_r^2=2^{d-k}\, \smse_r^2, \quad \forall r\in [R].
\end{align}
Then, applying \eqref{eq:mse-to-mser} proves the theorem.

\end{proof}

\section{Experiment Details}
\label{sec:full-experiments}

This section provides the implementation details and full results for the experiments in Section~\ref{sec:exp}.

\subsection{Simulation of the Expert Network}
\label{subsec:full-exp-expert}

We first simulate the expert network from Section~\ref{sec:example}. For each dimension $d$, the input space is represented by $R=2^d$ partition representatives. The target node output is a bijection onto Boolean states $\node_r\in\{-1,+1\}^d$, and the target node is piecewise parameterized as $\para=(\para_1,\ldots,\para_R)$, so only block $\para_r$ receives gradient on state $r$. In the simulation, the local map for each block is linear, $f_{\para_r}(x_r)=\bm{A}_r\para_r$, with $\bm{A}_r\sim \gN(0,\tfrac{1}{d}\bm I)$; the Boolean state $\node_r$ is used as the expert-visible representation, while $\bm{A}_r$ supplies the local Jacobian for pulling gradients back to $\para_r$.

The second layer contains $m$ expert nodes. Expert $j$ observes a $k$-coordinate projection $\bm{u}_{j,r}=P_j\node_r\in\{-1,+1\}^k$, computes $y_{j,r}=\bm a_j^\top \bm{u}_{j,r}+b_j$, and is trained against a linear target $t_j(\bm{u}_{j,r})=\bm \beta_j^\top \bm{u}_{j,r}+\gamma_j$ with squared loss. The clean scalar expert feedback is therefore $g_{j,r}=y_{j,r}-t_j(\bm{u}_{j,r})$, which depends only on $P_j\node_r$ and is constant on the preimage $\gF(r):=\{s:P_j\node_s=P_j\node_r\}$ of size $2^{d-k}$.

We compare ordinary backpropagation, which averages stochastic feedback separately at each state, with synthetic-gradient propagation, which estimates the conditional expert feedback table by pooling all states sharing the same projected expert input. The full population gradient is computed by exact enumeration over all $R$ representatives; estimator error is the squared Euclidean distance to this exact gradient.

We evaluate the three stochasticity sources from Section~\ref{sec:example}: noisy terminal feedback, stochastic expert activation, and partial observation. For noisy terminal feedback, each expert receives its clean feedback corrupted by independent additive Gaussian noise, $\widehat{g}_{j,r}=g_{j,r}+\xi_j$, where $\xi_j\sim\mathcal{N}(0,\sigma^2)$. For stochastic expert activation, expert $j$ is active with probability $p$, and its feedback is reweighted as $\widehat{g}_{j,r}=(\eta_j/p)g_{j,r}$ with $\eta_j\sim\operatorname{Bernoulli}(p)$. For partial observation, the expert target depends on an unobserved zero-mean Gaussian variable $q_j\sim\mathcal{N}(0,\tau^2)$, producing feedback $\widehat{g}_{j,r}=g_{j,r}+\alpha_jq_j$, where each $\alpha_j$ is a scale constant sampled once from the standard Gaussian for each expert. Each construction is unbiased, so $\mathbb{E}[\widehat{g}_{j,r}\mid \mathbf{z}_r]=g_{j,r}$.

The sweep uses $d\in\{4,\ldots,12\}$, $k\in\{2,3,4\}$, $m=5$ experts and $20$ Monte Carlo trials. Every possible input variable is sampled once. The reported metric is the ratio $\Mse^2/\sMse^2$ between the squared error of the backpropagation estimator and the squared error of the synthetic-gradient estimator for the full population gradient.

\begin{figure}[t]
    \centering
    \begin{subfigure}[b]{0.32\textwidth}
        \centering
        \includegraphics[width=\linewidth]{figures/expert-example-figures/ratio_vs_d_noisy.png}
        \caption{Noisy terminal feedback.}
        \label{fig:expert-ratio-noisy}
    \end{subfigure}
    \hfill
    \begin{subfigure}[b]{0.32\textwidth}
        \centering
        \includegraphics[width=\linewidth]{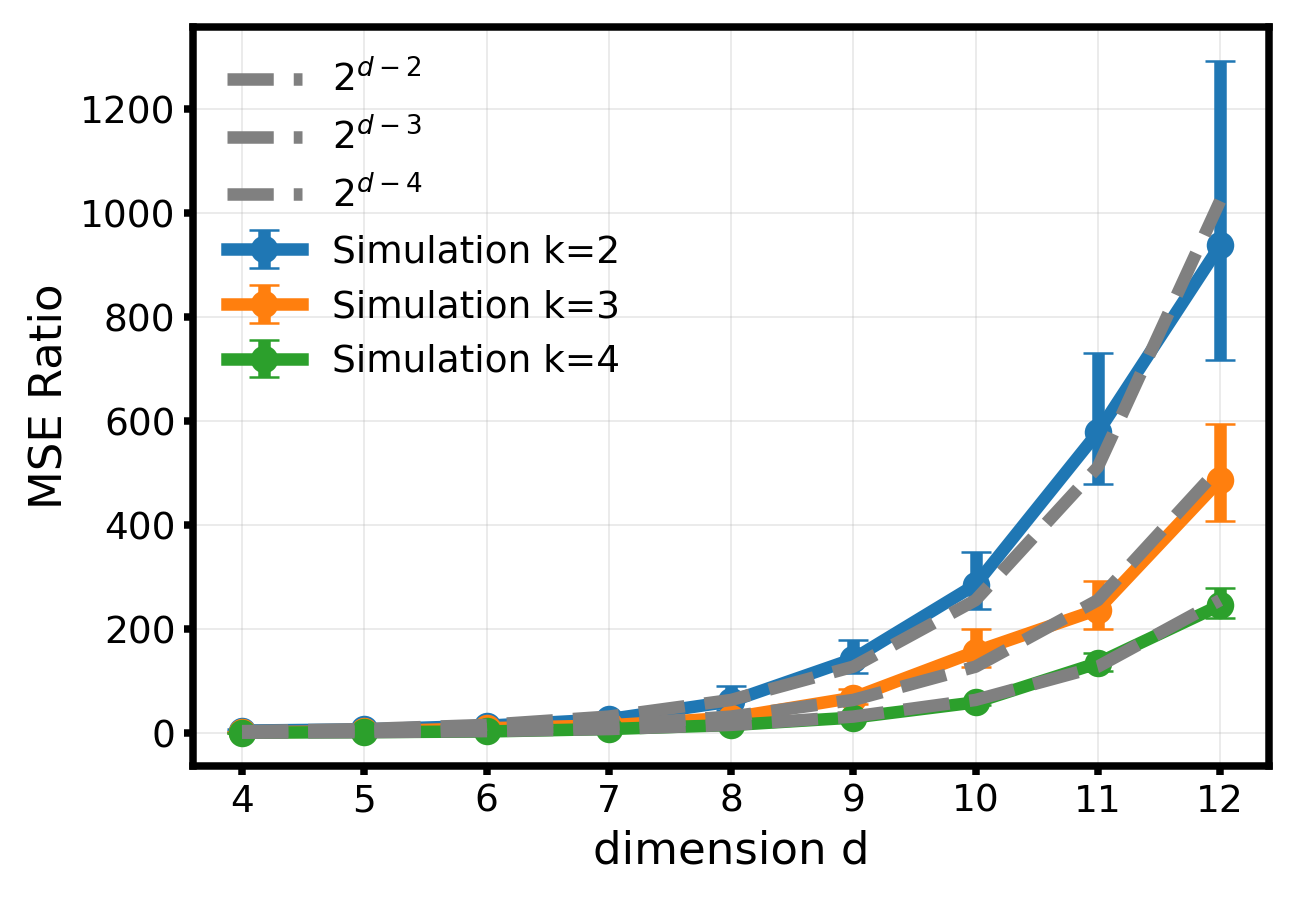}
        \caption{Stochastic expert activation.}
        \label{fig:expert-ratio-activation}
    \end{subfigure}
    \hfill
    \begin{subfigure}[b]{0.32\textwidth}
        \centering
        \includegraphics[width=\linewidth]{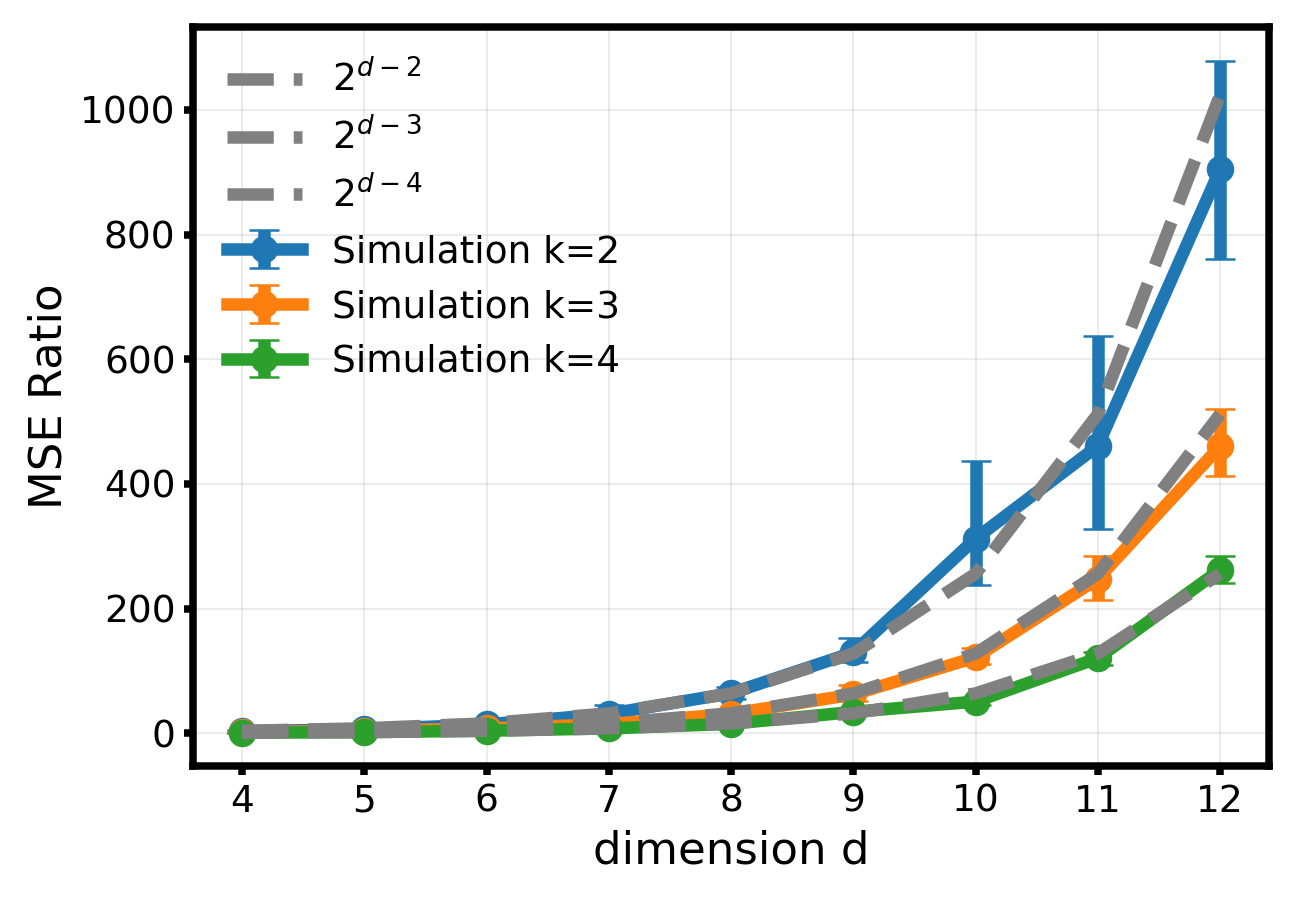}
        \caption{Partial observation.}
        \label{fig:expert-ratio-partial}
    \end{subfigure}
    \caption{\textbf{Finite-partition expert simulation.}  MSE ratio $\Mse^2/\sMse^2$ between the squared error of the backpropagation estimator and the squared error of the synthetic-gradient estimator for the full population gradient. In all three cases, the ratio grows with dimension reduction $d-k$, matching the predicted advantage from theory.}
    \label{fig:expert-ratio-d}
\end{figure}

\subsection{SparseNet and Adaptive Gradient Mixing}
\label{subsec:full-exp-sparsenet}

The MNIST bandit and POPGym experiments use the same SparseNet neuron template. A SparseNet hidden state is a collection of $W$ neurons,
\begin{align}
    H=\{h_i\}_{i=1}^W,\qquad h_i\in\sR^q,
\end{align}
where $q$ is called the neuron dimension. Each neuron connects to $n_{\mathrm{parents}}$ other neurons, where the connections are uniformly randomly picked and fixed at initialization.
Each neuron observes the concatenated outputs $\bm{u}\in \sR^d$ from its parent neurons and optionally from external inputs.

Each neuron $i$ contains $M=n_{\mathrm{slots}}$ slot maps. For slot $s$, the policy parameters are a key $k_{i,s}$, a matrix $A_{i,s}$, and a bias $b_{i,s}$. The neuron first routes its local input over slots,
\begin{align}
    \alpha_{i,s}
    =
    \frac{
    \exp\left(\langle \bm{u}_i,k_{i,s}\rangle/\sqrt{d}\right)}
    {\sum_{r=1}^M
    \exp\left(\langle \bm{u}_i,k_{i,r}\rangle/\sqrt{d}\right)},
\end{align}
then computes slot outputs
\begin{align}
    y_{i,s}=\tanh(A_{i,s}\bm{u}_i+b_{i,s}),
\end{align}
and returns the mixture
\begin{align}
    \bar h_i=\sum_{s=1}^M\alpha_{i,s}y_{i,s}.
\end{align}
For MNIST, hidden layers after the first use the residual update implemented in the sweep:
\begin{align}
    h_{\ell,i}
    =
    \tanh(\bar h_{\ell,i})
    + h_{\ell-1,i},
    \qquad \ell>0,
\end{align}
while POPGym uses the sparse transition recurrently through time. 
The MNIST action layer flattens the final hidden state and forms scaled dot-product logits. The POPGym actor uses an analogous slot mechanism at the action layer: for each action, actor keys route over action slots and the slot logits are mixed to produce the final action logit.

This neuron design makes the synthetic gradient predictor convenient in a sparse network: it uses the same local input $\bm{u}_i$ and the same routing weights $\alpha_{i,s}$, but has predictor-specific slot parameters $\tilde A_{i,s},\tilde b_{i,s}$:
\begin{align}
    \hat\synth_i(\bm{u}_i)
    =
    \sum_{s=1}^M
    \alpha_{i,s}
    \left(\tilde A_{i,s}\bm{u}_i+\tilde b_{i,s}\right).
\end{align}
Thus, the predictor estimates an incoming hidden-state gradient in $\sR^q$ using only the information available at that neuron.

The adaptive $\lambda$ rule estimates the per-neuron bias-variance tradeoff from the current mini-batch. Let $\sss_b$ be the received gradient on this neuron's parameter on sample $b$. Let $\qqq_b$ be the gradient on this neuron's parameter induced by its synthetic gradient prediction. We use the batch empirical mean of $\sss_b$ to approximate the true expected gradient, and use \eqref{eq:lambda-star} to calculate a mixing weight $\lambda_{\mathrm{batch}}$ for this batch. In practice, we found it to be more stable to normalize the gradients to unit $\ell_2$ norm before estimating the mixing weight. 
Then, the mixing weight is updated with momentum,
\begin{align}
    \lambda
    \leftarrow
    \mu\lambda+(1-\mu)\lambda_{\mathrm{batch}}.
\end{align}

\subsection{MNIST Contextual Bandit}
\label{subsec:full-exp-mnist}

The experiment is a contextual bandit simulation. The context is an MNIST image, and the action set consists of the class predictions $\{0,1,2, 3,4,5,6,7,8,9\}$. The policy samples an action from its softmax distribution. The true reward is one for a correct sampled label and zero otherwise; during training this reward is flipped with probability $p_{\mathrm{flip}}$, while evaluation uses the noiseless reward.

The SparseNet policy is multi-layer sparse feedforward network (SparseNet neuron operation and adaptive gradient-mixing rule detailed in Section~\ref{subsec:full-exp-sparsenet}). The first layer samples $d_{\mathrm{in}}=qn_{\mathrm{parents}}$ random  raw pixel coordinates per neuron, where $q$ is the neuron dimension and $n_{\mathrm{parents}}$ is the number of parents. In later layers, each neuron receives input from the concatenation of the states of its $n_{\mathrm{parents}}$ parents. Sparse connections are uniformly randomly picked, fixed at initialization. In this experiment, it has five sparse layers of width $200$, neuron output dimension $10$, neuron slots number $5$, and number of parents $5$ in the base setting. The final action layer flattens the last sparse hidden state and uses learned action keys to produce logits. We compare SparseNet trained by synthetic gradient propagation (sgprop), backpropagation (backprop), and a dense MLP backpropagation baseline. The MLP uses five layers, width $745$, matching the SparseNet parameter count. For every batch, once the gradient is estimated, we use Adam to optimize every network. 
We plot the evaluation rewards along training steps in Figure~\ref{fig:mnist-main-results}.

We observe that each method has a different optimal learning rate. For example, as synthetic gradient reduces variance, it may benefit from a bigger learning rate. we use their corresponding optimal learning rate according to Figure~\ref{fig:mnist-ablations} (a). More ablations (reward noise scale, number of parents, adaptive mixing weight dynamics) are shown in Figure~\ref{fig:mnist-ablations}; detailed hyperparameters are  detailed in Table~\ref{tab:mnist-settings}.  

\begin{table}[t]
    \centering
    \caption{\textbf{MNIST contextual bandit settings.} Learning rate are chosen to be their optimal ones respectively according to Figure~\ref{fig:mnist-ablations} (a). MLP layers and width are chosen to match the parameter count of SparseNet. }
    \label{tab:mnist-settings}
    \begin{tabular}{@{}p{0.2\textwidth}p{0.3\textwidth}p{0.34\textwidth}@{}}
        \toprule
        Component & Setting & Value \\
        \midrule
        Task & Base reward flip probability & $p_{\mathrm{flip}}=0.4$ \\
        Training & Steps / seeds & $1500$ / $\{0,1,2,3,4,5\}$ \\
        Training & Batch size / eval batch size & $64$ / $4000$ \\
        SparseNet & synthetic grad learning rate & $0.005$ \\
        SparseNet & Layers / width / neuron dim & $5$ / $200$ / $10$ \\
        SparseNet & Num. of parents & $5$  \\
        SparseNet SGProp & Base  learning rate & $0.005$ \\
        SparseNet backprop & Base  learning rate & $0.001$ \\
        MLP backprop & Layers / width / learning rate & $5$ / $745$ / $2\cdot 10^{-5}$ \\
        \bottomrule
    \end{tabular}
\end{table}

\begin{figure}[t]
    \centering
    \begin{subfigure}[b]{0.32\textwidth}
        \centering
        \includegraphics[width=\linewidth]{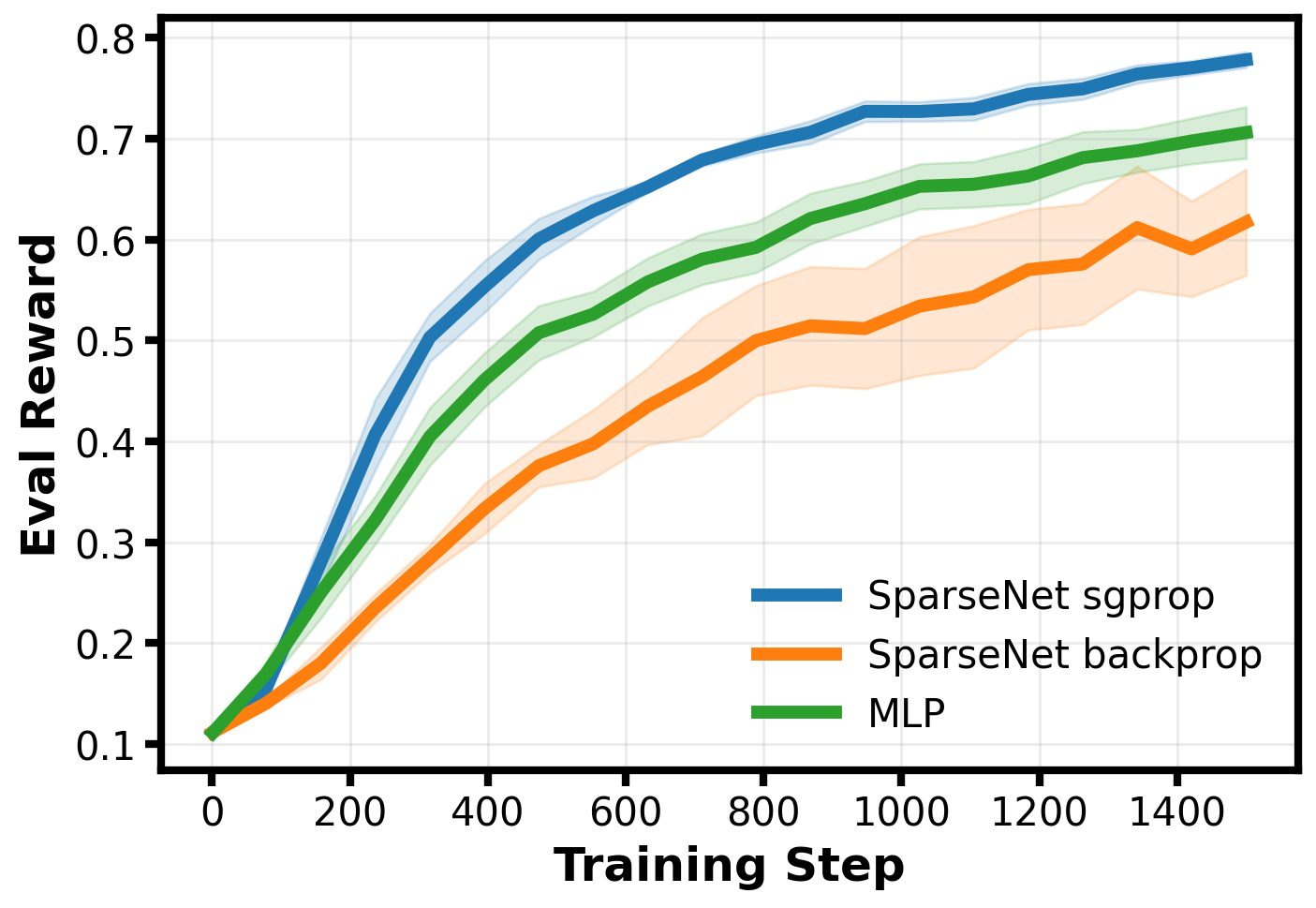}
        \caption{Training dynamics}
        \label{fig:mnist-base-training}
    \end{subfigure}
    \begin{subfigure}[b]{0.32\textwidth}
        \centering
        \includegraphics[width=\linewidth]{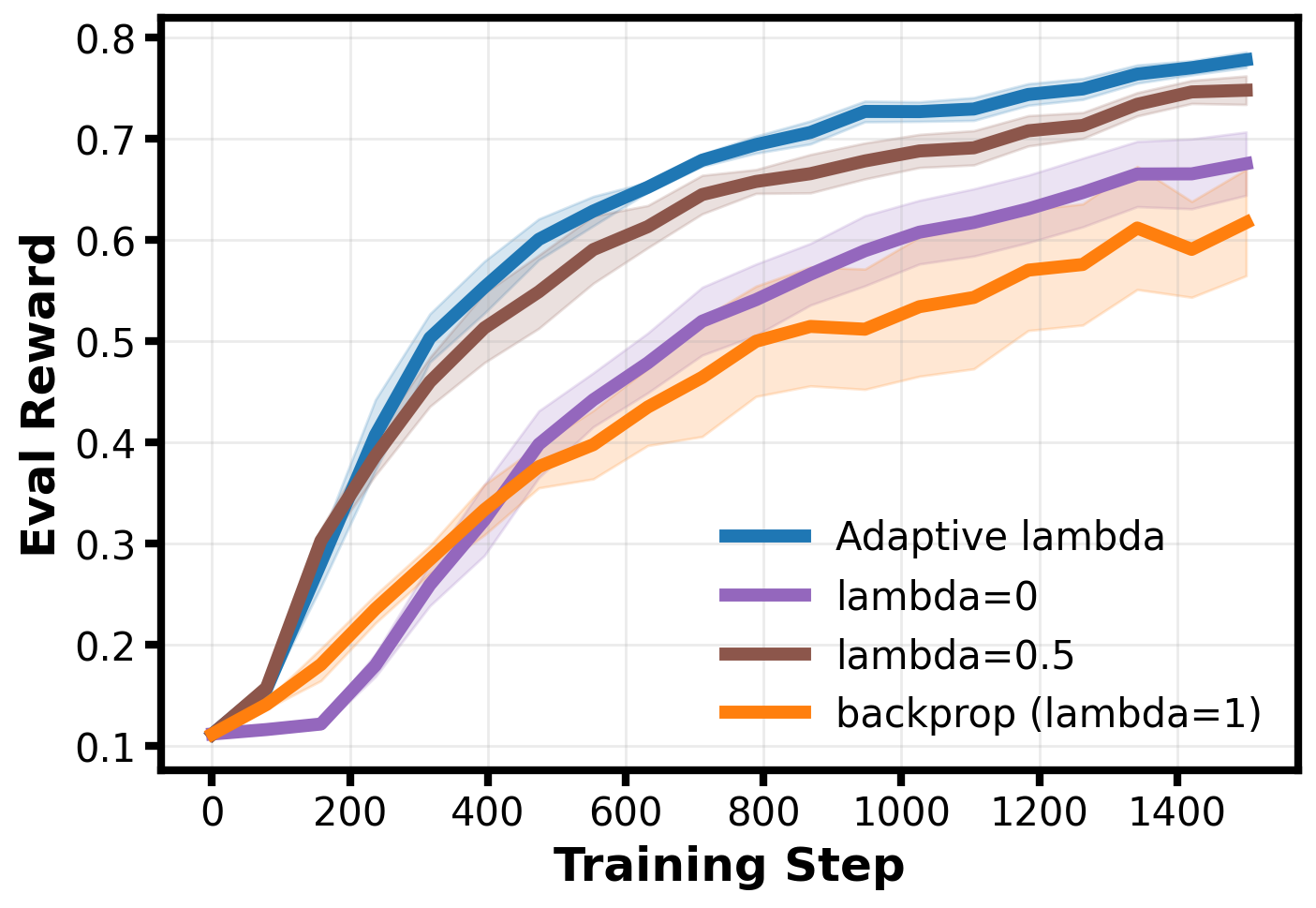}
        \caption{Compare adaptive and fixed-$\lambda$}
        \label{fig:mnist-theta-ablation-training}
    \end{subfigure}
    \caption{\textbf{MNIST contextual bandit.} Mean evaluation rewards along training steps. Each curve aggregates six seeds, with the standard error of the mean (SEM) reported. For SparseNet, synthetic gradient improves sample efficiency relative to backpropagation, out performing the dense MLP baseline. The fixed-$\lambda$ comparison shows that the benefit depends on balancing synthetic and backpropagated gradients. }
    \label{fig:mnist-main-results}
\end{figure}

\begin{figure}[t]
    \centering
    \begin{subfigure}[b]{0.33\textwidth}
        \centering
        \includegraphics[width=\linewidth]{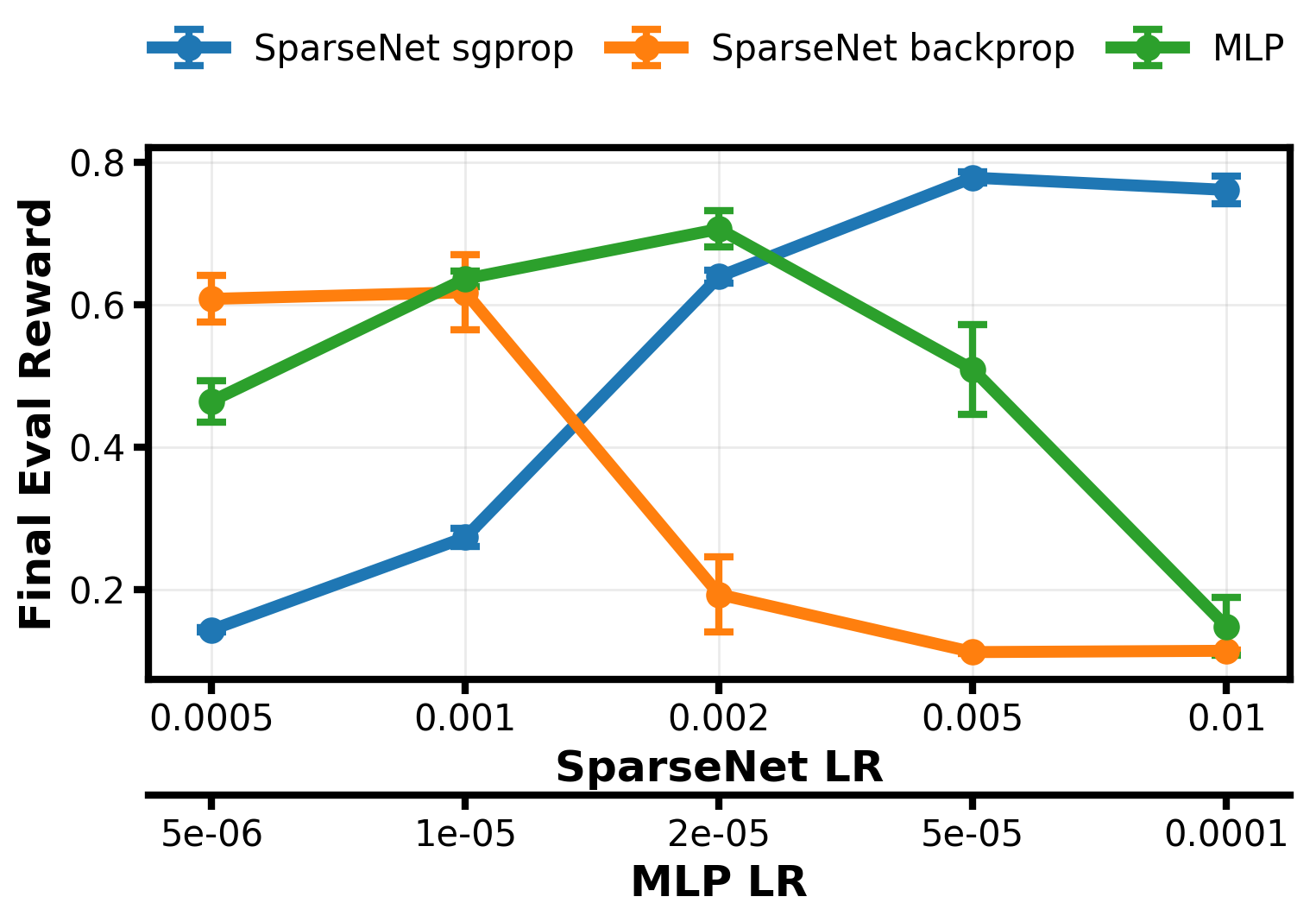}
        \caption{Learning rate.}
        \label{fig:mnist-policy-lr}
    \end{subfigure}
    \begin{subfigure}[b]{0.33\textwidth}
        \centering
        \includegraphics[width=\linewidth]{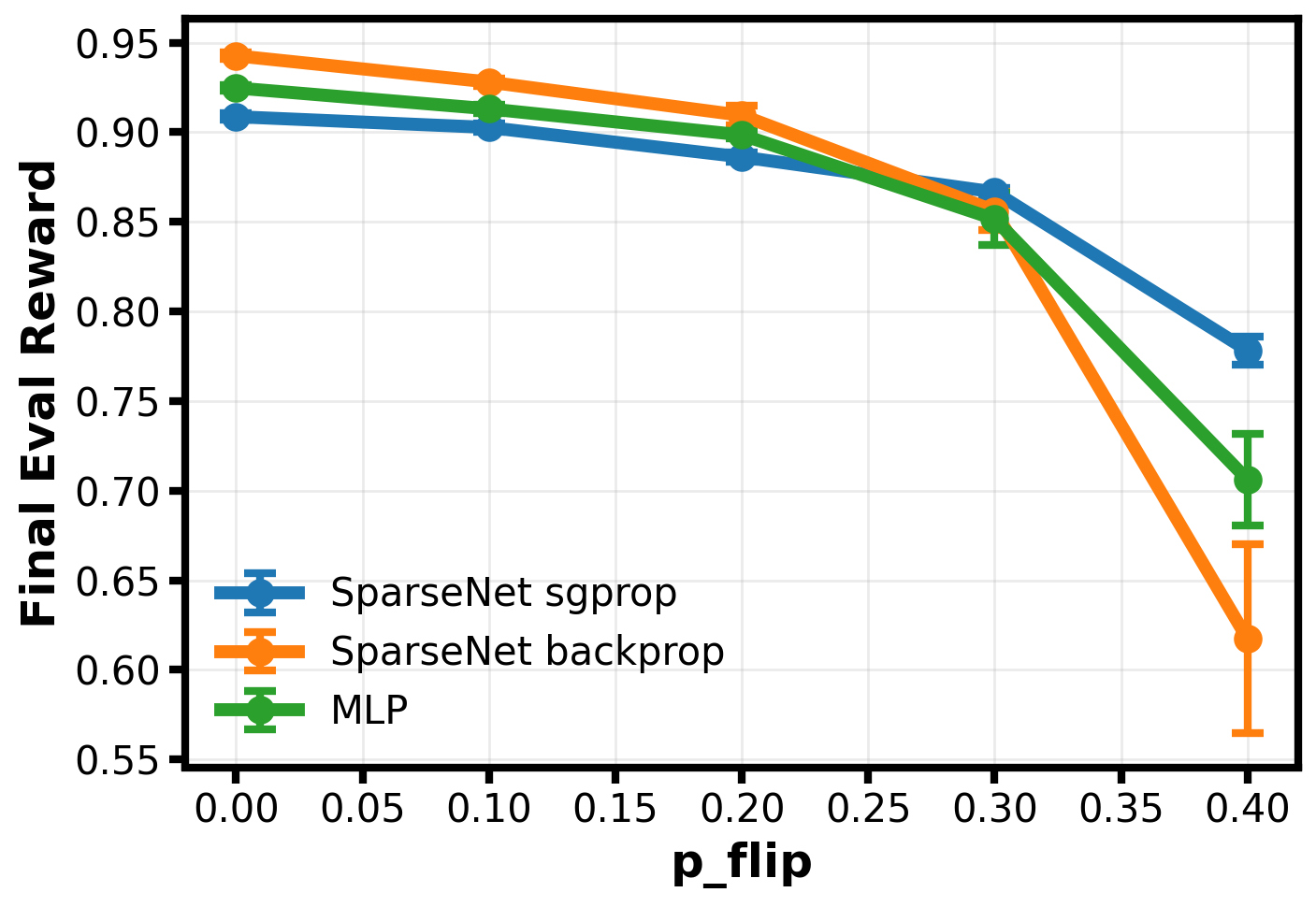}
        \caption{Scale of the reward noise.}
        \label{fig:mnist-p-flip}
    \end{subfigure}

    \begin{subfigure}[b]{0.33\textwidth}
        \centering
        \includegraphics[width=\linewidth]{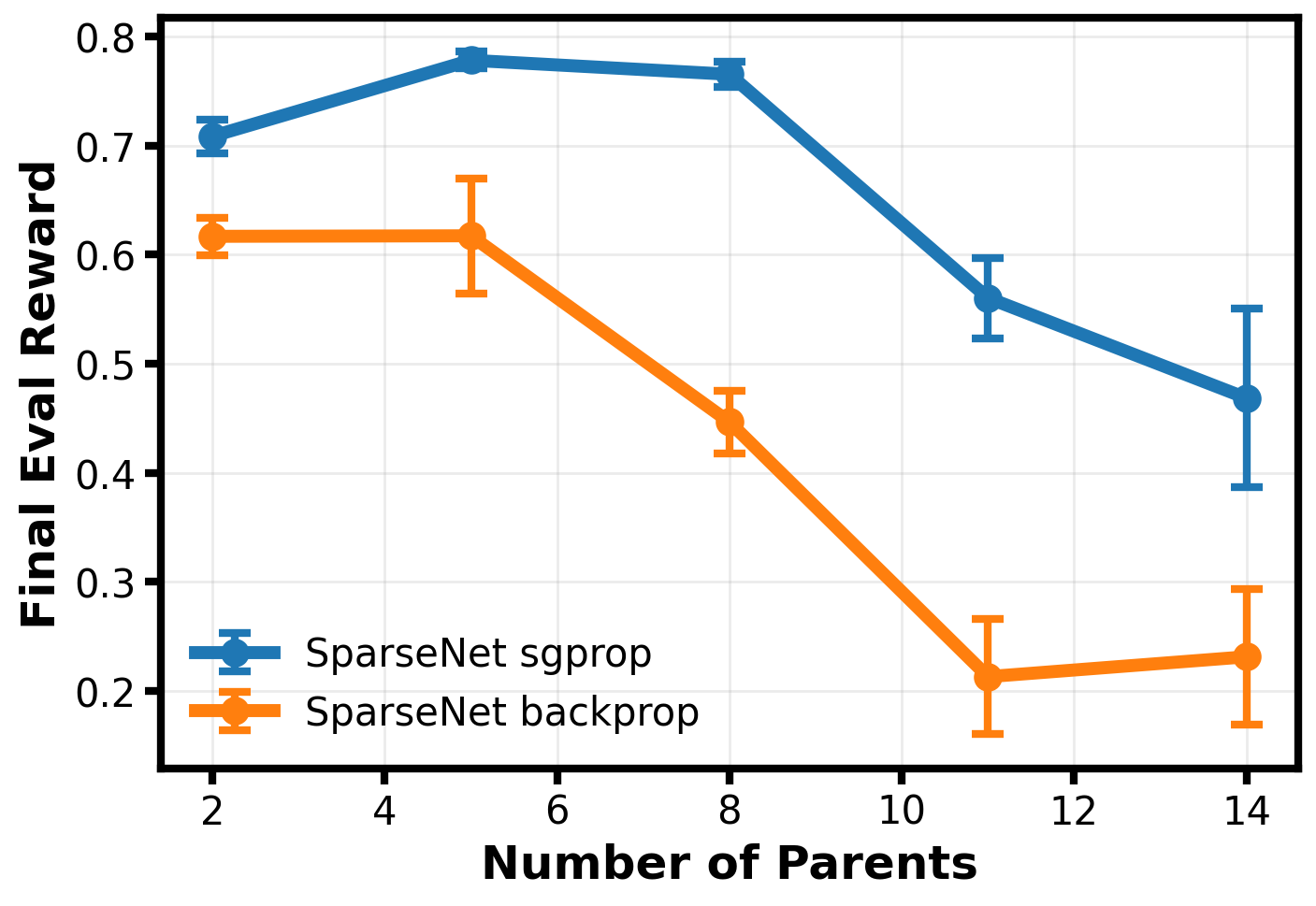}
        \caption{Num. of parents.}
        \label{fig:mnist-num-parents}
    \end{subfigure}
    \begin{subfigure}[b]{0.33\textwidth}
        \centering
        \includegraphics[width=\linewidth]{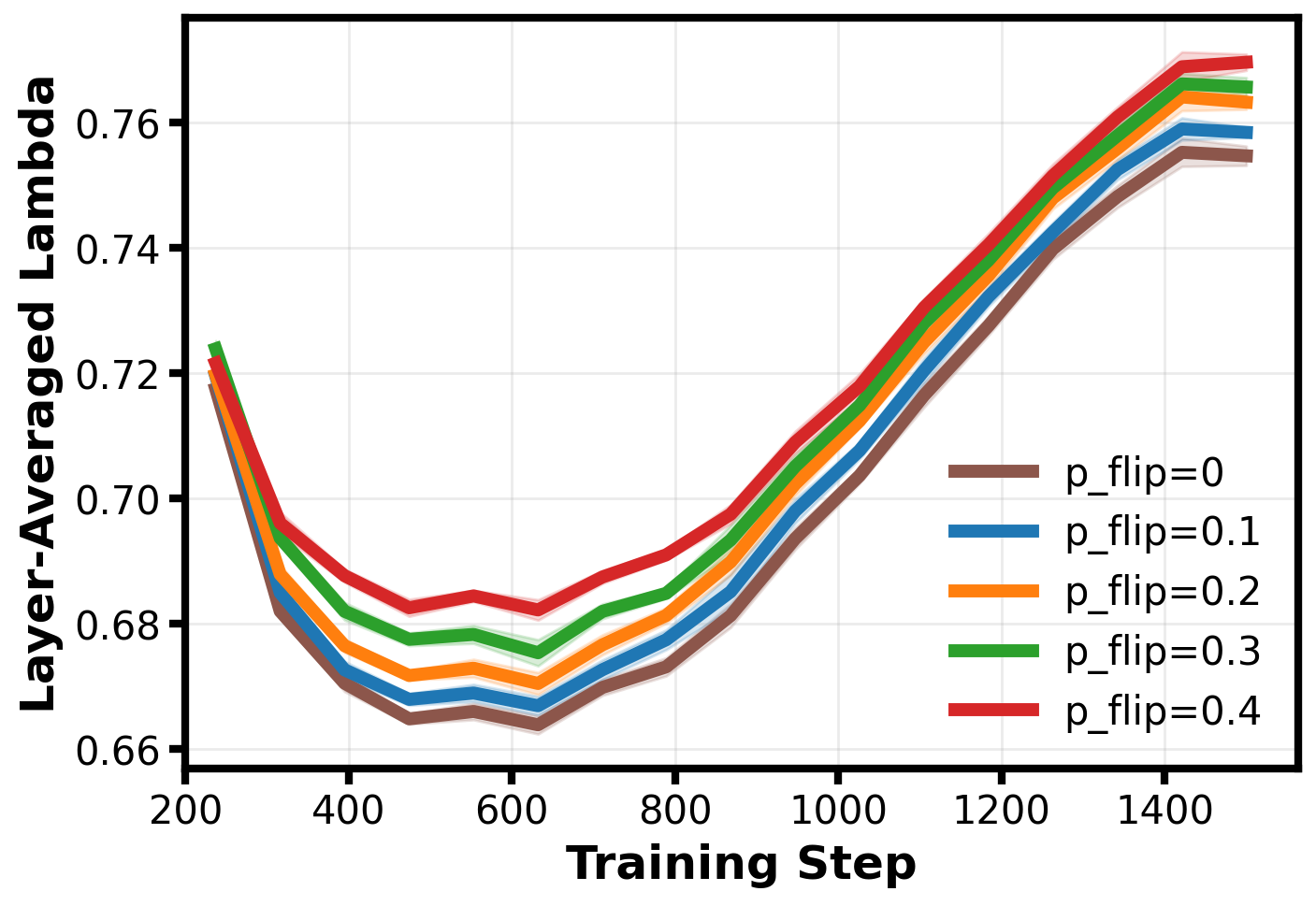}
        \caption{Adaptive $\lambda$ along training steps.}
        \label{fig:mnist-adaptive-theta-p-flip}
    \end{subfigure}
    \caption{\textbf{MNIST contextual bandit ablations.} (a) Each method has different optimal learning rate due. (b) Synthetic gradient propagation shows advantages in high variance regime. (c) Synthetic gradient propagation favors sparse connectivity. (d) Dynamics of the adaptive mixing weight along training steps.}
    \label{fig:mnist-ablations}
\end{figure}

\subsection{POPGym Reinforcement Learning}
\label{subsec:full-exp-popgym}

\begin{figure}[t]
    \centering
    \captionsetup[subfigure]{font=footnotesize,skip=2pt}
    \begin{subfigure}[b]{0.31\textwidth}
        \centering
        \includegraphics[width=\linewidth]{figures/popgym-figures/escape_envs20_training_dynamics.png}
        \caption{\texttt{escape}, 20 envs: training.}
        \label{fig:popgym-escape20-training}
    \end{subfigure}
    \hfill
    \begin{subfigure}[b]{0.31\textwidth}
        \centering
        \includegraphics[width=\linewidth]{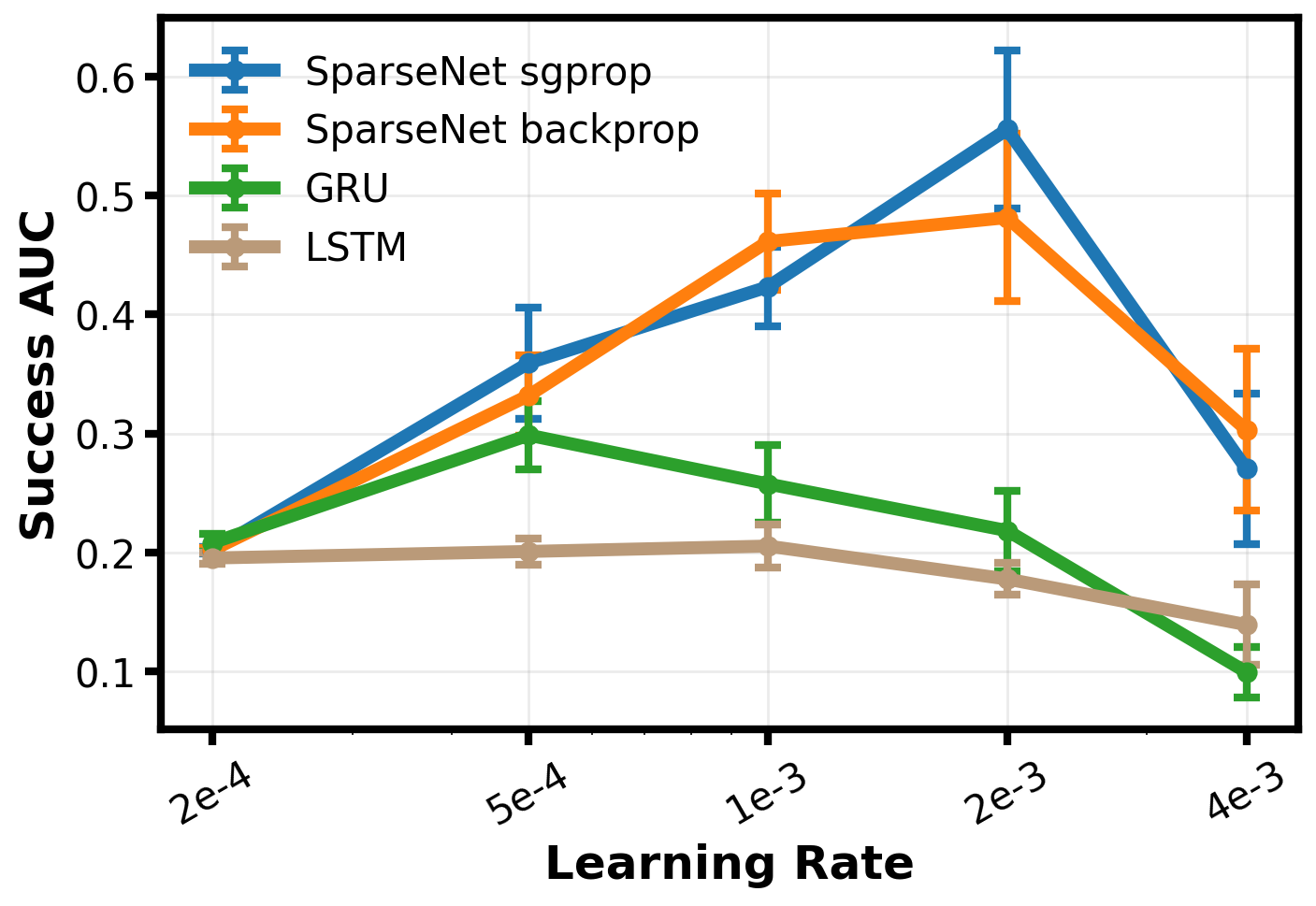}
        \caption{\texttt{escape}, 20 envs: LR sweep.}
        \label{fig:popgym-escape20-lr}
    \end{subfigure}
    \hfill
    \begin{subfigure}[b]{0.31\textwidth}
        \centering
        \includegraphics[width=\linewidth]{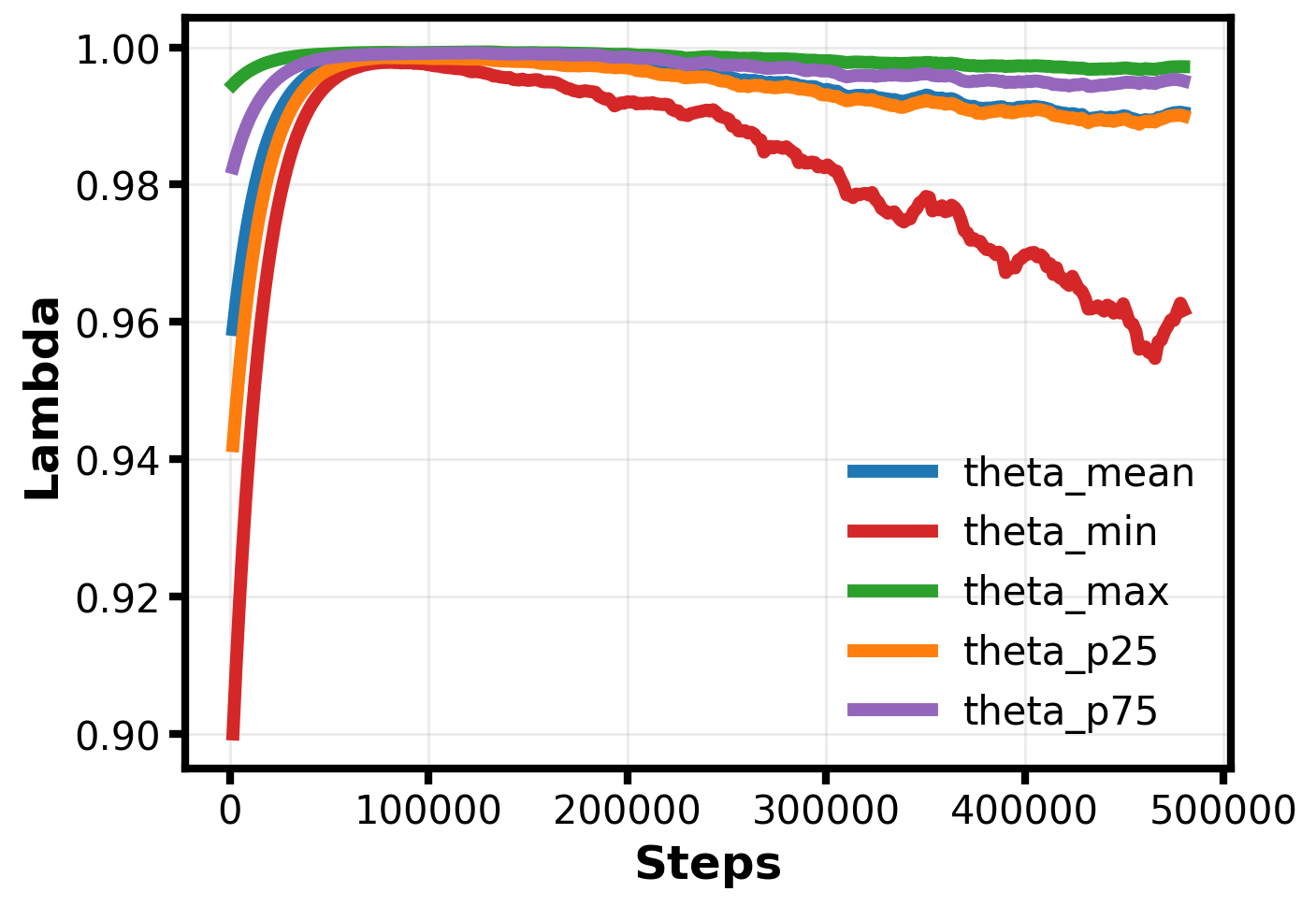}
        \caption{\texttt{escape}, 20 envs: $\lambda$.}
        \label{fig:popgym-escape20-theta}
    \end{subfigure}

    \vspace{0.35em}
    \begin{subfigure}[b]{0.31\textwidth}
        \centering
        \includegraphics[width=\linewidth]{figures/popgym-figures/escape_envs40_training_dynamics.png}
        \caption{\texttt{escape}, 40 envs: training.}
        \label{fig:popgym-escape40-training}
    \end{subfigure}
    \hfill
    \begin{subfigure}[b]{0.31\textwidth}
        \centering
        \includegraphics[width=\linewidth]{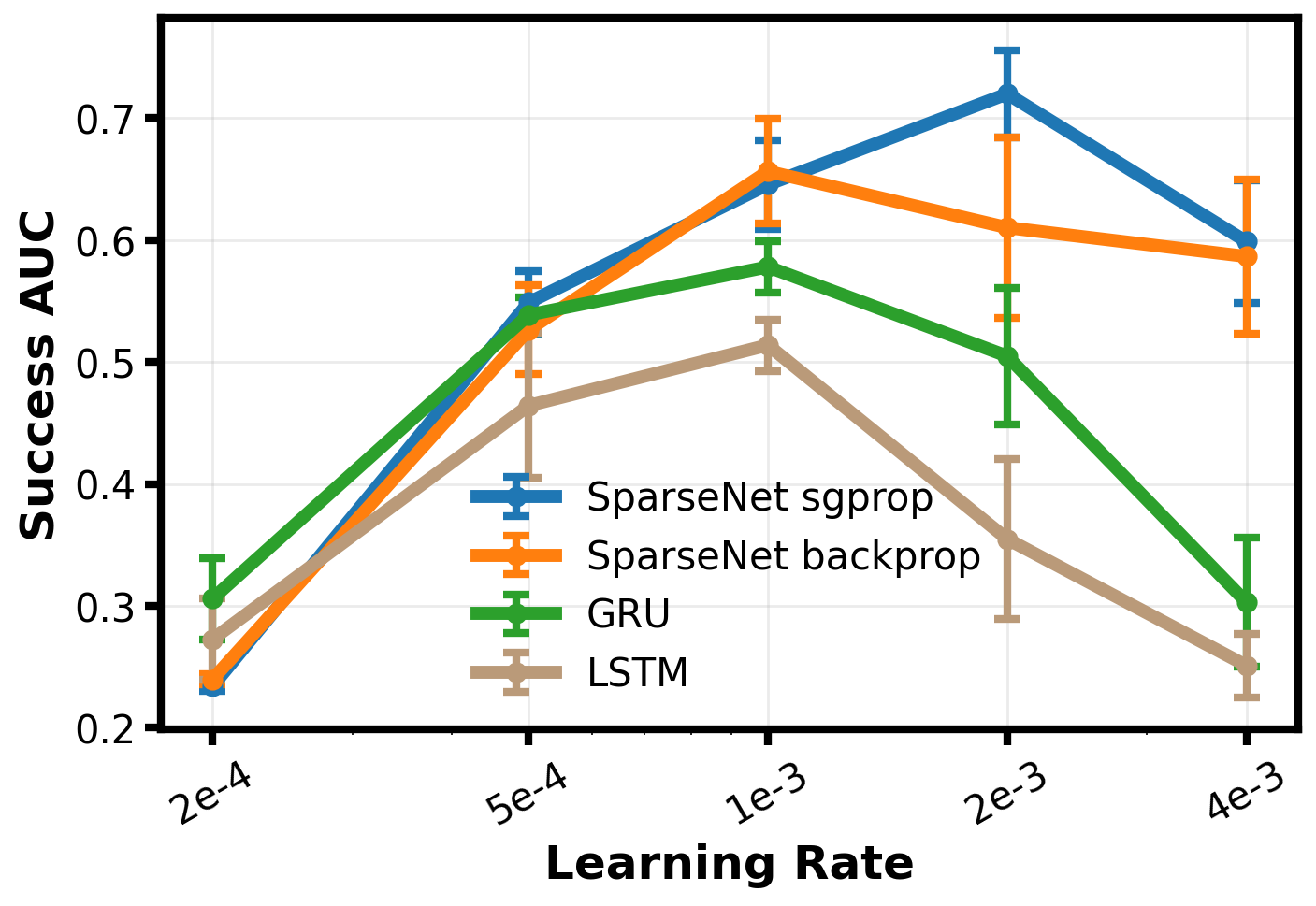}
        \caption{\texttt{escape}, 40 envs: LR sweep.}
        \label{fig:popgym-escape40-lr}
    \end{subfigure}
    \hfill
    \begin{subfigure}[b]{0.31\textwidth}
        \centering
        \includegraphics[width=\linewidth]{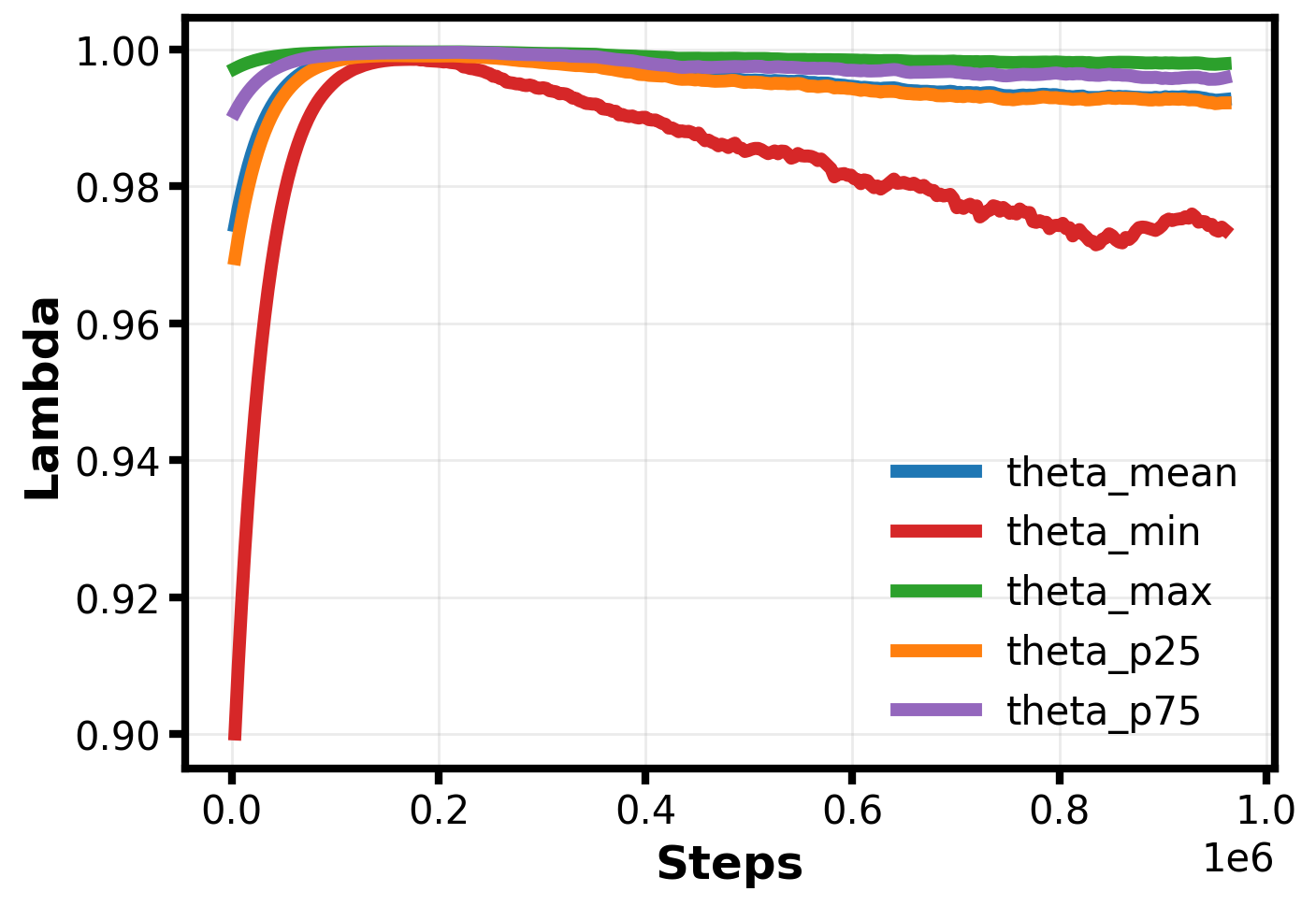}
        \caption{\texttt{escape}, 40 envs: $\lambda$.}
        \label{fig:popgym-escape40-theta}
    \end{subfigure}

    \vspace{0.35em}
    \begin{subfigure}[b]{0.31\textwidth}
        \centering
        \includegraphics[width=\linewidth]{figures/popgym-figures/explore_envs20_training_dynamics.png}
        \caption{\texttt{explore}, 20 envs: training.}
        \label{fig:popgym-explore20-training}
    \end{subfigure}
    \hfill
    \begin{subfigure}[b]{0.31\textwidth}
        \centering
        \includegraphics[width=\linewidth]{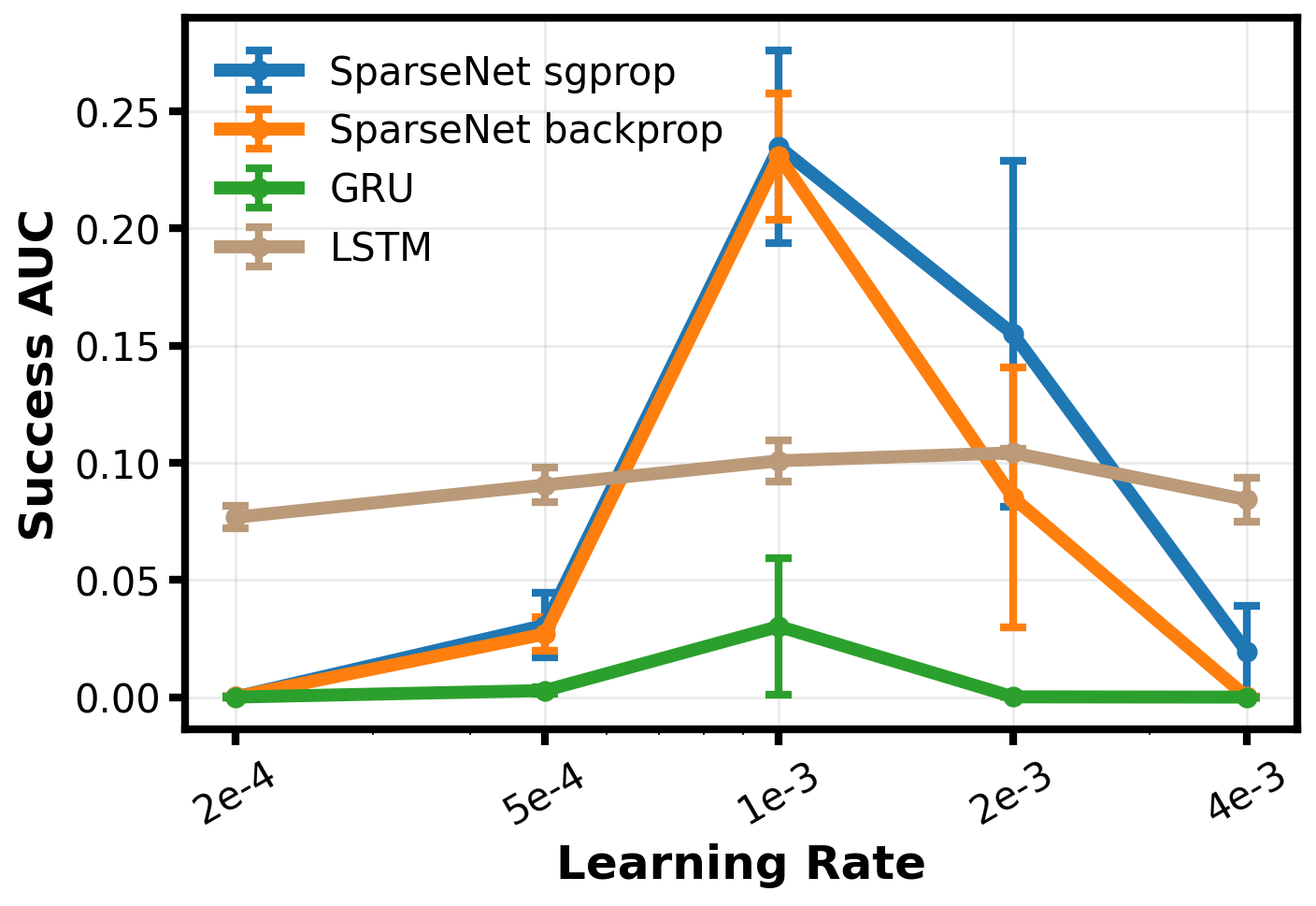}
        \caption{\texttt{explore}, 20 envs: LR sweep.}
        \label{fig:popgym-explore20-lr}
    \end{subfigure}
    \hfill
    \begin{subfigure}[b]{0.31\textwidth}
        \centering
        \includegraphics[width=\linewidth]{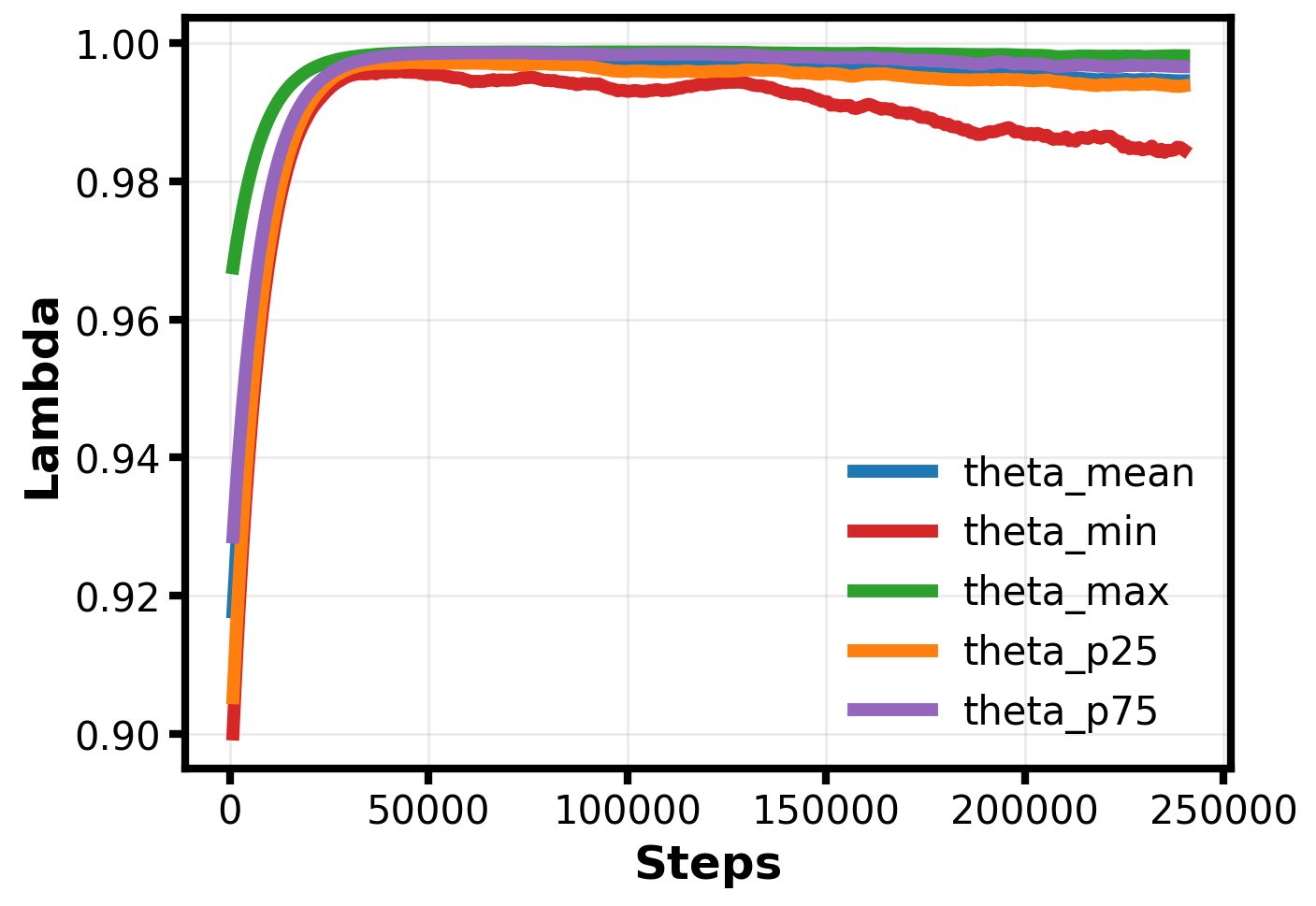}
        \caption{\texttt{explore}, 20 envs: $\lambda$.}
        \label{fig:popgym-explore20-theta}
    \end{subfigure}

    \vspace{0.35em}
    \begin{subfigure}[b]{0.31\textwidth}
        \centering
        \includegraphics[width=\linewidth]{figures/popgym-figures/explore_envs40_training_dynamics.png}
        \caption{\texttt{explore}, 40 envs: training.}
        \label{fig:popgym-explore40-training}
    \end{subfigure}
    \hfill
    \begin{subfigure}[b]{0.31\textwidth}
        \centering
        \includegraphics[width=\linewidth]{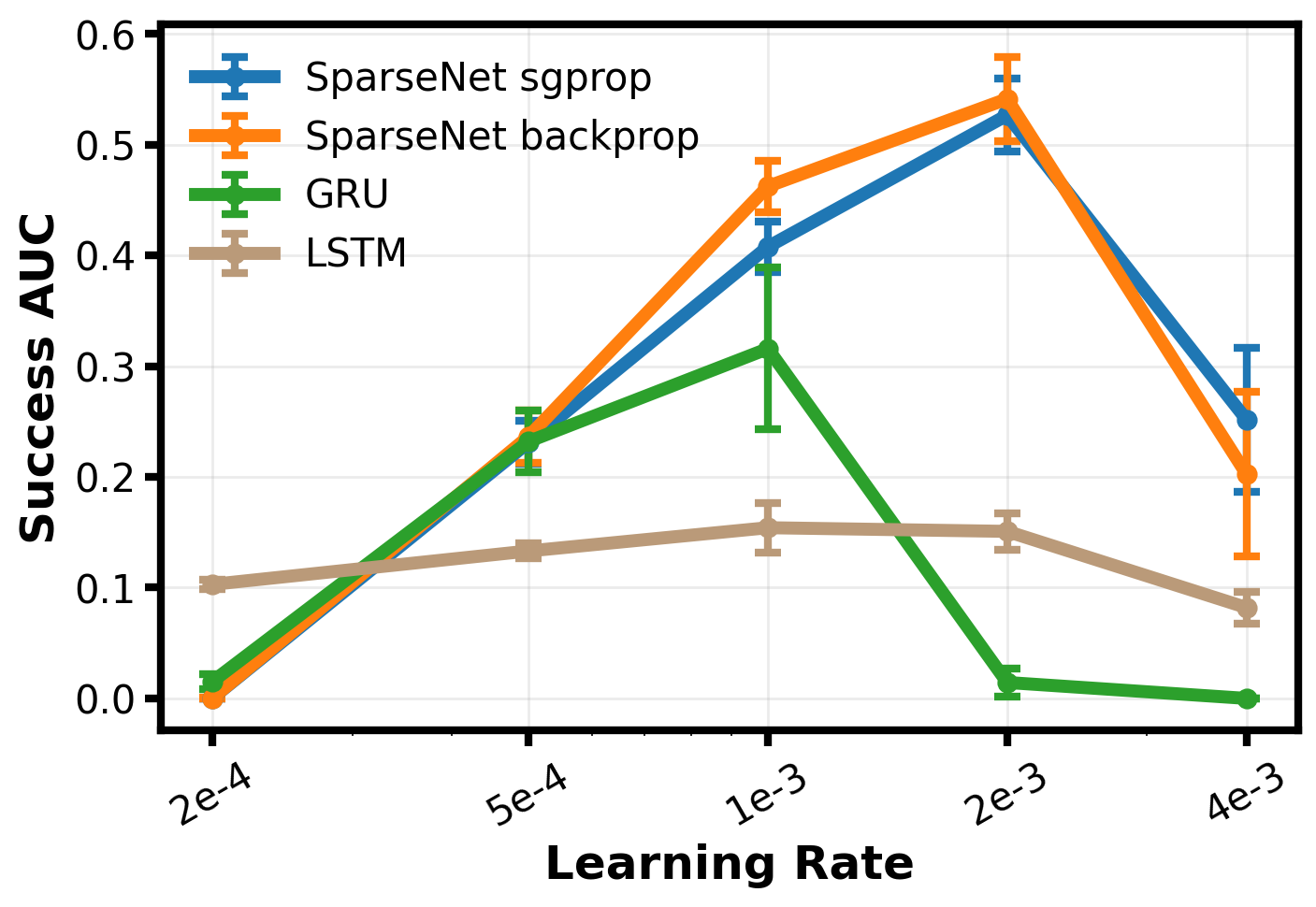}
        \caption{\texttt{explore}, 40 envs: LR sweep.}
        \label{fig:popgym-explore40-lr}
    \end{subfigure}
    \hfill
    \begin{subfigure}[b]{0.31\textwidth}
        \centering
        \includegraphics[width=\linewidth]{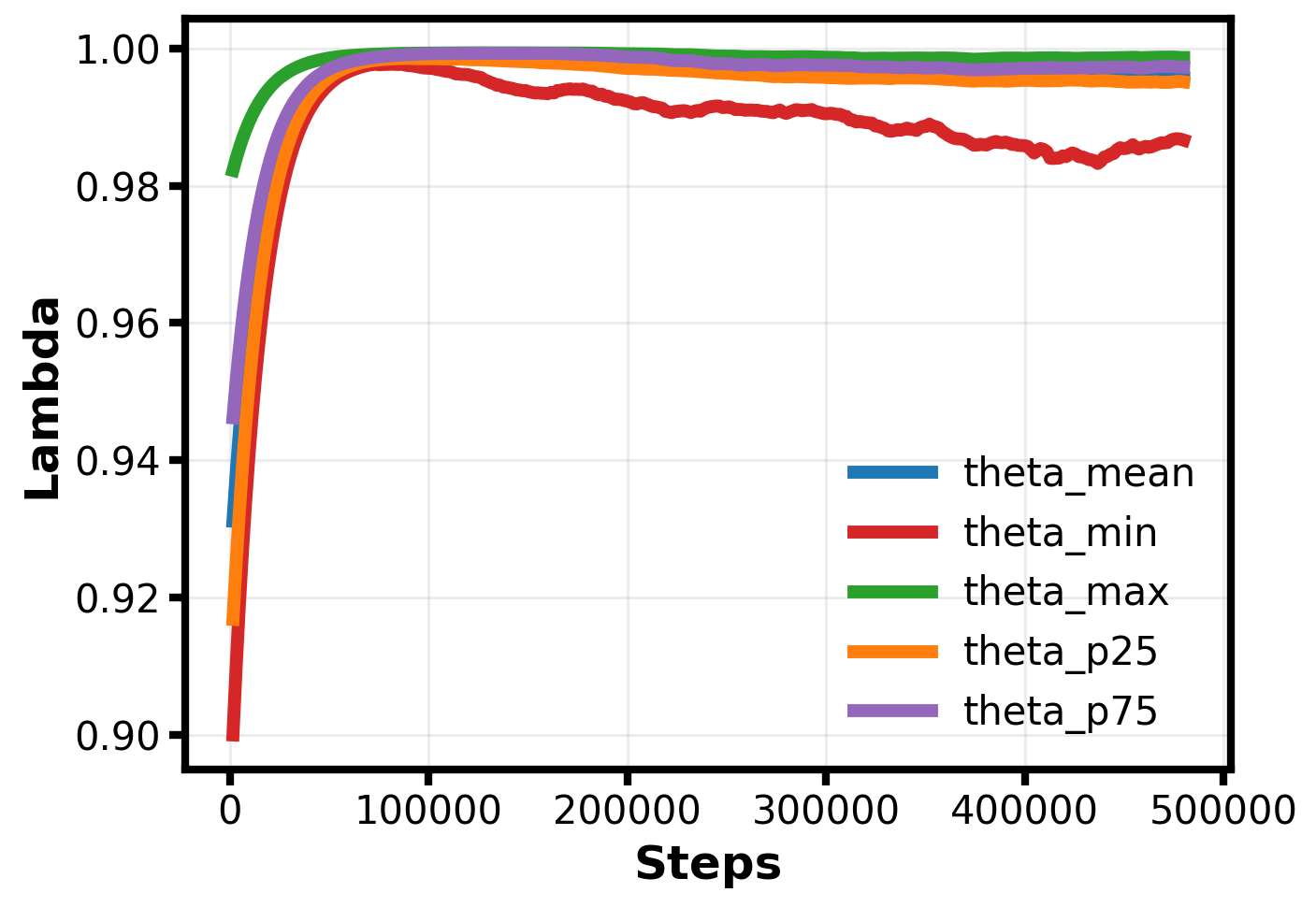}
        \caption{\texttt{explore}, 40 envs: $\lambda$.}
        \label{fig:popgym-explore40-theta}
    \end{subfigure}
    \caption{\textbf{POPGym labyrinth results.} Each row shows one environment and parallel-environment setting; columns show training dynamics, learning-rate sensitivity measured by the area-under-curve (AUC) of the success rate curve, and the adaptive mixing weight $\lambda$ along steps. The more parallel environments, the more sampled trials the agent sees before each parameter update. Results show that synthetic gradient propagation helps the most in few-sample regime. Each model has different optimal learning rate, and the training dynamis show their optimal ones. }
    \label{fig:popgym-results}
    \label{fig:popgym-escape20}
    \label{fig:popgym-escape40}
    \label{fig:popgym-explore20}
    \label{fig:popgym-explore40}
\end{figure}

This experiment tests synthetic gradients in a recurrent SparseNet for partially observable RL. We use two maze navigation environments from POPGym: \texttt{labyrinth\_escape} (the agent escapes the maze) and \texttt{labyrinth\_explore} (the agent explores the entire maze). If the agent successfully escapes the maze in \texttt{labyrinth\_escape} or successfully visits all maze squares in \texttt{labyrinth\_explore}, a terminal reward $r=1$ is given. For both environments, there is a small penalty $r=-0.01$ for each step taken. In the exploration task, a reward $r=1/(\text{total maze squre count})$ is given with probability $0.75$ for each unseen square the agent visits. 
For both environments we use maze size $(6, 6)$. Both POPGym labyrinth environments are partially observable. The raw maze observation (adjacent tiles) is wrapped with the previous action and then flattened as the resulting observation.  A learnable linear mapping is used  to map the observation from the environment to an input embedding. For SparseNet, we use $128$ neurons, $16$ of which are used to receive the input embedding (each receives the same input embedding), while the other neurons only receive inputs from their connected neurons.

Training uses a full-episode REINFORCE policy gradient update through backpropagation through time (BPTT): episodes are collected from parallel environments, shorter episodes are padded and masked, discounted Monte Carlo returns are computed, and one policy gradient update is applied. For \texttt{labyrinth\_escape} we use an episode length of $80$; for \texttt{labyrinth\_escape} we use an episode length of $40$. Rollout length is the same as the episode length in both environments. As mentioned in Appendix~\ref{subsec:full-exp-sparsenet}, to estimate the mixing weight adaptively, since the true gradient in  \eqref{eq:lambda-star} is not available, we simply use the downstream propagated gradient to replace the true gradient. This, however, could accumulate excessive bias in the gradient estimation, especially for earlier time steps. To make sure that the portion of synthetic gradients is consistent across time steps, we employ the following scheduling heuristic.
Each neuron estimates its mixing weight $\theta$ normally, but a schedule is applied during  gradient propagation through time. For node $i$ in the BPTT graph, let $\mathcal{C}(i)$ be its children and define the mean child deficit as $\delta_{t+1,i}=|\mathcal{C}(i)|^{-1}\sum_{j\in\mathcal{C}(i)}(1-\theta_{t+1,j})$. The update is $\theta_{t,i}=\min\{1,\theta_{t+1,i}+\rho\,\delta_{t+1,i}\}$, so earlier values increase in proportion to the downstream deficit and are maxed out at $1$. In this experiment, we use $\rho=0.1$.

SparseNet is recurrent in this experiment; the common SparseNet functions and adaptive mixing rule are described in Section~\ref{subsec:full-exp-sparsenet}.  Here, the synthetic gradient mixing is applied during BPTT on the time unrolled computational graph. We compare SparseNet sgprop, SparseNet backprop, GRU~\citep{cho2014learning}, and LSTM~\citep{schmidhuber1997long}. The GRU and LSTM baselines use backpropagation and use the same observation preprocessing and actor head layout as SparseNet. The recurrent hidden states of all models are reset to zero at the episode starting step. For the baseline choices: LSTM and GRU are both popular RNN architectures, and in particular GRU is shown to be the best in POPGym environments, outperforming other baselines including MLP, LSTM, and Fast Autoregressive Transformers (FART)~\citep{katharopoulos2020transformers}.

The results are plotted and discussed in Figure~\ref{fig:popgym-results}. Other detailed experimental settings are shown in Table~\ref{tab:popgym-models} and Table~\ref{tab:popgym-sweeps}. 

\begin{table}[h!]
    \centering
    \caption{\textbf{POPGym model dimensions.}}
    \label{tab:popgym-models}
    \begin{tabular}{@{}p{0.12\textwidth}p{0.8\textwidth}@{}}
        \toprule
        Model & Dimensions \\
        GRU & input size $128$, hidden size $256$, actor hidden size $128$, one recurrent layer \\
        LSTM & input size $128$, hidden size $256$, actor hidden size $128$, one recurrent layer \\
        SparseNet & input size $128$, num. of neuron $128$, neuron dim $6$,  num. of slots $6$, input neurons $16$,  each neuron is sparsely connected with  $6$ parents\\
        \bottomrule
    \end{tabular}
\end{table}

\begin{table}[h!]
    \centering
    \caption{\textbf{POPGym settings for \texttt{labyrinth\_escape} and  \texttt{labyrinth\_explore}.} }
    \label{tab:popgym-sweeps}
    \begin{tabular}{@{}p{0.45\textwidth}p{0.3\textwidth}p{0.1\textwidth}@{}}
        \toprule
        Parallel environments & $\{20,40\}$\\
        Learning rates & \multicolumn{2}{p{0.5\textwidth}@{}}{$\{2\cdot10^{-4},5\cdot10^{-4},10^{-3},2\cdot10^{-3},4\cdot10^{-3}\}$} \\
        Synhtetic grad predictor learning rate & \multicolumn{2}{p{0.5\textwidth}@{}}{$10^{-5}$} \\
        Seeds & \multicolumn{2}{p{0.5\textwidth}@{}}{$\{0,1,2,3,4,5\}$} \\
        Total batches & \multicolumn{2}{p{0.5\textwidth}@{}}{$300$} \\
        Discount factor & \multicolumn{2}{p{0.5\textwidth}@{}}{$\gamma=0.95$ } \\
        Entropy coefficient & \multicolumn{2}{p{0.5\textwidth}@{}}{$0.01$} \\
        \bottomrule
    \end{tabular}
\end{table}

\section{Discussion: Gradient-Estimation MSE and Condition $(B)$}\label{sec:discuss-appendix}

\subsection{The Bridge from Gradient-Estimation MSE to Optimization Convergence}\label{subsec:discuss-appendix-mse}
We present the core principles behind how lower gradient MSE leads to better optimization convergence under smoothness objectives. Suppose we apply stochastic gradient descent $ \theta_{t+1}\leftarrow \theta_t-\eta \hat{g}_t^\theta$ with step size $\eta$ on model parameters $\theta$ to a loss function $\ell(\theta)$, where $\hat{g}_t^\theta$ is the gradient estimator of interest. Suppose $\ell$ is $\beta$-smooth with respect to the model parameters $\theta$, a standard assumption in optimization theory, meaning 
$$ \ell(\theta')\le \ell(\theta)+\langle\nabla\ell(\theta),\theta'-\theta\rangle+\frac{\beta}{2}\|\theta'-\theta\|^2 $$ 
for any parameters $\theta, \theta'$. Defining the gradient-estimation MSE as $\mathbb E\!\left[\|\hat{g}_t^\theta-\nabla\ell(\theta_t)\|^2\right]\leq \Delta^2$, using step size  $\eta\leq 1/\beta$, we can obtain the one-step update inequality via standard algebraic manipulations   
$$\mathbb E_t[\ell(\theta_{t+1})]\le \ell(\theta_t)-\frac{\eta}{2}\|\nabla\ell(\theta_t)\|^2+\frac{\eta}{2}\Delta^2. $$
Thus, we can see that a stochastic-gradient step achieves the usual descent from the true gradient, up to an additive penalty proportional to the estimator MSE. 

Beyond the principled one-step analysis above, gradient estimation MSE is shown to be useful in analyzing various SGD convergence in general. For an extended convergence analysis of SGD with biased gradient estimators, for example, see Theorem 1 in \citep{karimi2019non}, section 3 and section 4 in~\citep{driggs2022biased}, or section 4 in~\citep{demidovich2023guide}.

Moreover, we provide direct experimental diagnostics comparing gradient-estimation MSE and training performance, as shown in Section~\ref{subsec:ablation-mse}.

\subsection{Discussion on When We Can Expect Condition $(B)$ to Hold in General}\label{subsec:discuss-appendix-cond-B}

We illustrate that Condition $(B)$ is similar to a network capability condition. Consider a layered directed computational graph and a node $v$. For each sample, the node receives a local input $z_v$ from preceding nodes and produces an output $x_v=f_v(z_v)$. We ask the question: how capable can this node be in helping to optimize the overall loss? To measure this capability potential, we compare it to an idealized target: let $Z$collect the additional upstream information that could, in principle, be used to choose a better $x_v$.
Let $\ell$ be the loss term associated with this sample, which is a function of the node variables including $x_v$. As the node $v$ aims to optimize its function $x_v=f_v(z_v)$ , the loss landscape given this input $z_v$ is $$
L_{\mathrm{local}}(z_v,x_v):=\mathbb E[\ell\mid z_v],
$$ 
whereas a hypothetical oracle node that also observes more information $Z$ faces 
$$
L_{\mathrm{oracle}}(z_v,Z,x_v) :=\mathbb E[\ell\mid z_v,Z].
$$
Therefore, for the current architecture to be capable of achieving near-oracle performance, the information available at node $v$ should, in principle, be sufficient for it to select an action $x_v=f_v(z_v)$ close to the oracle-optimal one. This capability requirement can be expressed as
\begin{align}
   \arg\min_{x_v} L_{\mathrm{local}}(z_v,x_v) \approx \arg\min_{x_v} L_{\mathrm{oracle}}(z_v,Z,x_v). \label{eq:cond-b-setting-1} 
\end{align}

Now recall that condition $(B)$ says that the gradient $g_v$ the node receives should be sufficiently entailed by its local observation: $ \mathbb E[g_v\mid z_v]\approx \mathbb E[g_v\mid z_v,Z] $. Noting that $g_v:=\nabla_{x_v}\ell$, Condition $(B)$ is essentially
\begin{align}
   \nabla_x L_{\mathrm{local}}(z_v,x_v) \approx \nabla_x L_{\mathrm{oracle}}(z_v,Z,x_v) . \label{eq:cond-b-setting-2}
\end{align}
Condition $(B)$ can therefore be interpreted as a first-order version of the capability requirement in \eqref{eq:cond-b-setting-1}. Thus, architecture with no hope of good performance cannot satisfy $(B)$. 

We measure how well the condition $(B)$ is satisfied in experiments in Section~\ref{subsec:ablation-cond-B}.

\section{Ablation Studies and Additional Experiments}\label{sec:ablations}
In this section, we present three additional ablation studies.

\subsection{From Gradient-Estimation MSE to Learning Performance}\label{subsec:ablation-mse}
We vary the quality of the synthetic gradient estimators and observe their corresponding gradient estimation MSE and training performance. Recall that synthetic gradient estimators are trained along with the regular model update. We can control the quality of the synthetic gradient estimator by varying how much effort and resources we put into the synthetic gradient estimator training per normal model update step. A good synthetic gradient estimator does reduce gradient estimation MSE and improve training performance.

We compare the following four gradient estimators on SparseNet in the MNIST bandit experiment.
\begin{itemize}
    \item (\textbf{backprop}) gradient obtained by ordinary backpropagation.
    \item (\textbf{sgprop}) For each model update step, the synthetic gradient estimators are updated using the same batch of samples as used in the model update. This is the setting we used in our main experiments.
    \item (\textbf{sgprop-5}) For each model update step, the synthetic gradient estimators are updated using 5 fresh batches of samples.
    \item (\textbf{sgprop-10}) For each model update step, the synthetic gradient estimators are updated using 10 fresh batches of samples.
\end{itemize}

We note that sgprop-5 and sgprop-10 only serve for ablation purposes, as their synthetic gradient predictors use extra data samples than regular training. Nonetheless, the regular sgprop which uses the same samples as backprop shows improved gradient MSE and training performance. For synthetic gradient propagation, we fix the mixing weight to be 0.5. All other configurations are set to be the same: we train for 1000 steps under learning rate 1e-3 for 3 random seeds with the SparseNet architecture same as our main experiments.

We report evaluation reward and per-sample gradient-estimation MSE along the training trajectory for each method in Table~\ref{tab:reward-and-gradient-mse}, where uncertainty is measured by the standard error of the mean (SEM) across seeds. The gradient-estimation MSE is estimated as follows: sample a batch of validation images and obtain per-sample gradient estimators $\hat{\tilde{\bm g}}^\theta$ from each method; calculate the mean $\bar {\bm g}^\theta$ of using the ground truth backpropagated gradients; calculate the empirical mean of $||\hat{\tilde{\grad}}^\theta-\bar \grad^\theta||_2^2$ as the estimated per-sample gradient-estimation MSE.

\begin{table}[h]
    \centering
    \small
    \setlength{\tabcolsep}{4pt}
    \caption{\textbf{Evaluation reward and gradient estimation MSE at different
    training steps}. All values in panel (b), including uncertainties,
    are multiplied by $10^{-3}$. The evaluation-reward ranking of the four methods aligns with their gradient-estimation MSE: lower MSE corresponds to higher evaluation reward.} 
    \label{tab:reward-and-gradient-mse}
    \begin{minipage}{\linewidth}
        \centering
        \textbf{(a) Evaluation reward ($\uparrow$; higher is better)} \\[4pt]
        \begin{tabular*}{\linewidth}
            {@{\extracolsep{\fill}}lccccc@{}}
            \toprule
            Method & Step 200 & Step 400 & Step 600 & Step 800 & Step 1000 \\
            \midrule
            backprop
            & $0.186 \pm 0.0043$
            & $0.252 \pm 0.011$
            & $0.363 \pm 0.0091$
            & $0.438 \pm 0.012$
            & $0.521 \pm 0.02$ \\
            sgprop
            & $0.223 \pm 0.015$
            & $0.361 \pm 0.0075$
            & $0.474 \pm 0.0065$
            & $0.532 \pm 0.013$
            & $0.602 \pm 0.027$ \\
            sgprop-5
            & $0.453 \pm 0.035$
            & $0.634 \pm 0.017$
            & $0.730 \pm 0.015$
            & $0.773 \pm 0.01$
            & $0.800 \pm 0.018$ \\
            sgprop-10
            & $0.503 \pm 0.025$
            & $0.693 \pm 0.012$
            & $0.753 \pm 0.0028$
            & $0.796 \pm 0.0058$
            & $0.822 \pm 0.0062$ \\
            \bottomrule
        \end{tabular*}
    \end{minipage}

    \par\vspace{10pt}

    \begin{minipage}{\linewidth}
        \centering
        \textbf{(b) Gradient estimation MSE ($\times 10^{-3}$; $\downarrow$; lower is better)} \\[4pt]
        \begin{tabular*}{\linewidth}
            {@{\extracolsep{\fill}}lccccc@{}}
            \toprule
            Method & Step 200 & Step 400 & Step 600 & Step 800 & Step 1000 \\
            \midrule
            backprop
            & $196 \pm 9.6$
            & $222 \pm 20$
            & $290 \pm 15$
            & $335 \pm 28$
            & $380 \pm 15$ \\
            sgprop
            & $11.5 \pm 1.9$
            & $13.6 \pm 1.5$
            & $17.2 \pm 1.9$
            & $17.4 \pm 0.81$
            & $21.0 \pm 0.73$ \\
            sgprop-5
            & $6.86 \pm 0.46$
            & $6.30 \pm 0.43$
            & $5.80 \pm 0.64$
            & $6.04 \pm 0.45$
            & $6.31 \pm 0.27$ \\
            sgprop-10
            & $6.04 \pm 0.061$
            & $5.15 \pm 0.51$
            & $5.00 \pm 0.52$
            & $4.79 \pm 0.31$
            & $5.37 \pm 0.78$ \\
            \bottomrule
        \end{tabular*}
    \end{minipage}
    
\end{table}

\subsection{how Well Condition $(B)$ is Satisfied?}\label{subsec:ablation-cond-B}

Condition $(B)$ in Definition~\ref{def:conditions} requires that a node’s local observation $\invar_v$ be sufficient for the conditional mean of its incoming gradient: $ \mathbb{E}[\grad_v\mid \invar_v] = \mathbb{E}[\grad_v\mid \invar_v,\bm{Z}_{<l}] $, where $\bm{Z}_{<l}$ is any information in previous layers, e.g., maximally the global input $\bm{Z}_0$. In this ablation study, we use the global input, since with our deterministic neural networks, any $\bm{Z}_{<l}$ can be derived from $\bm{Z}_0$.

To evaluate how well this condition is approximately satisfied, we need to estimate the two terms: $ \mathbb{E}[\grad_v\mid \invar_v]$ and $ \mathbb{E}[\grad_v\mid \invar_v,\bm{Z}_{0}] $.

For the first term, we train each synthetic gradient predictor from $\invar_v$ to $\grad_v$ to its best using an extensive amount of samples.
For the second term, we compute the ground truth conditional mean gradient $\mathbb{E}[\grad_v\mid \invar_v,\bm{Z}_{0}] $ by analytically averaging over all actions under the fixed policy. We then compare the l2 distance and cosine similarity between the two gradients, averaged over nodes.
The l2 distance between the two terms provides a conservative measure of how closely Condition $(B)$ is satisfied, since it includes both condition $(B)$ violation and the synthetic gradient predictor error.

For SparseNet, we use the same configuration as the main MNIST bandit experiment, freeze a randomly initialized policy, and train its synthetic-gradient predictors for 1000 steps on clean MNIST samples. Then, for each validation image $\bm{Z}_0$, each node $v$ gets its local observation $\invar_v$ and makes a prediction of its gradient that approximates $ \mathbb{E}[\grad_v\mid \invar_v]$, and we compute the ground truth conditional mean gradient $\mathbb{E}[\grad_v\mid \invar_v,\bm{Z}_{0}] $ by analytically averaging over all actions under the fixed policy. Results are averaged across nodes and five seeds, with standard error of mean (SEM) reported, as shown in Table~\ref{tab:layer-wise-metrics}.

\begin{table}[h]
    \centering
    \small
    \setlength{\tabcolsep}{4pt}
    \caption{\textbf{Layer-wise L2 distance ($\ell_2$; lower is better) and cosine
    similarity (Cosine; higher is better)}. L2 distances and their uncertainties
    are multiplied by $10^{-2}$. At the layer interface of width $d=2000$, cosine similarities of approximately 0.5–0.6 reflect substantial positive alignment, compared with a chance-level absolute cosine similarity of about $O(1/\sqrt{d})$ . The predicted gradients are meaningfully aligned with the ground-truth gradients, showing that condition B is reasonably approximately satisfied in SparseNet on the MNIST bandit task.}\label{tab:layer-wise-metrics} 
    \vspace{0.5em}
    \begin{tabular}{@{}lccccc@{}}
        \toprule
        Metric & Layer 0 & Layer 1 & Layer 2 & Layer 3 & Layer 4 \\
        \midrule
        $\ell_2$ ($\times 10^{-2}$) 
        & $1.51 \pm 0.00428$
        & $1.28 \pm 0.00637$
        & $1.18 \pm 0.00604$
        & $1.10 \pm 0.00555$
        & $1.02 \pm 0.00446$ \\
        Cosine 
        & $0.515 \pm 0.0049$
        & $0.599 \pm 0.00446$
        & $0.605 \pm 0.0046$
        & $0.608 \pm 0.00491$
        & $0.610 \pm 0.00514$ \\
        \bottomrule
    \end{tabular}
\end{table}

\subsection{Preliminary Extension to Transformers}\label{subsec:ablation-gpt}

We conduct preliminary experiments that apply synthetic gradients in Mixture-of-Experts (MoE) layers in a small GPT-style transformer. We consider a next-token prediction experiment on TinyStories~\citep{eldan2023tinystories} using a standard three-block GPT-style Transformer. For synthetic gradient prediction, each expert in the MoE layer does a linear transformation from its hidden representations to a gradient prediction. The expert’s predictions are then weighted and averaged using the same expert scores computed by the router. The transformer uses a standard GPT-style architecture with MoE FFN: it has 3 transformer blocks, a 50258 token vocabulary, each block’s FFN is a 32-expert MoE with top-8 activation, and it amounts to about 17M parameters in total. To simulate the scenario where condition $(A)$ presents, we create a noisy dataset from TinyStories in which the next token target has a 0.2 probability of being the true one, while the rest are random tokens.

We compare the training performance of three variants.
\begin{itemize}
    \item (\textbf{gpt-moe sgprop}) GPT with MoE FFNs trained with synthetic-gradient propagation.
    \item (\textbf{gpt-moe backprop}) GPT with MoE FFNs trained by backpropagation.
    \item (\textbf{gpt backprop}) GPT with MLP FFN trained by backpropagation.
\end{itemize}

This comparison separates the effect of the MoE architecture from the effect of synthetic gradients. Mixing weight is fixed at 0.5. Training is done with AdamW, with a learning rate of 1e-3, and each batch of samples corresponds to 1k tokens. We report clean validation loss mean ± SEM across three seeds in Table~\ref{tab:validation-loss}. 

\begin{table}[h]
    \centering
    \caption{\textbf{Validation loss at different training token counts
    (lower is better)}. In this noisy dataset training study, synthetic gradient propagation achieves lower loss than backprop on the GPT with MoE. The standard GPT with MLP serves as a baseline.}
    \label{tab:validation-loss}
    \vspace{0.5em}
    \small
    \setlength{\tabcolsep}{3pt}
    \begin{tabular}{@{}lccccc@{}}
        \toprule
        Method & 1k tokens & 50k tokens & 100k tokens
        & 150k tokens & 200k tokens \\
        \midrule
        gpt-moe sgprop
        & $10.834 \pm 0.001$
        & $8.591 \pm 0.098$
        & $7.395 \pm 0.114$
        & $7.009 \pm 0.072$
        & $6.722 \pm 0.073$ \\
        gpt-moe backprop
        & $10.838 \pm 0.001$
        & $9.140 \pm 0.069$
        & $7.658 \pm 0.029$
        & $7.203 \pm 0.048$
        & $6.830 \pm 0.030$ \\
        gpt backprop
        & $10.626 \pm 0.007$
        & $8.264 \pm 0.247$
        & $7.914 \pm 0.112$
        & $7.554 \pm 0.110$
        & $7.189 \pm 0.095$ \\
        \bottomrule
    \end{tabular}
\end{table}

To evaluate how condition $(B)$ is satisfied in the GPT-MoE architecture. We train the SG predictors within the MoE layers for 100 fresh batches of 16 sequences at SG LR 1e-5. We treat the backpropagated gradient on the clean next token target as the ground truth, and measure the synthetic gradient prediction. This accounts for a conservative measure of Condition $(B)$ as it combines both condition $(B)$ and the prediction error.

Evaluation is done on two kinds of models: randomly initialized; and the same model after five model training warm-up steps. Evaluation uses 32 clean sequences of 256 tokens and computes per-token $\ell_2$ distance and cosine similarity at each layer. Values are averaged over tokens and then across seeds, with SEM reported across five random seeds. The results are presented in Table~\ref{tab:warmup-random-metrics} as follows.

\begin{table}[h]
    \centering
    \caption{\textbf{Layer-wise L2 distance ($\ell_2$; lower is better)
    and cosine similarity (Cosine; higher is better)} for warmup and
    random initialization. $\ell_2$ distances and their uncertainties
    are multiplied by $10^{-3}$. At the MoE layer interface of width $d=256$, meaningful gradient alignment (a conservative proxy for Condition $(B)$) presents across the GPT-MoE layers at random initialization and improves after warming up for five training updates.}
    \label{tab:warmup-random-metrics}
    \vspace{0.5em}
    \small
    \setlength{\tabcolsep}{3pt}

    \begin{minipage}{0.7\linewidth}
        \centering
        \textbf{(a) Model after warmup} \\[4pt]
        \begin{tabular*}{\linewidth}
            {@{\extracolsep{\fill}}lccc@{}}
            \toprule
            Metric & Layer 0 & Layer 1 & Layer 2 \\
            \midrule
            $\ell_2$ ($\times 10^{-3}$)
            & $6.85 \pm 0.61$
            & $4.38 \pm 0.35$
            & $3.69 \pm 0.23$ \\
            Cosine
            & $0.291 \pm 0.013$
            & $0.257 \pm 0.011$
            & $0.199 \pm 0.008$ \\
            \bottomrule
        \end{tabular*}
    \end{minipage}

    \par\vspace{10pt}

    \begin{minipage}{0.7\linewidth}
        \centering
        \textbf{(b) Model at random initialization} \\[4pt]
        \begin{tabular*}{\linewidth}
            {@{\extracolsep{\fill}}lccc@{}}
            \toprule
            Metric & Layer 0 & Layer 1 & Layer 2 \\
            \midrule
            $\ell_2$ ($\times 10^{-3}$)
            & $13.8 \pm 0.070$
            & $8.54 \pm 0.054$
            & $6.27 \pm 0.040$ \\
            Cosine
            & $0.147 \pm 0.0021$
            & $0.161 \pm 0.0028$
            & $0.176 \pm 0.0037$ \\
            \bottomrule
        \end{tabular*}
    \end{minipage}
\end{table}

\end{document}